%% file: neurips_2026.tex
\documentclass{article}

     \PassOptionsToPackage{numbers, compress}{natbib}

 \usepackage[main, final]{neurips_2026}

\usepackage{}
\usepackage[utf8]{inputenc} 
\usepackage[T1]{fontenc}    

\usepackage{hyperref}
\usepackage{url}

\usepackage[utf8]{inputenc} 
\usepackage[T1]{fontenc}    
\usepackage{hyperref}       
\usepackage{url}            
\usepackage{booktabs}       
\usepackage{amsfonts}       
\usepackage{nicefrac}       
\usepackage{microtype}      
\usepackage{xcolor}         
\usepackage{bbm}            
\usepackage{amsmath}        
\usepackage{tikz}           
\usepackage{subcaption}     
\usepackage{booktabs}
\usepackage{titletoc}
\usepackage{multirow}
\usepackage{amsthm}         
\usepackage{amssymb}
\usepackage[capitalize,nameinlink]{cleveref} 
\usepackage{thmtools}       
\usepackage{thm-restate}
\usepackage{algorithm}
\usepackage{adjustbox}
\usepackage{enumitem}
\RequirePackage{algorithm}
\RequirePackage{algpseudocode}
\usepackage{soul}
\usepackage{comment}
\usepackage{wrapfig}

\theoremstyle{plain}
\declaretheorem{theorem}
\declaretheorem[name=Proposition]{proposition}
\declaretheorem[name=Lemma]{lemma}
\declaretheorem[name=Corollary]{corollary}

\declaretheorem[name=Theorem,numbered=no]{theorem*}
\theoremstyle{definition}
\declaretheorem[name=Definition]{definition}
\declaretheorem[name=Assumption]{assumption}
\theoremstyle{remark}
\declaretheorem[sibling=theorem, numberwithin=section]{remark}

\crefname{assumption}{Assumption}{Assumptions}
\crefname{proposition}{Prop.}{Props.}
\crefname{corollary}{Cor.}{Cors.}

\newcommand{\R}{\mathbb{R}}
\newcommand{\E}{\mathbb{E}}

\newcommand{\Law}{\mathcal{L}}

\newcommand{\dd}{\,\mathrm{d}}
\newcommand{\OT}{\,\mathrm{OT}}
\newcommand{\TV}{\,\mathrm{TV}}

\newcommand{\muA}{\mu_{\eta}}           
\newcommand{\muU}{\mu_{\eta,\Delta}}   
\newcommand{\zu}{z}
\newcommand{\fcost}{c_{\mathcal{D}}}  

\newcommand{\vincent}[1]{{\color{orange}#1}}

\title{Geometry of Relaxed Fair Regression: A Unified Framework for Aware and Unaware Settings}

\author{%
  Marie Generali Lince
    \\
  Inria, Fairplay joint team\\
  CREST, ENSAE, IP Paris \\
  Palaiseau, France \\
  \texttt{marie.generali@ensae.fr} \\
  \And
  Vincent Divol \\
  CREST, ENSAE, IP Paris \\
  Palaiseau, France \\
  \texttt{vincent.divol@ensae.fr} \\
  \AND
  Remi Flamary \\
  CMAP, Ecole Polytechnique \\
  Palaiseau, France\\
  \texttt{remi.flamary@polytechnique.edu} \\
  \And
  Solenne Gaucher \\
  CMAP, Ecole Polytechnique \\
  Palaiseau, France\\
  \texttt{solenne.gaucher@polytechnique.edu} \\
  \And
  Patrick Loiseau \\
  Inria, Fairplay joint team \\
  Palaiseau, France \\
  \texttt{patrick.loiseau@ensae.fr} \\
}

\begin{document}

\maketitle

\begin{abstract}
Fairness-accuracy trade-offs are a central concern in the deployment of fairness-aware machine learning methods. 
When sensitive attributes are unavailable at inference time–the so called unawareness setting, principled methods for obtaining accurate predictions under relaxed fairness constraints are largely missing. 
In this work, we address this gap by formulating regression under a demographic parity penalty as an optimal transport problem. 
Our framework, unifies both the \emph{aware} and \emph{unaware} settings and characterizes optimal prediction functions via optimal transport maps, under both squared Wasserstein-2 and Total Variation penalties. 
These results reveal that the choice of penalty reflects fundamentally different fairness philosophies: the Wasserstein penalty induces a smooth, population-wide compromise, while Total Variation enforces parity for a subset of individuals. 
Building on these theoretical characterizations, we propose FAUST, an algorithm that is simple to implement, computationally efficient, and consistently matches or outperforms state-of-the-art baselines on real-world benchmarks.

  \end{abstract}

\section{Introduction}


As machine learning  models are increasingly deployed in high-stakes domains, mitigating algorithmic bias has become a critical objective. 
A significant portion of the fairness literature addresses group fairness, imposing fair distribution of the prediction across groups. 
Among these, Demographic Parity (DP) has received considerable attention due to its simplicity and interpretability, aiming to ensure that predictions are statistically independent of the sensitive attribute. Mitigation strategies typically intervene on the training data (\emph{pre-processing}), the learning algorithm (\emph{in-processing}), or the model's predictions (\emph{post-processing}) \citep{barocas-hardt-narayanan,dwork2011fairnessawareness}. 
Post-processing methods are particularly attractive as they allow for the retrofitting of any black-box predictor without requiring retraining, making them applicable in real-world settings where the base model is already deployed. 
In this work, we focus on post-processing for regression under the DP constraint.

\begin{wrapfigure}{r}{0.3\textwidth}
  \vspace{-15pt}
  \centering
  \includegraphics[width=0.3\textwidth]{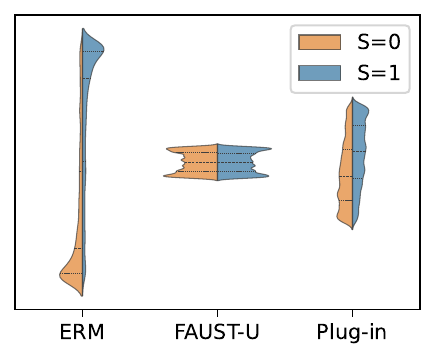}
   \vspace{-20pt}
\caption{\textbf{FAUST-U vs. Plug-In.} Marginal distributions of the predictions in each group (on data from \cref{subsec:philosophies_relaxed_fairness}). The Plug-in approach fails retaining the ERM gap, whereas FAUST-U successfully aligns group distributions.}
  \label{fig:plugin_failure}
  \vspace{-10pt}
\end{wrapfigure}The strategy for enforcing DP depends on data availability: in the \emph{aware} setting, the sensitive attribute is observed alongside standard features, while in the \emph{unaware} setting, it is unavailable at inference time. 
A classical heuristic to circumvent this constraint 
is to explicitly estimate the sensitive attribute and apply group-aware corrections using this plug-in estimator, but \textbf{plug-in approaches generally fail to achieve fairness} (see \cref{fig:plugin_failure} \cref{app:toy_plugin}). Furthermore, enforcing \emph{exact} DP in either setting is often too rigid, causing severe degradation in predictive utility. This limitation motivates the study of \emph{relaxed} fairness, which allows for controlled violations to optimize the accuracy-fairness trade-off.

\vspace{-2mm}
\paragraph{Related works.} Empirical Risk Minimization (ERM) under fairness penalties or relaxed fairness constraints is common within the \emph{in-processing} paradigm. 
Numerous methods directly incorporate fairness penalties or bounded constraints into the learning objective during model training \citep{agarwal2019fairregressionquantitativedefinitions,Berk2017ACF,pmlr-v80-hashimoto18a, fairness_without_demographics_lahoti,li2019kerneldependenceregularizersgaussian,zafar2019fairness,cotter_optim_nondiffconstraints}. 
More recently, optimal transport (OT) has been leveraged within this paradigm to regularize latent representations and enforce relaxed fairness even when sensitive attributes are unobserved \citep{leteno2025fairtextclassificationtransferable}.  
These approaches inherit a core limitation of in-processing methods: they presuppose control over the algorithm's training procedure, which is unrealistic in many real-world settings. 
By contrast, OT provides a natural geometric framework for enforcing DP through \emph{post-processing} approaches. 
In the \emph{aware} setting, exact fairness is achieved by mapping group marginal distributions of the Bayes predictor to their Wasserstein barycenter \citep{chzhen2020fairregressionwassersteinbarycenters, Silvia_Ray_Tom_Aldo_Heinrich_John_2020, gouic2020projectionfairnessstatisticallearning, gaucher2022fairlearningwassersteinbarycenters, denis2023fairnessguaranteemulticlassclassification, Berk2017ACF,li2019kerneldependenceregularizersgaussian}, and closed-form relaxations exist to control the trade-off between fairness and accuracy \citep{chzhen2022minimaxframeworkquantifyingriskfairness}. 
In the \emph{unaware} setting, achieving parity often relies on noisy proxies or robust optimization techniques \citep{chen2019fairness, awasthi2020equalized}.
Recent work has characterized the optimal regression function under DP using OT over 2D measures \citep{divol2024demographicparityregressionclassification}. 
Yet extending the fairness relaxations to the unaware setting remains an open challenge. 
State-of-the-art methods currently bypass characterization of the optimal fair predictor: FairReg \citep{Fairreg} relies on grid-based discretization and stochastic optimization, while the recent FRAPPÉ framework \citep{tifrea2024frappe} learns an additive neural network correction via stochastic gradient descent.
This highlights a significant gap in the literature: while exact fairness under unawareness admits an OT formulation, its relaxed counterpart is solved via grid-based approximations or gradient-based methods.

\paragraph{Outline and contributions.} 
In this paper, we bridge this gap by establishing pointwise closed-form reductions
for relaxed fairness across both \emph{aware} and \emph{unaware} settings, focusing on the two-groups case. 
While restrictive, the two-groups assumption is common in the literature
\citep{NEURIPS2018_8e038477} and captures the practically important scenario in
which one group corresponds to a protected class, while already revealing key
phenomena. Our main contributions are as follows:

\vspace{-2mm}
\begin{itemize}[leftmargin=15pt]
   \setlength\itemsep{0.3em}
   \setlength{\parskip}{0pt}
   \setlength{\topsep}{-5pt}
   \item \textbf{A unified framework for relaxed fair regression:} In \cref{subsec:framework_OT_relaxed}, we formulate relaxed fair regression as a risk minimization problem penalized by a fairness discrepancy metric as
 {
   \setlength{\abovedisplayskip}{3pt} 
   \setlength{\belowdisplayskip}{3pt} 
    \begin{equation*}
   f_{\lambda, \mathcal{D}}^* \in \arg\min_{f} \; \mathbb{E}\left[ \left(f(X) - \eta(X)\right)^2 \right]  + \lambda \  \mathcal{U}_{\mathcal{D}}(f),
 \end{equation*}
   }
accommodating different fairness penalties ($\mathcal W_2$ and TV) across both aware and unaware settings.
    \item \textbf{Closed-form solutions via OT reduction:} 
    In \cref{subsec:New_results} we show that, 
    the penalized relaxation reduces to a standard OT problem. Unlike existing state-of-the-art methods, \textbf{we characterize the optimal predictor} for almost-fair prediction functions under both $\mathcal W_2^2$ and TV penalties and provide closed-form solution via OT maps. 
    
    \item \textbf{Practical implications:} In \cref{subsec:philosophies_relaxed_fairness}, we make explicit the trade-offs inherent in the choice of fairness penalty and contrast their practical implications. We show that the relaxation penalty is not just a numerical convenience, but reflects fundamentally different approaches to fairness.
  \item \textbf{Implementation and experiments:} In \cref{sec:algorithm}, we demonstrate how these closed-form solutions can be efficiently estimated from finite samples to build out-of-sample fair predictors. In \cref{sec:experiments}, our theory-driven approach is empirically validated on the Law School and Communities and Crime datasets. \footnote{Our code is available at \url{https://github.com/mgeneralince/FAUST}.}.
\end{itemize}

\paragraph{Notation.}
Let $\mathcal{P}(\Omega)$ be the set of probability measures on $\Omega$, and $\Law(\cdot)$ the law of a random variable. 
For $\mu \in \mathcal{P}(\mathbb{R})$, denote its cumulative distribution function (c.d.f.) 
 by $F_\mu$.
For a measurable map $T$, $T_{\#} \mu$ is the pushforward measure. 
The set of couplings with marginals $\mu$ and $\nu$ is denoted $\Pi(\mu, \nu)$. 
The OT cost between $\mu$ and $\nu$ under a  cost $c$ is defined by $\OT_c(\mu, \nu)\!=\! \inf_{\pi \in \Pi(\mu, \nu)} \int c(x,y) \dd\pi(x,y)$. $\mathcal{W}_2^2(\mu, \nu)$ is the squared Wasserstein-2 distance, i.e., $\OT_{c_2}(\mu, \nu)$ with cost $c_2(x,y) \!=\! (x-y)^2$, 
and $\TV(\mu, \nu)$ the Total Variation distance, i.e. $\OT_{c_{0}}(\mu, \nu)$ with cost $c_{0}(x,y) \!= \!\mathbbm{1}_{x \neq y}$  \cite[Proposition 4.7]{dobrushin4_7}. 
Finally, for $a\in\mathbb{R}$, $a_+ \!=\! \max(a,0)$ and $a_- \!=\! - \min(a, 0)$.

\section{Unified framework for aware and unaware relaxed fair regression}
\label{sec:unified_framework}

We consider a regression problem where individuals are characterized by a non-sensitive feature $ X$ taking values in some set $ \mathcal{X} \subset \mathbb{R}^d$, a binary sensitive attribute $S \in \{+, -\}$ with prior probabilities $p^+ = \mathbb{P}(S = +)$ (resp. $p^- = \mathbb{P}(S = -)$), and a continuous target variable $Y \in \mathbb{R}$ with $\mathbb{E}[Y^2]<\infty$. The law of $X$ in group $S = +$ (resp. in group $S = +$) is denoted $\chi^+ = \mathcal{L}(X \vert S = +)$ (resp. $\chi^- = \mathcal{L}(X \vert S = -)$). Abusing notations, let $\eta(x,s) = \mathbb{E}[Y \mid X=x, S=s]$ and $\eta(x) = \mathbb{E}[Y \mid X=x]$. 

In regression problems, the goal is to predict an outcome $Y \in \mathbb{R}$ from either the pair $(X,S)$ (in the awareness setting) or from $X$ (in the unawareness setting).    
Then, under mean square error, the Bayes optimal predictor is given by the conditional expectation $\eta$. 
Unfortunately, its predictions often reflect systemic biases toward $S$ \citep{barocas-hardt-narayanan}. 
Among the criteria characterizing fair prediction, DP is satisfied by a prediction function $f$ if its predictions are statistically independent of $S$: 
($f(X,S) \perp\!\!\!\perp S$) in the awareness setting, and ($f(X) \perp\!\!\!\perp S$) in the unawareness setting. 

\subsection{Existing results for exact fairness}\label{subsec:existing_results}


\paragraph{Exact fairness under awareness.} 
In the awareness setting, \citet{chzhen2020fairregressionwassersteinbarycenters} show that finding the optimal fair predictor is equivalent to solving the following Wasserstein barycenter problem.  For $s \in \{+,-\}$, let $\muA^s\!=\!\eta(\cdot,s)_{\#}\chi^s$. 
When the distributions $\muA^+$ and $\muA^-$ are non atomic, 
the authors show that the distribution of the optimal fair prediction $\nu^*$ corresponds to the Wasserstein barycenter
\begin{equation}
    \nu^* = \arg\min_{\nu \in \mathcal{P}(\mathbb{R})} \left\{ p^+ \mathcal{W}_2^2(\muA^+, \nu) + p^- \mathcal{W}_2^2(\muA^-, \nu) \right\}. \label{eq:OTAExact}
\end{equation}
Since 
$\muA^+$ and $\muA^-$ are one-dimensional, the barycenter problem admits a closed-form OT map. 
The optimal fair prediction $\smash{f^*(x, s)}$ is obtained by applying this map (transporting from $\smash{\muA^s}$ to the barycenter $\smash{\nu^*}$) to 
$\smash{\eta(x,s)}$. We detail the closed-form solution in \cref{app:aware_corollaries}.

\textbf{Exact fairness under unawareness.} 
In the \emph{unaware} setting, $S$ is unobserved at inference time. 
As established by \citet{divol2024demographicparityregressionclassification}, the optimal fair prediction implicitly infers $S$ by basing its predictions on the signed group probability $\Delta$ given by
\begin{equation}
    \Delta(x) = \frac{\mathbb{P}(S=+ \mid X=x)}{p^+} - \frac{\mathbb{P}(S=- \mid X=x)}{p^-}.
    \label{eq:theory_delta}
\end{equation}
Let us denote $\left(\left(\chi^+\! -\! \chi^-\right)_+\!,\! \left(\chi^+ \!-\! \chi^-\right)_-\right)$ the normalized Jordan decomposition of the signed measure $\smash{\chi^+\! - \!\chi^-}$, where the normalization constant is chosen so that $\left(\chi^+ - \chi^-\right)_+$ and $\smash{\left(\chi^+ \!-\! \chi^-\right)_-}$ are probability measures \footnote{ See \citet{divol2024demographicparityregressionclassification}, Section 3.2.1 or \cref{app:tightness_bound} for a construction of these measures.}.
Then, the fair regression problem reduces to an OT problem between the probability measures $\muU^+$ and $\muU^-$ on $\mathbb{R}\times \mathbb{R}$, where 
\vspace{-0.1cm}
\begin{equation*}
  \muU^+ = \left(\eta(\cdot), \Delta(\cdot)\right)_{\#} \left(\chi^+ - \chi^-\right)_+ \qquad \text{and}\qquad \muU^- = \left(\eta(\cdot), \Delta(\cdot)\right)_{\#} \left(\chi^+ - \chi^-\right)_-.
\end{equation*} 
By construction, the probability measure $\muU^+$ (resp. $\muU^-$) only gives weight to the half space $\{(h,d) \in \mathbb{R}^2: d >0\}$ (resp. the half space $\{(h,d) \in \mathbb{R}^2: d <0\}$).
Then, the fair regression problem can be reduced to an OT problem under the following assumption.
\begin{assumption}\label{ass_non_atomic}
The measures $\mu^+_{\eta,\Delta}$ and $\mu^-_{\eta,\Delta}$ give zero mass to graphs of functions in the sense that for any
measurable function $F: \mathbb{R}^*\rightarrow \mathbb{R}$ and $s\in \{+,-\}$, $\mu^s_{\eta,\Delta}\left(\left\{\left(F(d), d\right): d \neq 0\right\}\right)  = 0.$
\end{assumption}
Under \cref{ass_non_atomic}, \citet{divol2024demographicparityregressionclassification} shows that finding the optimal fair predictor reduces to solving the following barycenter problem 
\vspace{-0.1cm}
\begin{equation}\label{eq_Pinfinity}
  \inf_{\nu \in \mathcal{P}(\mathbb{R})} \left\{ \OT_{c_{u}}(\muU^+, \nu) + \OT_{c_{u}}(\muU^-, \nu) \right\},
  \tag{$P_\infty$}
\end{equation}
with cost $c_{u}(\zu, y)= \tfrac{\left\vert h - y\right\vert^2}{\vert d\vert}$ specific to the unaware geometry, where $\zu = (h, d) \in \mathbb{R} \times \mathbb{R}$. 
Note that this is a two-to-one dimensional OT problem, which in
general does not admit a closed-form solution but can be solved numerically
\citep{tanguy2024computing}.

\paragraph{Reducing awareness to unawareness.} \label{remark:unifying_aware_unaware} 
As noted by \citet{divol2024demographicparityregressionclassification}, the results for the awareness setting can be recovered from those in the unawareness setting by assuming that the sensitive attribute $S$ is included in the features $X$, or equivalently that $S$ is $X$-measurable. In this case, 
$\Delta(x) \!=\! \frac{s}{p^s}$. 
Then, \cref{ass_non_atomic} reduces to assuming non-atomicity of the measures $\mu_{\eta}^s$, a condition that is necessary for the results of \citet{chzhen2020fairregressionwassersteinbarycenters} to hold. 
Similarly, the transport cost simplifies as  $c((\eta(x), \Delta(x)), y) = p^s(\eta(x)\!-\! y)^2$. 
Consequently, the OT problem defined in the unawareness setting (\cref{eq_Pinfinity}) reduces, in the awareness setting, to the problem given in \cref{eq:OTAExact}. 
For this reason, we focus in the remainder of the paper on the more general unawareness setting, noting that the awareness case follows as a direct special case of our results.

\vspace{0.3em}

\begin{remark}[Noisy attributes] Our formulation also covers the case where a proxy $\smash{\hat{S}}$ is available at inference time. Indeed, including $\hat{S}$ as a covariate in the feature vector $X$, we obtain $\Delta(x, \hat{s})=\tfrac{\mathbb{P}(S=+ \mid X=x, \hat{S}=\hat{s})}{p^+} \!- \!\tfrac{\mathbb{P}(S=- \mid X=x, \hat{S}=\hat{s})}{p^-}$, thus interpolating between the two settings. When $\hat{S}$ is $X$-measurable or contains uninformative noise, $\smash{\mathbb{P}(S \vert X, \hat{S}) = \mathbb{P}(S \vert X)}$, recovering the unaware formulation. Conversely, when $\smash{\hat{S} = S}$, we recover  the aware setting \citep{chzhen2022minimaxframeworkquantifyingriskfairness}. 
\end{remark}

\subsection{Optimal Transport for relaxed fairness: a unified framework}
\label{subsec:framework_OT_relaxed}

 In many practical settings, enforcing DP can significantly degrade predictive accuracy, sometimes to an unacceptable extent. In this section, we present our general framework for relaxed fair regression. As established in \cref{subsec:existing_results}, DP is satisfied if and only if the pushforwards of the positive and negative variations of the signed measure $\chi^+\! -\! \chi^-$ are equal. To achieve relaxed fairness, we penalize the risk of a prediction function by its $\mathcal D$-unfairness defined below.

 \begin{definition}[$\mathcal{D}$-unfairness]
   \label{def:generalized_unfairness}
 Let $\mathcal{D}$ be a probability discrepancy measure over $\mathcal{P}(\mathbb{R})$, and let $m = \TV(\chi^+, \chi^-)/2$. 
 The $\mathcal{D}$-unfairness of a prediction function $f$ is defined as 
 \vspace{-0.1cm}
\begin{equation*}
       \mathcal{U}_{\mathcal{D}}(f) =m \cdot \mathcal{D}\left(f_{\#}(\chi^+ - \chi^-)_+, \; f_{\#}(\chi^+ - \chi^-)_-\right).
   \end{equation*}
   \vspace{-0.6cm}
 \end{definition}
We consider two choices for the discrepancy measure $\mathcal{D}$: the squared Wasserstein-2 distance ($\mathcal U_{\mathcal W_2^2}$), which evaluates the quadratic displacement required to align the distributions, and the Total Variation distance ($\mathcal{U}_{TV}$), which evaluates the maximal probability mass of individuals receiving disparate predictions. Both penalties can be expressed as $ \mathcal{D}(\nu^+, \nu^-) = \OT_{\fcost}(\nu^+, \nu^-)$ for some cost function $\fcost$, equal to $c_2(y,y')=(y - y')^2$ for $\mathcal W_2^2$ and to $c_0(y,y')=\mathbbm{1}_{\{y \neq y'\}}$ for TV. Our objective is to find the almost-fair prediction function $f^*_{\lambda, \mathcal{D}}$ solving the following penalized risk minimization problem:
 \vspace{-0.2cm}
 \begin{equation}
   f_{\lambda, \mathcal{D}}^* \in \arg\min_{f} \; \mathbb{E}\left[ \left(f(X) - \eta(X)\right)^2 \right]  + \lambda \  \mathcal{U}_{\mathcal{D}}(f).
   \label{eq:erm_penalized}
 \end{equation}
On the one hand, driving $\lambda \!\rightarrow\! \infty$ enforces $\mathcal{U}_{\mathcal{D}}(f)\!=\! 0$, which is equivalent to enforcing DP. On the other hand, the choice of unfairness measure $\mathcal U_\mathcal{D}$ enables us to cast the relaxed regression problem as an OT problem for general values of $\lambda$. In particular, it upper bounds the $\mathcal{D}$ distance between the distributions of the predictions: $\mathcal{U}_{\mathcal{D}}(f) \geq \mathcal{D}\left(\Law(f(X) \vert S=+), \Law(f(X) \vert S=-)\right)$, with equality if $\mathcal{D} = \TV$, as well as in the awareness setting (see Appendix \ref{app:tightness_bound}). When $\mathcal D = \mathcal{W}_2$, the gap is still bounded by the TV distance between the pushforward variations (see Appendix~\ref{app:tightness_bound}). Moreover, our algorithm empirically promotes low unfairness with respect to multiple metrics (see \cref{sec:experiments}). 

\paragraph{Relaxed fair regression as an OT problem.}
To solve \eqref{eq:erm_penalized}, we adopt an OT perspective. 
Under \cref{ass_non_atomic}, we show that when $\mathcal{D}$ is taken to be either  
$\mathcal W_2^2$ or  
TV, and the cost $\fcost$ is such that $\mathcal D(\nu^+, \nu^-) = \OT_{\fcost}(\nu^+, \nu^-)$, the value of \eqref{eq:erm_penalized} coincides with that of the following OT problem
\begin{equation}\label{eq:P-lambda-general}
 \inf_{\nu^+, \nu^- \in \mathcal{P}(\mathbb{R})} \left\{ \OT_{c_u}(\muU^+, \nu^+) + \OT_{c_u}(\muU^-, \nu^-) + \lambda \OT_{\fcost}(\nu^+, \nu^-) \right\} \tag{\ensuremath{P_{\lambda, \mathcal{D}}}}.
\end{equation}

\begin{restatable}{theorem}{lemReducOT} \label{lem:reduc_OT}
Let $\mathcal{D}(\cdot, \cdot) = \OT_{\fcost}(\cdot, \cdot)$ for $\fcost(y,y') = (y-y')^2$ or $\fcost(y,y') = \mathbbm{1}_{\{y\neq y'\}}$. 
Under \cref{ass_non_atomic}, Problem (\ref{eq:P-lambda-general}) is solved by deterministic transport maps $T^+_{\lambda,\mathcal{D}}$ and $T^-_{\lambda,\mathcal{D}}$ which push $\muU^+$ and $\muU^-$ to the minimizing measures $\nu^+$ and $\nu^-$, respectively. Moreover, the solution of the relaxed fair regression problem \eqref{eq:erm_penalized} is given by
\vspace{-1mm}
\begin{align}\label{eq:sol_from_transport_maps}
    f_{\lambda, \mathcal{D}}^*(x) = 
    \begin{cases}
        T^+_{\lambda,\mathcal{D}}(\eta(x),\Delta(x)) \quad \text{if} \quad \Delta(x) > 0, \\
        T^-_{\lambda,\mathcal{D}}(\eta(x),\Delta(x)) \quad \text{if} \quad \Delta(x) < 0,\\
        \eta(x) \quad \text{else} .
    \end{cases}
\end{align}
\vspace{-1mm}
\end{restatable}
\cref{lem:reduc_OT} shows that the relaxed fair regression problem reduces to solving Problem (\ref{eq:P-lambda-general}) and finding an OT map.  
The corresponding problems arising in the unawareness and awareness settings are illustrated in \cref{fig:ot_schema}. In the awareness setting, the solution amounts to constructing transport maps between one-dimensional distributions. 
In the unawareness setting, by contrast, one must transport two-dimensional measures $\smash{\muU^{\pm}}$ to a one dimensional prediction, where each point $(h, d)$ in the support of $\smash{\muU^{\pm}}$ encodes the Bayes prediction $\smash{\eta}$ and the signed group probability $\smash{\Delta}$. Crucially, we show in \Cref{prop:relaxed_reduction,prop:tv_relaxed_reduction} that the 2-to-1-dimensional barycenter problem can be reduced to an OT problem between two 2-dimensional measures, with a penalty-dependent ground cost ($C_{\lambda}$), thereby enabling efficient numerical estimation.

\begin{figure}[h!]
  \centering
  \begin{subfigure}[t]{0.43\textwidth} 
      \centering
      {\normalsize Awareness}
      \includegraphics[width=\textwidth]{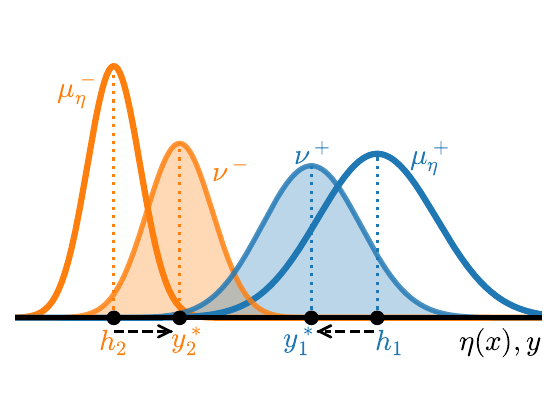}
  \end{subfigure}
  \hfill
  \begin{subfigure}[t]{0.43\textwidth} 
      \centering
      {\normalsize Unawareness}
      \includegraphics[width=\textwidth]{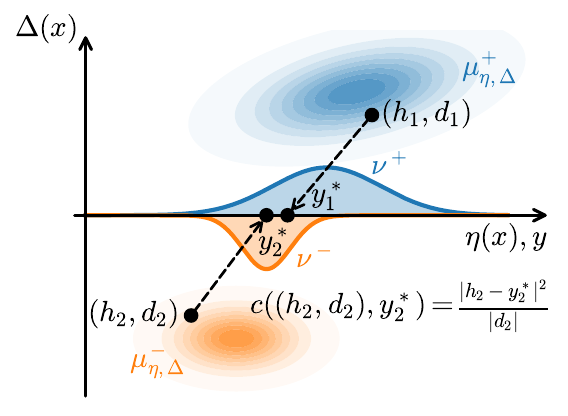}
  \end{subfigure}
  \caption{\textbf{Geometry of the OT maps.} 
  \emph{Left (awareness):} 1D-to-1D transport mapping group-conditional distributions to almost-fair predictions. 
  \emph{Right (unawaress):} 2D-to-1D transport mapping the joint distributions of $(\eta(X), \Delta(X))$ under $\muU^+$ and $\muU^-$ to almost-fair predictions.}
\label{fig:ot_schema}
\end{figure}

In \cref{tab:contribution_grid}, we instantiate the general framework in \ref{eq:P-lambda-general} for the aware and unaware settings, and for both unfairness measures considered. 
We obtain three new formulations for relaxed fair regression, and we recover one existing formulation \textbf{FAUST-A $\mathbf{\mathcal W_2}$} \citep{chzhen2022minimaxframeworkquantifyingriskfairness}.


 \begin{table}[t!]
  \centering
  \renewcommand{\arraystretch}{1.7}
  \resizebox{\textwidth}{!}{
  \begin{tabular}{@{}l p{5.8cm} p{6.6cm} @{}}
      \toprule
      & \textbf{Unaware Setting} & \textbf{Aware Setting} \\
       \toprule
       \textbf{Objective} & $\OT_{c_u}(\muU^+, \nu^+)+\OT_{c_u}(\muU^-, \nu^-)\newline+\lambda \OT_{\fcost}(\nu^+, \nu^-)$ & $p^+\OT_{c_2}(\muA^+, \nu^+)+p^-\OT_{c_2}(\muA^-, \nu^-) \newline +\lambda \OT_{\fcost}(\nu^+, \nu^-)$\\
      \midrule
      \textbf{$\mathbf{\mathcal W_2^2}$-fairness} \newline 
      & \textbf{FAUST-U $\mathbf{\mathcal W_2}$} (\cref{prop:relaxed_reduction})\newline 
      \textbf{Cost:} $c_u(z,y) = (h-y)^2 / \vert d \vert$ \newline
      \textbf{Penalty:} $\fcost(y, y') = (y-y')^2$\newline
      \textbf{Ground cost:} $C_{\lambda,c_2}^u = \frac{\lambda(h_1-h_2)^2}{1+\lambda(\vert d_1 \vert+\vert d_2 \vert)}$
      & \textbf{FAUST-A \!$\mathbf{\mathcal W_2}$}~(\citeauthor{chzhen2022minimaxframeworkquantifyingriskfairness},~\cref{cor:aware_w2})\newline 
      \textbf{Cost: }$c_2(y,y') = (y-y')^2$ \newline
      \textbf{Penalty: }$\fcost(y, y') = (y-y')^2$\newline
      \textbf{Ground cost:} $C_{\lambda,c_2}^a = \frac{\lambda p^+ p^-}{\lambda + p^+ p^-}(h_1-h_2)^2$ \\
      \midrule
      \textbf{$\mathbf{TV}$-fairness} 
      & \textbf{FAUST-U TV} (\cref{prop:tv_relaxed_reduction})\newline 
      \textbf{Cost: }$c_u(z,y) = (h-y)^2 / \vert d \vert $ \newline
      \textbf{Penalty: }$\fcost(y, y') = \mathbbm{1}_{\{y\neq y'\}}$\newline 
      \textbf{Ground cost:} $C_{\lambda,c_0}^u = \min ( \lambda, \frac{(h_1-h_2)^2}{\vert d_1 \vert+\vert d_2 \vert})$
      & \textbf{FAUST-A TV} (\cref{cor:aware_tv})\newline 
        \textbf{Cost: }$c_2(y,y') = (y-y')^2$ \newline
        \textbf{Penalty: }$\fcost(y, y') = \mathbbm{1}_{\{y \neq y'\}}$ \newline
        \textbf{Ground cost:} $C_{\lambda,c_0}^a \!\!=\! \min ( \lambda, p^+\! p^-(h_1\!-\!h_2)^2 )$ \\
      \bottomrule
  \end{tabular}}
  \vspace{0.5em}
  \caption{Summary of the unified OT framework. The general problem $P_{\lambda, \mathcal{D}}$ instantiates into specific pointwise closed-form reductions depending on error cost $c \in \{c_u, c_2\}$ and penalty $\fcost$.}
  \label{tab:contribution_grid}
\end{table}

\subsection{New results for relaxed fairness}\label{subsec:New_results}
\cref{lem:reduc_OT} establishes that the solution to the fair relaxation problem is given by the OT maps solving \eqref{eq:P-lambda-general}. In this section, we show that \eqref{eq:P-lambda-general} can itself be reduced to an OT problem between
$\smash{\muU^+}$ and $\smash{\muU^-}$ with an appropriately chosen cost. 
Specializing to the $\smash{\mathcal{W}_2^2}$ and the TV-fairness penalty, we obtain solutions to \eqref{eq:erm_penalized} expressed explicitly in terms of the OT maps. Closed-form solutions in the awareness setting follow as a corollary.

\paragraph{Relaxation under $\mathcal{W}_2^2$-fairness.} We consider the fairness penalty $\mathcal{D}(\cdot, \cdot)\! = \! \mathcal{W}_2^2(\cdot, \cdot)$. The following proposition reduces \hyperref[eq:P-lambda-general]{$(P_{\lambda, \mathcal{W}_2})$} to an OT problem between $\muU^+$ and $\muU^-$ for a suitable cost, and gives an expression for the optimal almost-fair prediction via the OT plan. The proof is in \cref{app:unaware_proof_relaxed_W2}.

\begin{restatable}[\textbf{FAUST-U $\mathbf{\mathcal W_2}$} ]{proposition}{proprelaxedWtwo}
  \label{prop:relaxed_reduction}
  Under \cref{ass_non_atomic}, Problem \hyperref[eq:P-lambda-general]{$(P_{\lambda, \mathcal{W}_2})$} is equivalent to the OT problem 
  \vspace{-0.3cm}  
  \begin{equation}\label{Eqref:OT_W2_relaxed}
   \pi^* \in \arg\min_{\pi \in \Pi(\muU^+, \muU^-)} \int_{\mathbb{R}^2 \times \mathbb{R}^2} C_{\lambda, c_2}^u(\zu_1, \zu_2)\dd\pi(\zu_1,\zu_2),
    \end{equation}
    where $C_{\lambda, c_2}^u(\zu_1, \zu_2) = \frac{\lambda}{1+\lambda(\vert d_1 \vert +\vert d_2 \vert)}(h_1-h_2)^2$ for $\zu_1=(h_1, d_1)$ and $\zu_2=(h_2, d_2)$. 
    Moreover, a solution to the almost-fair regression problem in \eqref{eq:erm_penalized} is given by \cref{eq:sol_from_transport_maps} with 
    \vspace{-0.5mm}
  \begin{equation}\label{eq:y-star}
        T^+_{\lambda,\mathcal{W}_2^2}(\zu_1) = \frac{(1+\lambda \vert d_2 \vert)h_1+\lambda \vert d_1 \vert h_2}{1+\lambda(\vert d_1 \vert+\vert d_2 \vert)} \quad \text{and} \quad
        T^-_{\lambda,\mathcal{W}_2^2}(\zu_2) = \frac{\lambda \vert d_2 \vert h_1+(1+\lambda \vert d_1 \vert)h_2}{1+\lambda(\vert d_1 \vert+\vert d_2 \vert)}
    \end{equation}
    for all pairs $(\zu_1,\zu_2)$ in the support of $\pi^*$.
  \end{restatable}
\cref{prop:relaxed_reduction} provides a roadmap for building almost-fair
classifiers: first, solve the OT problem (\ref{Eqref:OT_W2_relaxed}) with ground cost $C_{\lambda, c_2}$, then
compute the prediction function via the OT plan using
\cref{eq:sol_from_transport_maps} and \cref{eq:y-star}. While 
\cref{Eqref:OT_W2_relaxed} is solved by a (non-deterministic) transport plan,
we show in \cref{app:unaware_proof_relaxed_W2} that  
the value defining $T^+_{\lambda,\mathcal{W}_2^2}(z_1)$  does not depend on the choice of $z_2$ for $(z_1,z_2)$ in the support of $\pi^*$, so that $\smash{T^+_{\lambda,\mathcal{W}_2^2}}$ is a well-defined map (and similarly for $\smash{T^-_{\lambda,\mathcal{W}_2^2}}$). 
As noted in
\cref{app:aware_corollaries} and \cref{rem:lambda_to_infty}, this formulation
encompasses prior results as boundary cases. In the exact
fairness limit ($\lambda\! \to\! \infty$), \cref{prop:relaxed_reduction} recovers,
in the awareness framework, the fair regression function of
\citet{chzhen2020fairregressionwassersteinbarycenters}, and in the unawareness
framework, the OT map of
\citet{divol2024demographicparityregressionclassification}. For generic
$\lambda>0$, in the awareness framework it recovers the geodesic interpolations of
\citet{chzhen2022minimaxframeworkquantifyingriskfairness}, while in the unawareness
framework it yields a new characterization of the optimal predictor under a relaxed fairness constraint.

\paragraph{Relaxation under TV-fairness} We now consider the fairness penalty $\mathcal{D} (\cdot, \cdot) = \TV(\cdot, \cdot)$. The following proposition reduces \hyperref[eq:P-lambda-general]{$(P_{\lambda, \TV})$} to an OT problem between $\muU^+$ and $\muU^-$ with a suitable cost, and expresses the optimal almost-fair prediction in terms of the OT plan. The proof is in \cref{app:unaware_proof_relaxed_TV}.

\begin{restatable}[\textbf{FAUST-U TV}]{proposition}{proprelaxedTV}
  \label{prop:tv_relaxed_reduction}
  Under \cref{ass_non_atomic}, Problem \hyperref[eq:P-lambda-general]{$(P_{\lambda, \TV})$} is equivalent to the OT problem 
  \begin{equation}\label{eq:OT_TV_relaxed}
      \pi^* \in \arg\min_{\pi \in \Pi(\muU^+, \muU^-)} \int_{\mathbb{R}^2 \times \mathbb{R}^2} C_{\lambda, c_0}^u(\zu_1, \zu_2) \dd\pi(\zu_1,\zu_2),
  \end{equation}
  where $C_{\lambda, c_0}^u(\zu_1, \zu_2) = \min \left( \lambda, \; \frac{(h_1-h_2)^2}{\vert d_1 \vert+\vert d_2 \vert} \right)$. 
  Moreover, a solution to the fair regression problem in \cref{eq:erm_penalized} is given by \cref{eq:sol_from_transport_maps} with
  \vspace{-0.25cm}
  \begin{equation}\label{eq:y_star_OT}
      \begin{cases}
          T^+_{\lambda,\TV}(\zu_1) = T^-_{\lambda,\TV}(\zu_2) = \displaystyle \frac{\vert d_2 \vert h_1 + \vert d_1 \vert h_2}{\vert d_1 \vert+\vert d_2 \vert} & \text{if} \quad \frac{(h_1-h_2)^2}{\vert d_1 \vert+\vert d_2 \vert} \le \lambda, \\[3ex]
          T^+_{\lambda,\TV}(\zu_1) = h_1, \quad T^-_{\lambda,\TV}(\zu_2) = h_2 & \text{if} \quad \frac{(h_1-h_2)^2}{\vert d_1 \vert+\vert d_2 \vert} > \lambda,
      \end{cases}
  \end{equation}
  for all pairs $\zu_1=(h_1, d_1)$ and $\zu_2=(h_2, d_2)$ in the support of $\pi^*$.
\end{restatable}

Here again, while Problem~\eqref{eq:OT_TV_relaxed} is solved by a
(non-deterministic) transport plan, we show in
\cref{app:unaware_proof_relaxed_TV} that $T^+_{\lambda,\TV}(\zu_1)$ is well-defined in the sense that its value does not depend on the choice of $\zu_2$ for $(\zu_1,\zu_2)$ in the support
of $\pi^*$ (and similarly for $T^-_{\lambda,\TV}(\zu_2)$). 
In the exact fairness
limit ($\lambda\!\to\!\infty$), \cref{prop:relaxed_reduction} also recovers
previous results \citep{chzhen2020fairregressionwassersteinbarycenters,
divol2024demographicparityregressionclassification}. For 
$\lambda>0$, it
yields a novel characterization of the optimal predictor under a TV-fairness
penalty in the unawareness framework. 

\subsection{The implicit philosophies of relaxed fairness} \label{subsec:philosophies_relaxed_fairness}

The characterizations of the optimal almost-fair prediction function in
\cref{prop:relaxed_reduction,prop:tv_relaxed_reduction} 
highlight how the choice of fairness penalty impacts the
behaviour of fair algorithms. In this section, we contrast the two philosophies
induced by the geometry of the $\mathcal{W}_2$ and TV penalties.

\paragraph{Homogeneous gap reduction under Wasserstein relaxation.}
The $\mathcal W_2$ penalty penalizes the magnitude of the discrepancy between corresponding quantiles of the prediction in the two groups.
This implies that it is preferable for all individuals to be slightly moved together toward the fair prediction than for a few to experience massive disparity. 
Philosophically, this aligns with an \emph{egalitarian} view of fairness \citep{rawls1971theory}, which treats inequality as a shared societal burden to be mitigated globally across the population. 

\textbf{Heterogeneous parity under Total Variation relaxation.}
The TV penalty penalizes non-overlapping probability mass, ignoring distances in the prediction space.
Compromising by moving predictions slightly closer yields no benefit unless they are merged perfectly. 
As shown in \cref{prop:tv_relaxed_reduction}, this results in a hard-thresholding behavior. 
This implies that TV enforces parity where it is ``affordable'' (i.e.,
where the cost 
is low), and accepts total unfairness for the extremes of
the distribution. 
TV ensures that a maximized subset of the population achieves indistinguishability.
Ethically, this reflects a \emph{sufficientarian} approach \citep{frankfurt1987equality}, focusing on securing a plateau of equality for the majority.

\textbf{Illustration of the paradigms.}
We illustrate the effect of the $\mathcal W_2$ and TV relaxations on a synthetic
dataset for the unaware setting  (\cref{fig:unaware_evolution}) using the implementation detailed in \cref{sec:algorithm}.
The individual trajectories across relaxation levels reveal the distinct mechanisms of each penalty.
As the fairness constraint tightens, the $\mathcal W_2$ penalty shifts the distributions uniformly, with individual predictions linearly interpolating toward the barycenter. 
By contrast, the TV penalty exhibits hard-thresholding behavior. Predictions either remain at their original values (horizontal paths) or merge.
Furthermore, a saturation effect occurs: for sufficiently large $\lambda$, the solution stabilizes to the fair prediction.
For data generation details, visualizations of the mapping geometries,
and further examples, see \cref{app:empirical_illustrations}. 



\begin{figure*}[t!]
  \centering
  \includegraphics[width=1\textwidth]{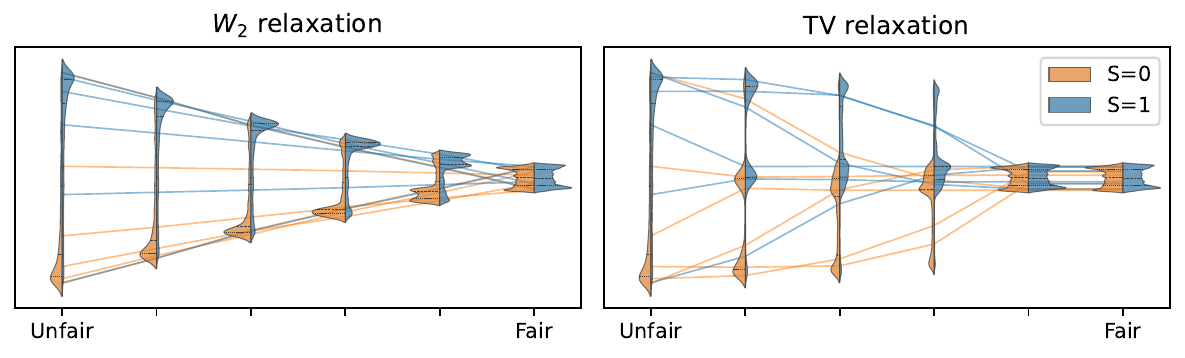}
  \vspace{-0.6cm}
  \caption{\textbf{Individual prediction trajectories for $\mathcal{W}_2$ vs. TV Relaxation.} 
  Evolution of marginal distributions and subset of individual prediction paths from Unfair (Bayes) to exactly Fair (100\%-0\% of initial unfairness). 
  The lines show $\mathcal W_2$ induces proportional shift for all individuals, while TV leaves predictions unchanged (horizontal lines) until a hard threshold triggers abrupt merging.
  }
  \label{fig:unaware_evolution}
\end{figure*}

\section{Algorithmic implementation}
\label{sec:algorithm}
In the previous section, we expressed the solution of the almost-fair regression problem in terms of an OT problem between population distributions. Building on this characterization, we now
propose an algorithm for \textbf{F}airness for \textbf{A}ware and \textbf{U}naware \textbf{S}e\textbf{T}tings (\textbf{FAUST}). We assume access to an estimator $\hat{\eta}$ of the Bayes regression function $\eta$ and an estimator $\hat{\mathbb{P}}(S =s \vert X =x)$ of the posterior group probabilities $\mathbb{P}(S =s \vert X =x)$ (which can be obtained using standard regression and probabilistic classification methods), and to an unlabeled dataset $(x_i, s_i)_{i\leq n}$. The
pseudocode is in \cref{app:algorithmic_implementation}.

\textbf{Estimation of $\hat{\Delta}(x)$.}
We first estimate $\hat{p}^+$, $\hat{p}^-$ of $p^+$ and $p^-$ using $(x_i, s_i)_{i\leq n}$, and define $\hat{\Delta}(x) = \tfrac{\hat{\mathbb{P}}(S =+ \vert X =x)}{\hat{p}^+}-\tfrac{\hat{\mathbb{P}}(S =- \vert X =x)}{\hat{p}^-}$. 
We then apply a threshold $\tau$ to partition training indices into positive ($\smash{\mathcal{I}^+ \!=\! \{i: \hat{\Delta}(x_i)\!>\!\tau\}}$) and negative ($\smash{\mathcal{I}^- \!=\! \{i: \hat{\Delta}(x_i)\!<\!-\tau\}}$) subsets based on the sign of $\hat{\Delta}$, forming the discrete supports of our empirical measures. 
Crucially, to account for finite-sample estimation errors, we normalize the local weights ($\hat{d}_i \!\propto \! \hat{\Delta}(x_i)$ for $i \in \mathcal{I}^+$ and $\hat{d}_i \!\propto \!-\hat{\Delta}(x_i)$ for $i \in \mathcal{I}^-$) to enforce mass conservation ($\sum_{i\in \mathcal{I}^+} \vert \hat{d}_i\vert \!=\! \sum_{i\in \mathcal{I}^-} \vert \hat{d}_i\vert \!=\! 1$).

\paragraph{OT plan and barycentric projection.}
From the measures $(\hat{\eta}_i, \hat{d}_i)_{i \in \mathcal{I}^+}$
and $(\hat{\eta}_j, \hat{d}_j)_{j \in \mathcal{I}^-}$, we construct the pairwise cost matrix $\mathbf{C}_{\lambda}\in \mathbb{R}_+^{\vert\mathcal{I}^+\vert \times \vert\mathcal{I}^-\vert}$ using the closed-form solutions derived in
\cref{prop:relaxed_reduction} (for the $\mathcal W_2$ relaxation) or
\cref{prop:tv_relaxed_reduction} (for the TV relaxation).
We then solve the discrete OT problem between the weighted measures $(\hat{\eta}_i, \hat{d}_i)_{i \in \mathcal{I}^+}$
and $(\hat{\eta}_j, \hat{d}_j)_{j \in \mathcal{I}^-}$ with corresponding weights $(\hat{d}_i)_{i \in \mathcal{I}^+}$ and $(\hat{d}_i)_{i \in \mathcal{I}^+}$ to obtain the optimal coupling matrix $\mathbf{\Pi}$. This reweighting allows us to estimate the distribution $\muU^+$ and $\muU^-$, as discussed in \Cref{app:algorithmic_implementation}. With a slight abuse of notation, let $T_{\lambda, \mathcal{D}}^s\left(z_i, z_j\right)$ denote the closed-form target mapping evaluated for the coupled pair $(z_i, z_j)$ (derived via \cref{eq:y-star} for $\mathcal{W}_2$ and \cref{eq:y_star_OT} for TV). 
The pseudo-label for point $i$ is then computed as
\vspace{-0.1cm}
\begin{equation*}
    \hat{f}_i = \frac{1}{\vert \hat{d}_i\vert} \sum_{j \in \mathcal{I}^{(-s)}} \mathbf{\Pi}_{i,j} T_{\lambda, \mathcal{D}}^s\left((\hat{\eta}_i, \hat{d}_i), (\hat{\eta}_j, \hat{d}_j)\right).
\end{equation*}
\textbf{Continuous mapping and final predictor.}
For out-of-sample 
predictions, we concatenate training points $(x_i)_{i\leq n}\!$ and fit a regressor $\smash{\hat{f}_{fair}}$ 
to predict $(\hat{f}_i)_{i\leq n}$.
The prediction for a new instance $x$ is obtained by evaluating $\hat{\eta}(x)$ and $\hat{\Delta}(x)$, and predicting $\hat{f}_{\text{fair}}(\hat{\eta}(x), \hat{\Delta}(x))$.

\textbf{Computational complexity.} 
The complexity is dominated by the $O(n^3 \log n)$ discrete OT step, computed using the network simplex implementation from the POT library \cite{flamary2021pot}. 
Despite this worst-case bound, our approach is more efficient than \citet{Fairreg} as it bypassed stochastic gradient descent. Empirically, it runs in seconds for $n \sim 10^4$, significantly outperforming FairReg, 
and can be further accelerated via mini-batching \cite{fatras2019learning} to scale to millions of samples without degrading accuracy or fairness (see our extensive benchmark in \cref{app:computation_times})

\section{Numerical experiments}\label{sec:experiments}



\subsection{Baseline and state-of-the-art comparison}

We benchmark our methods
against the unconstrained (and unfair) Empirical Risk Minimizer (ERM) and the state-of-the-art methods for unaware relaxed fair regression, FairReg \citep{Fairreg} and FRAPPÉ \citep{tifrea2024frappe}.


\textbf{Unaware Baselines.}
FairReg relies on a dual stochastic convex program to learn a randomized prediction rule over a discretized finite grid. FRAPPÉ relies on stochastic gradient descent to learn an additive post-hoc neural network correction that minimizes a fairness-regularized loss. Full algorithmic details and implementation parameters for both baselines are deferred to \cref{app:baseline_details}. We compare our FAUST method using the experimental setup of FairReg.

\textbf{Aware and proxy-based baselines.}
To establish the empirical lower bound for the accuracy-fairness trade-off, we include the aware OT models (FAUST-A $\mathcal{W}_2$ and FAUST-A TV), which have access to the true sensitive attribute $S$ at both training and inference time. We also benchmark against a naive plug-in estimator (Plug-In $\mathcal W_2$). 
A common heuristic \citep{awasthi2020equalized,chen2019fairness, Kallu_Shat} for the unaware setting is to predict a proxy $\hat{S}$ using a probabilistic classifier and apply the classical 1D aware OT mapping. Comparing these approaches highlights the cost of classification errors compared to our 2D framework. 
A common heuristic \citep{awasthi2020equalized,chen2019fairness, Kallu_Shat} for the unaware setting is to predict a proxy $\hat{S}$ using a probabilistic classifier and apply the classical 1D aware OT mapping. Comparing these approaches highlights the cost of classification errors compared to our 2D framework.


\textbf{Evaluation metrics.}
We evaluate predictive utility via Mean Squared Error (MSE). 
To quantify fairness, we report three metrics. 
First, the Wasserstein-2 ($\mathcal{W}_2$) and Total Variation (TV) unfairness.
Then, to ensure fair comparison with the FairReg baseline, we report $\text{KS}^{\mathcal{T}}$, which approximates the KS distance by computing the maximum absolute c.d.f. difference across a discrete 50-bin grid $\mathcal{T}$ (we demonstrate the robustness of our performance to this choice of discretization in \cref{app:metric_sensitivity}).

\subsection{Real-world datasets}

To evaluate our framework against the FairReg baseline \citep{Fairreg}, we use real-world datasets from their evaluation protocol: Communities and Crime, and Law School Admissions. 
Identical experiments and visualizations for a synthetic and the Adult (Census Income) dataset are provided in \cref{app:synthetic_dataset,app:adult_dataset} respectively, which show similar performances of our OT algorithms.

\textbf{Datasets descriptions.}
The Law School Admissions dataset comprises 20,649 student records from 163 U.S. law schools, featuring academic indicators (e.g., LSAT, GPA) and demographic variables. 
The regression task is to predict the student's standardized first-year GPA (\texttt{zfygpa}). 
Because the data exhibits racial imbalance (approximately 84.6\% White), we follow standard practice and define the binary sensitive attribute $S$ as White ($S=1$) versus Non-White ($S=0$).\\
The Communities and Crime dataset comprises 1,994 U.S. communities characterized by 128 socio-economic and demographic features. 
The regression task is to predict the violent crime rate per 100,000 population (\texttt{ViolentCrimesPerPop}). 
To construct $S$, we identify the majority racial demographic within each community and binarize it into White ($S=1$) versus non-White ($S=0$).

\begin{figure*}[t!]\label{fig:commu}
  \centering
  \includegraphics[width=1\textwidth]{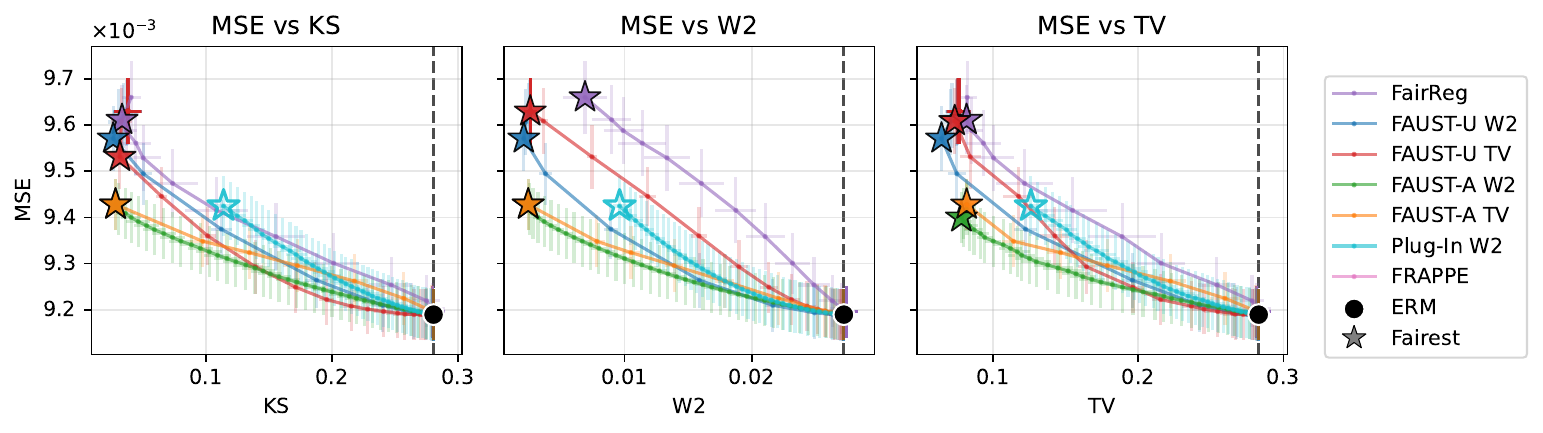}
\vfill
\vspace{-0.4cm}
  \includegraphics[width=1\textwidth]{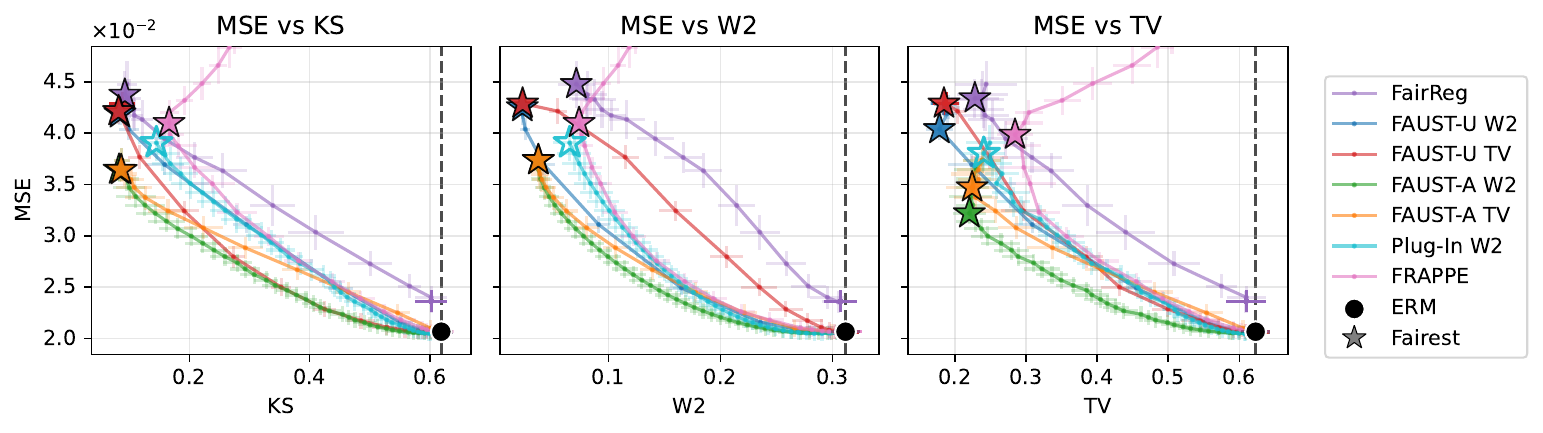}
  \vspace{-0.8cm}
  \caption{\textbf{Accuracy-fairness trade-offs on the Law School (top) and Communities (bottom) datasets.} 
  Relaxation trajectories are shown from the unconstrained ERM (black dot) to the exact fair models (stars). 
  The MSE of the constant predictor is $10.1 \times 10^{-3}$ on Law School, and $5.5 \times 10^{-2}$ on Communities. 
  Error bars indicate the standard deviation across 10 random data splits. FRAPPÉ is cropped when it falls outside the axis limits; full trajectories are provided in \cref{app:frappe_tradeoffs}.}
  \vspace{-0.6cm}
\end{figure*}

\textbf{Results.}
As shown in \cref{fig:commu}, our unaware OT framework (FAUST-U) consistently matches or outperforms the baselines. 
The naive estimator (Plug-In W2) degrades the accuracy-fairness trade-off by relying on a noisy proxy attribute, pushing its curve strictly worse than FairReg on the Communities dataset. Furthermore, the gradient-based FRAPPÉ baseline suffers from severe optimization instability: as the fairness constraint tightens,  MSE degrades while failing to reach parity (see \cref{app:frappe_tradeoffs} for full plots). \\
Conversely, by relying on exact geometric mappings rather than stochastic approximations, FAUST-U $\mathcal{W}_2$ tracks the theoretical aware bounds. It achieves superior trade-offs on $\mathcal{W}_2$ and remains competitive on KS and TV. Interestingly, $\mathcal W_2$-penalized models achieve better tradeoffs than TV-based ones \emph{even when evaluated on the TV metric}. Finally, our framework exhibits empirical robustness to finite-sample errors and miscalibration of the base estimators, as detailed in \cref{app:estimation_error}.

\textbf{Extended empirical validation.}
To validate our framework, we conduct an extensive set of supplementary experiments detailed in the Appendix. We provide complete evaluations on the Adult dataset (\cref{app:adult_dataset}) with trade-off plots, detail the implementations of the baselines (\cref{app:baseline_details,app:frappe_tradeoffs}), and illustrate the mapping geometries across varying levels of confounding (\cref{app:empirical_illustrations,app:synthetic_dataset}). We also stress-test the framework's stability, confirming its robustness to finite-sample estimation errors, artificial noise injected into the base estimators (\cref{app:estimation_error}), and varying metric discretization grids (\cref{app:metric_sensitivity}). Furthermore, we demonstrate that FAUST-U is model-agnostic, consistently maintaining its superior accuracy-fairness trade-offs across different unconstrained base estimators, including Random Forests and Gradient Boosting (\cref{app:base_models_benchmark}). Finally, we expand on our computational efficiency results: FAUST-U runs up to $28\times$ faster than FairReg and scales to massive datasets (e.g., $N=200,000$) in seconds via Minibatch OT, without degrading parity (\cref{app:computation_times}).

\section{Limitations and perspectives}
This work introduces a unified geometric framework for relaxed fair regression, yielding closed-form optimal predictors via OT maps. This characterization clarifies the philosophical implications of fairness penalties and enables a highly efficient, state-of-the-art post-processing algorithm for both aware and unaware settings.\\
We identify three main limitations. First, our closed-form reductions are restricted to TV and $\mathcal W_2^2$ penalties. While these provide a geometric lens to contrast sufficientarian and egalitarian philosophies, extending the framework to general $\OT_c$ costs could inspire novel fair algorithms. Second, extending our binary sensitive attribute formulation to multi-group fairness requires generalizing the Jordan decomposition to multi-marginal OT, posing an interesting open challenge. Finally, as a post-processing method, performance depends on the base estimators $\eta$ and $\Delta$. While proving non-asymptotic convergence rates for plug-in OT maps is notoriously challenging \citep{manole2024pluginestimationsmoothoptimal}, our framework bypasses the curse of dimensionality. By mapping the high-dimensional feature space $\mathcal{X}$ into a 2D pushforward space $(\eta(X), \Delta(X))$, the empirical 2D-to-1D OT problem converges at an almost parametric rate \citep{ajtai1984optimal}. Thus, given consistent base estimators, our algorithm should approximate the optimal solution at a dimension-independent statistical rate, with formal theoretical guarantees left for future work.

\section*{Acknowledgements and Disclosure of Funding}

This work was supported by Hi! PARIS and the ANR/France 2030 program (ANR-23-IACL-0005), as well as by the Fondation de l’Ecole Polytechnique. Additional support was provided by
the French National Research Agency (ANR) through grant ANR-23-CE23-0002, the PEPR IA FOUNDRY project (ANR-23-PEIA-0003),  the France 2030 program under the MacLeOD project (ANR-25-PEIA-0005) and the Labex Ecodec (ANR-11-LABEX-0047).

\clearpage
\bibliographystyle{plainnat}
\bibliography{biblio}
\clearpage
\appendix
\input{appendix}

\clearpage



\newpage

\end{document}

%% file: appendix.tex
\onecolumn
\clearpage
\appendix
\setcounter{page}{1}

\addcontentsline{toc}{section}{Appendix}

\startcontents[appendix]
\begin{center}
    \Large\bfseries Appendix Table of Contents
\end{center}
\hrule
\vspace{1em}
\printcontents[appendix]{}{1}{\setcounter{tocdepth}{2}}
\vspace{1em}
\hrule
\vspace{3em}

\clearpage

\medskip

\section{Proofs of propositions and lemmas} \label{app:proofs}

\paragraph{Notation for the proofs.} 
Throughout the appendix, we define the 2D mapped feature space as $\mathcal{Z} := \mathbb{R} \times \mathbb{R}$. An element in this space is denoted as $\zu=(h,d) \in \mathcal{Z}$, where $h \in \mathbb{R}$ represents the initial base prediction $\eta(x)$, and $d \in \mathbb{R}$ represents the local density ratio $\Delta(x)$. For readability, we denote its absolute value as $a = |d|$. Additionally, for a measure $\mu \in \mathcal{P}(\mathbb{R})$, we denote its generalized quantile function as $Q_\mu(t) = \inf\{y \in \mathbb{R} : F_\mu(y) \ge t\}$ for $t \in (0,1)$.

We begin the proof of \cref{lem:reduc_OT} by showing that if Problem (\ref{eq:P-lambda-general}) is solved by transport maps, then the solution to Problem (\ref{eq:erm_penalized}) is indeed given as in \eqref{eq:sol_from_transport_maps}. This is done in \cref{lem:equiv_erm_ot}. The remainder of the proof will then consist in proving that for the choice of fairness penalties considered, and under \cref{ass_non_atomic}, Problem (\ref{eq:P-lambda-general}) is indeed solved by transport maps. To do this, we show that Problem (\ref{eq:P-lambda-general}) reduces to an OT problem between measures $\muU^+$, $\muU^-$ (\cref{lem:general_reduction}), before showing in \cref{app:unaware_proof_relaxed_W2} and \cref{app:unaware_proof_relaxed_TV} that for the $\mathcal{W}^2_2$ and the TV fairness penalties respectively, this problem is indeed solved by deterministic transport maps. Along the way, we prove \cref{prop:relaxed_reduction} in \cref{app:unaware_proof_relaxed_W2}, and \cref{prop:tv_relaxed_reduction} in \cref{app:unaware_proof_relaxed_TV}.

\subsection{Tightness of the upper bound between \texorpdfstring{$\mathcal{U}_{\mathcal{D}}$}{Unfairness}  and the conditional disparity}
\label{app:tightness_bound}

In \cref{def:generalized_unfairness}, we defined the fairness penalty $\mathcal{U}_{\mathcal{D}}$ via the normalized variations of the demographic distributions rather than the true conditional disparity. We demonstrate here that the gap between these two quantities is either zero or tightly bounded.

Let $\nu^+ = \frac{1}{m}f_{\sharp} (\chi^+-\chi^-)_+$ and $\nu^- = \frac{1}{m}f_{\sharp} (\chi^+-\chi^-)_-$ be the normalized pushforwards. Let $a'=f_\sharp \chi^+$ and $b'=f_\sharp \chi^-$. Observe that $a' \geq m \nu^+$ and $b' \geq m \nu^-$, with the shared mass being $\sigma := a' - m\nu^+ = b' - m\nu^-$. We want to compare our proxy $\mathcal{U}_{\mathcal{D}}(f) = m\mathcal D(\nu^+,\nu^-)$ to the true disparity $\mathcal D(a',b')$. 

First, if $\mathcal D$ is a norm induced by a distance (as is the case for the Total Variation distance), the gap is zero. Indeed, since $\sigma$ cancels out, we have $m(\nu^+ - \nu^-) = a' - b'$. Because TV scales linearly with mass, $\mathcal{U}_{\TV}(f) = m \TV(\nu^+, \nu^-) = \TV(a',b')$.

A gap can only exist for the $\mathcal{W}_2$ penalty. In that case, we have $\mathcal U_{\mathcal W_2^2}(f) = m\mathcal{W}_2^2(\nu^+,\nu^-)\geq \mathcal{W}_2^2(a',b') = \mathcal W_2^2(\Law(f(X) \vert S= +), \Law(f(X) \vert S=-))$. This inequality naturally holds because a valid transport plan for the full distributions $a'$ and $b'$ can be constructed by leaving the shared mass $\sigma$ stationary (incurring zero cost for this $1-m$ fraction) and exclusively transporting the divergent mass $m \nu^+$ to $m \nu^-$. Since the true distance $\mathcal{W}_2^2(a',b')$ evaluates the infimum over all possible plans, its cost is upper-bounded by the cost of this specific construction, which is $m\mathcal{W}_2^2(\nu^+,\nu^-)$. A reverse inequality exists and takes the form
\begin{equation*}
\mathcal{U}_{\mathcal{W}_2^2}(f) - \mathcal W_2^2(\Law(f(X) \vert S= +), \Law(f(X) \vert S=-)) \lesssim B^2 \min(m,1-m),
\end{equation*}
where $m$ is the normalization constant defined earlier, and we assume that $f$ takes its values in $[-B,B]$. This gap is zero in two extreme cases: if $m=1$ (disjoint support case, where the sensitive attribute can be recovered perfectly from $X$) and if $m=0$ (where predictions are fair). Moreover, this upper bound on the gap is computable in practice, since it relies only on $m$, a quantity evaluated during our metric computations.

\begin{proof}[Proof sketch]
We use the dual formulation of OT: $\mathcal{W}_2^2(\mu,\nu) = \sup_{\phi\oplus \psi\leq |x-y|^2 } \int \phi \dd\mu+\int\psi \dd\nu$. Let $(\phi,\psi)$ be an optimal pair for $(m\nu^+, m\nu^-)$. Since it is suboptimal for $(a',b')$, we have
$$
m\mathcal{W}_2^2(\nu^+,\nu^-)-\mathcal{W}_2^2(a',b')\leq \int \phi \dd(m\nu^+-a')+\int \psi \dd(m\nu^--b') = \int g \dd\sigma,
$$
where $g=-\phi-\psi$. Optimal potentials $\phi$ and $\psi$ inherit the modulus of continuity of the cost function. Here, the cost is $\mathcal O(B)$-Lipschitz, which implies that $|g|= \mathcal O(B^2)$. Thus, the gap is upper bounded by $\mathcal O(B^2)|\sigma| = \mathcal O(B^2)(1-m)$. Moreover, $m\mathcal{W}_2^2(\nu^+,\nu^-)-\mathcal{W}_2^2(a',b')\leq m\mathcal{W}_2^2(\nu^+,\nu^-)\leq 4B^2 m$ since the cost is bounded by $4B^2$.
\end{proof}

\subsection{Equivalence between \texorpdfstring{\eqref{eq:erm_penalized}}{ERM penalized}  and (\texorpdfstring{\ref{eq:P-lambda-general}}{P-lambda-general}), assuming deterministic optimal transport maps} 
\begin{lemma}\label{lem:equiv_erm_ot}
If Problem $(\ref{eq:P-lambda-general})$ is solved by deterministic transport maps $T^+_{\lambda,\mathcal{D}}$ and $T^-_{\lambda,\mathcal{D}}$ which push $\muU^+$ and $\muU^-$ to the minimizing measures $\nu^+$ and $\nu^-$, respectively, then the solution of the relaxed fair regression problem $\eqref{eq:erm_penalized}$ is given by
\begin{align*}
    f_{\lambda, \mathcal{D}}^*(x) = 
    \begin{cases}
        T^+_{\lambda,\mathcal{D}}(\eta(x),\Delta(x)) \quad \text{if} \quad \Delta(x) > 0, \\
        T^-_{\lambda,\mathcal{D}}(\eta(x),\Delta(x)) \quad \text{if} \quad \Delta(x) < 0,\\
        \eta(x) \quad \text{else}.
    \end{cases}
\end{align*}
\end{lemma}
\begin{proof}
Let us start by setting some additional notations. We denote $\chi=p^+\chi^+ + p^-\chi^-$ the law of $X$. 
Let $|\chi^+-\chi^-|$ be the  variation of $\chi^+-\chi^-$, then for $m = \frac{\TV(\chi^+,\chi^-)}{2}$, we define the normalized  Jordan decomposition of $\chi^+-\chi^-$ through
\[
\begin{cases}
m(\chi^+-\chi^-)_+ = \frac{1}{2}(|\chi^+-\chi^-|+\chi^+-\chi^-), \\
m(\chi^+-\chi^-)_- = \frac{1}{2}(|\chi^+-\chi^-|-\chi^++\chi^-), 
 \end{cases}
 \]
that is $(m(\chi^+-\chi^-)_+, m(\chi^+-\chi^-)_-)$ is the Jordan decomposition of $\chi^+-\chi^-$. Then, we have $\chi^+ - \chi^- = m(\chi^+-\chi^-)_+ - m(\chi^+-\chi^-)_-$, and
\begin{align*}
    \Delta(x) &= \frac{\mathbb{P}(S = +\vert X)}{p^+} - \frac{\mathbb{P}(S = -\vert X)}{p^-}\\
    & = \frac{\dd(\chi^+ - \chi^-)}{\dd\chi}(x)\\
    & = m\frac{\dd(\chi^+ - \chi^-)_+}{\dd\chi}(x) - m\frac{\dd(\chi^+ - \chi^-)_-}{\dd\chi}(x).
\end{align*}
Finally, we also define $\mathcal{X}_+ = \{x \in \mathcal{X} : \Delta(x) >0 \}$, $\mathcal{X}_- = \{x \in \mathcal{X} : \Delta(x) <0 \}$, and $\mathcal{X}_= = \{x \in \mathcal{X} : \Delta(x) =0 \}$ : $\mathcal{X}_+$ is the support of $(\chi^+-\chi^-)_+$, and $\mathcal{X}_-$ is the support of $(\chi^+-\chi^-)_-$.  With these notations, for all prediction function $f$,
\begin{align*}
\mathbb{E}\left[(\eta(X) - f(X))^2\right] =& \int (\eta(x)-f(x))^2 \dd\chi(x)\\
=&\int_{\mathcal{X}_+}  (\eta(x)-f(x))^2 \frac{\dd\chi}{\dd (\chi^+-\chi^-)_+}(x) \dd (\chi^+-\chi^-)_+(x) \\
&\quad\qquad + \int_{\mathcal{X}_-} (\eta(x)-f(x))^2 \frac{\dd\chi}{\dd (\chi^+-\chi^-)_-}(x) \dd (\chi^+-\chi^-)_-(x)\\
&\quad\qquad+  \int_{\mathcal{X}_=}  (\eta(x)-f(x))^2 \dd \chi(x)\\
=&\int_{\mathcal{X}_+}  (\eta(x)-f(x))^2 \frac{m}{\vert \Delta(x)\vert } \dd (\chi^+-\chi^-)_+(x) \\
&\quad\qquad + \int_{\mathcal{X}_-} (\eta(x)-f(x))^2 \frac{m}{\vert \Delta(x)\vert} \dd (\chi^+-\chi^-)_-(x)\\
&\quad\qquad+  \int_{\mathcal{X}_=}  (\eta(x)-f(x))^2 \dd \chi(x)\\
=&m\int_{\mathcal{X}_+}  c_u((\eta(x),\Delta(x)),f(x))\dd (\chi^+-\chi^-)_+(x) \\
&\quad\qquad+ m\int_{\mathcal{X}_-} c_u((\eta(x),\Delta(x)),f(x))\dd (\chi^+-\chi^-)_-(x)\\
&\quad\qquad+  \int_{\mathcal{X}_=}  (\eta(x)-f(x))^2 \dd \chi(x).
\end{align*}
Recall that $\muU^+ = (\eta, \Delta)_{\#}(\chi^+-\chi^-)_+$, and define $\tilde{\nu}^+ = f_{\#}(\chi^+-\chi^-)_+$. Then, $(\eta(X), \Delta(X), f(X))$ defines a coupling between $\muU^+$ and $\tilde{\nu}^+$, so
\begin{align*}
    \int_{\mathcal{X}_+}  c_u((\eta(x),\Delta(x)),f(x))\dd (\chi^+-\chi^-)_+(x) \geq \OT_{c_u}(\muU^+, \tilde{\nu}^+).
\end{align*}
Defining similarly $\tilde{\nu}^- = f_{\#}(\chi^+-\chi^-)_-$, we have
\begin{align*}
    \int_{\mathcal{X}_-}  c_u((\eta(x),\Delta(x)),f(x))\dd (\chi^+-\chi^-)_-(x) \geq \OT_{c_u}(\muU^-, \tilde{\nu}^-).
\end{align*}
Now, by definition, $\mathcal{U}_{\mathcal{D}}(f) = m\OT_{c_{\mathcal{D}}}(\tilde{\nu}^+,\tilde{\nu}^-)$, so we find
\begin{equation}\label{eq:oneside}
\begin{split}
    &\underset{f}{\min}\ \mathbb{E}\left[(\eta(X) - f(X))^2\right] + \lambda \mathcal{U}_{\mathcal{D}}(f) \\
    &\qquad \geq m\inf_{\tilde{\nu}^+, \nu^- \in \mathcal{P}(\mathbb{R})} \left( \OT_{c_u}(\muU^+, \tilde{\nu}^+) + \OT_{c_u}(\muU^-, \tilde{\nu}^-)+ \lambda \OT_{c_{\mathcal{D}}}(\tilde{\nu}^+,\tilde{\nu}^-)\right).
    \end{split}
\end{equation}
Now, if the OT problem on the right hand side of \eqref{eq:oneside} is solved by deterministic transport maps $T^+_{\lambda,\mathcal{D}}$ and $T^-_{\lambda,\mathcal{D}}$ which push $\muU^+$ and $\muU^-$ to the minimizing measures $\nu^+$ and $\nu^-$, respectively, then for $f^*_{\lambda,\mathcal D}$ given in Equation \eqref{eq:sol_from_transport_maps} we have
\begin{align*}
    \int_{\mathcal{X}_+}  c_u((\eta(x),\Delta(x)),f^*_{\lambda,\mathcal D}(x))\dd (\chi^+-\chi^-)_+(x) =& \OT_{c_u}(\muU^+, \nu^+),\\
    \int_{\mathcal{X}_+}  c_u((\eta(x),\Delta(x)),f^*_{\lambda,\mathcal D}(x))\dd (\chi^+-\chi^-)_-(x) =& \OT_{c_u}(\muU^-, \nu^-),\\
    \int_{\mathcal{X}_=}  (\eta(x)-f^*_{\lambda,\mathcal D}(x))^2 \dd \chi(x) =& 0
\end{align*}
so 

\begin{align*}
    &\mathbb{E}\left[(\eta(X) - f(X)^*)^2\right] + \lambda \mathcal{U}_{\mathcal{D}}(f^*) \\
    &\qquad\qquad = m\underset{\tilde{\nu}^+, \tilde{\nu}^-}{\min}\left( \OT_{c_u}(\muU^+, \tilde{\nu}^+) + \OT_{c_u}(\muU^-, \tilde{\nu}^-)+ \lambda \OT_{c_{\mathcal{D}}}(\tilde{\nu}^+,\tilde{\nu}^-)\right).
\end{align*}

This concludes the proof of Lemma \ref{lem:equiv_erm_ot}.
\end{proof}

\subsection{Reduction of (\ref{eq:P-lambda-general}) to an OT problem between $\muU^+$ and $\muU^-$} \label{app:general_reduction} 



As an introductory result, we use the Gluing lemma to show that the minimization problem (\ref{eq:P-lambda-general}) can be recast as  a simple optimal transport problem for some appropriate cost.

\begin{restatable}{lemma}{lemGeneralReduction}
  \label{lem:general_reduction}
  Let $\fcost:\R\times \R\to \R$ be lower semi-continuous and bounded from below. 
  For any pair of points $\zu_1 = (h_1, d_1)$ and $\zu_2 = (h_2, d_2)$, 
  we define the pointwise relaxed cost
  \begin{equation}\label{eq:def_C_lambda}
    C_{\lambda, \mathcal{D}}(\zu_1, \zu_2) := \inf_{y_1, y_2 \in \mathbb{R}} \left\{ \frac{(h_1 - y_1)^2}{|d_1|} + \frac{(h_2 - y_2)^2}{|d_2|} + \lambda \fcost(y_1, y_2) \right\}.
  \end{equation}
  
  Then, the problem (\ref{eq:P-lambda-general}) is  equivalent to the optimal transport problem using this pointwise cost, in the sense that
    \begin{equation}\label{eq:OT-reduced}
    \inf (P_{\lambda, \mathcal D}) \;=\; \OT_{C_{\lambda,\mathcal D}}(\muU^+,\muU^-) \;=\;\min_{\pi\in\Pi(\muU^+,\muU^-)} \int_{\mathbb{R} \times \mathbb{R}} C_{\lambda, \mathcal{D}}(\zu_1,\zu_2)\,d\pi(\zu_1,\zu_2).
  \end{equation}
  Moreover, let $\pi^*$ be a minimizer of the right-hand side, and let $(y_1^*,y_2^*): \mathbb{R}\times \mathbb{R}\to \R^2$ be a measurable map such that $(y_1^*(z_1,z_2),y_2^*(z_1,z_2))$ realizes the infimum in \eqref{eq:def_C_lambda} for all $(z_1,z_2)\in \mathbb{R}^2\times \mathbb{R}^2$. Then, $(\nu^{+*},\nu^{-*}) = (y_1^*,y_2^*)_{\sharp} \pi_*$ is a minimizer of (\ref{eq:P-lambda-general}).
\end{restatable}

\begin{proof}[Proof of \cref{lem:general_reduction}] 
  To prove this equivalence, we first establish that the pointwise optimization constitutes a lower bound for the global problem, and then demonstrate that this bound is exactly achieved.

  \paragraph{Lower bound using the Gluing lemma.}
  Let $\nu^+, \nu^- \in \mathcal P(\R)$. 
  The cost $c_u$ is lower semi-continuous and bounded from below and we assume that the fairness penalty $\fcost$ is also lower semi-continuous and bounded from below. 
  Therefore, by Theorem 1.5 from \cite{Santambrogio2015}, optimal transport plans exist for all three costs $c_u$, $\fcost$ and $C_{\lambda, \mathcal{D}}$.
  Let $\gamma_1 \in \Pi(\muU^+, \nu^+)$ and $\gamma_2 \in \Pi(\muU^-, \nu^-)$ be  optimal transport plans for the cost $c_u$, and let $\gamma_3 \in \Pi(\nu^+, \nu^-)$ be the optimal transport plan for the fairness penalty $\fcost$. 
  
  By the Gluing Lemma (e.g., Lemma 5.5 in \citet{Santambrogio2015}), we can construct a joint probability measure $\rho$ over the product space $\mathbb{R}^2 \times \mathbb{R}^2 \times \mathbb{R} \times \mathbb{R}$, with coordinates denoted $(\zu_1, \zu_2, y_1, y_2)$, such that its respective two-dimensional marginals match $\gamma_1$, $\gamma_2$, and $\gamma_3$.

  The global penalized objective (\ref{eq:P-lambda-general}) can then be evaluated as a single integral over this joint measure $\rho$
  \begin{equation*}
      \int \left[ c_u(\zu_1, y_1) + c_u(\zu_2, y_2) + \lambda \fcost(y_1, y_2) \right] \dd\rho(\zu_1, \zu_2, y_1, y_2).
  \end{equation*}

  For every tuple $(\zu_1, \zu_2, y_1, y_2)$ in the support of $\rho$, the integrand is bounded from below by its pointwize infimum over all possible target values $y_1, y_2 \in \mathbb{R}$ for the given pair $(\zu_1, \zu_2)$. 
  Substituting this pointwise minimum, $C_{\lambda, \mathcal{D}}(\zu_1, \zu_2)$, into the integral yields a lower bound. 
  Furthermore, letting $\pi \in \Pi(\muU^+, \muU^-)$ denote the projection of $\rho$ onto its first two variables $(\zu_1, \zu_2)$, we obtain
  \begin{equation*}
      \OT_{c_u}(\muU^+, \nu^+) + \OT_{c_u}(\muU^-, \nu^-) + \lambda \OT_{\fcost}(\nu^+, \nu^-) \ge \int_{\mathbb{R}^2 \times \mathbb{R}^2} C_{\lambda, \mathcal{D}}(\zu_1, \zu_2) \dd\pi(\zu_1, \zu_2).
  \end{equation*}
  Taking the infimum over all valid source couplings $\pi$ establishes that $\min_{\pi} \int C_{\lambda, \mathcal{D}} \dd\pi$ is a global lower bound for (\ref{eq:P-lambda-general}).

  \paragraph{Existence and achieving the bound.}
  The cost $C_{\lambda, \mathcal D}$ is lower semi-continuous as an infimum of lower semi-continuous costs. It is also clearly lower bounded since $c_{\mathcal{D}}$ is bounded from below.
  Therefore, we can apply Theorem 1.5 of \citet{Santambrogio2015}, which gives us an optimal  coupling $\pi^* \in \Pi(\muU^+, \muU^-)$. 
We denote the objective function 
  \begin{equation*}
    \Phi_{\zu_1, \zu_2}(y_1,y_2) :=  \frac{(h_1 - y_1)^2}{|d_1|} + \frac{(h_2 - y_2)^2}{|d_2|} + \lambda \fcost(y_1, y_2).
  \end{equation*}
This function is lower semi-continuous and coercive, and therefore has a nonempty, compact set of minimizers.
  Since  $\Phi_{\zu_1, \zu_2}(y_1,y_2)$ is measurable with respect to the variable $(\zu_1,\zu_2)$,  the Measurable Maximum Theorem (see Theorem 18.19 in \citet{Aliprantis2006}) applies, which implies that  there exists a measurable map  $(y_1^*,y_2^*): \mathbb{R}^2\times \mathbb{R}^2\to \R^2$ such that $(y_1^*(z_1,z_2),y_2^*(z_1,z_2))$ is a minimizer of $ \Phi_{\zu_1, \zu_2}$ for all $(\zu_1,\zu_2)\in \mathbb{R}\times \mathbb{R}$.


  We can therefore push forward $\pi^*$ via the map $(\zu_1, \zu_2) \mapsto (\zu_1, \zu_2, y_1^*, y_2^*)$ and construct a  joint measure $\rho^*$. 
  The marginals of $\rho^*$ over its third and fouth coordinates ($y_1$ and $y_2$) define  target distributions $\nu^{+*}$ and $\nu^{-*}$. 
  Because $\rho^*$ constitutes a valid joint transport plan for these targets, the optimal transport cost is bounded above by the integral over $\rho^*$, yielding
  \begin{equation*}
      \OT_{c_u}(\muU^+, \nu^{+*}) + \OT_{c_u}(\muU^-, \nu^{-*}) + \lambda \OT_{\fcost}(\nu^{+*}, \nu^{-*}) \le \int_{\mathbb{R}^2 \times \mathbb{R}^2} C_{\lambda, \mathcal{D}}(\zu_1, \zu_2) \dd\pi^*(\zu_1, \zu_2).
  \end{equation*}
  Combined with the previously established lower bound, this proves the equality \eqref{eq:OT-reduced}, and that $\nu^{+*}$ and $\nu^{-*}$ are minimizers of (\ref{eq:P-lambda-general}).

Moreover, the previous argument clearly shows that for \emph{any} optimal transport plan $\pi^*$ and \emph{any} choice of measurable minimizers $(y_1^*,y_2^*)$, the measures $\nu^{+*}$ and $\nu^{-*}$ obtained through the previous construction are minimizers of (\ref{eq:P-lambda-general}), concluding the proof of the lemma.
  \end{proof}

\subsection{Proof of \texorpdfstring{\cref{prop:relaxed_reduction} }{relaxed-red}(FAUST-U \texorpdfstring{$\mathcal W_2$)}{W_2}} \label{app:unaware_proof_relaxed_W2}

\proprelaxedWtwo*
\begin{proof}
By \cref{lem:general_reduction}, the global problem reduces to evaluating the pointwise relaxed cost for $\mathcal{D}(y_1, y_2) = (y_1 - y_2)^2$. 
Fix $\zu_i=(h_i,d_i)$ with absolute density weights $a_i=|d_i| > 0$, and minimize
\begin{equation*}
   \Phi_{\zu_1, \zu_2}(y_1,y_2)=\frac{(h_1-y_1)^2}{a_1}+\frac{(h_2-y_2)^2}{a_2}+\lambda (y_1-y_2)^2.
\end{equation*}

Because $\Phi:= \Phi_{\zu_1, \zu_2}$ is strictly convex with respect to $y_1$ and $y_2$, the optimal pointwise targets are characterized by the first-order optimality conditions, which can be expressed as 
\[ 
\begin{cases}
  \frac{2(y_1-h_1)}{a_1}+2\lambda(y_1-y_2)=&0,\\
\frac{2(y_2-h_2)}{a_2}-2\lambda(y_1-y_2)=&0.
\end{cases}
\]

Solving the system yields the optimal pointwise targets $y_1^*$ and $y_2^*$. Let $u:=y_1^*-y_2^*$. From the optimality conditions we have
\[
\begin{cases}
  y_1^*-h_1=&-\lambda a_1 u,\\
  y_2^*-h_2=&\lambda a_2 u,
\end{cases}
\]
Subtracting these identities gives $u = (h_1-h_2)-\lambda(a_1+a_2)u$, which resolves to $u=\frac{h_1-h_2}{1+\lambda(a_1+a_2)}$. On the one hand, plugging this value into the first-order conditions yield
\begin{align}\label{eq:expression_OT_maps_W2}
\begin{cases}
    y_1^* &= \frac{h_1(1 + \lambda a_2) + \lambda h_2 a_1}{1+\lambda(a_1+a_2)},\\
    y_2^* &= \frac{h_2(1 + \lambda a_1) + \lambda h_1 a_2}{1+\lambda(a_1+a_2)}.
\end{cases}
\end{align}
On the other hand, plugging these back into $\Phi$ at the optimum yields
\begin{align*}
  C_{\lambda, c_2}(\zu_1,\zu_2) =& \Phi(y_1^*,y_2^*) = \frac{(\lambda a_1 u)^2}{a_1}+\frac{(\lambda a_2 u)^2}{a_2}+\lambda u^2 \\
  =& \lambda u^2\left(1+\lambda(a_1+a_2)\right) \\
  =& \lambda \frac{(h_1-h_2)^2}{(1+\lambda(a_1+a_2))^2}\,(1+\lambda(a_1+a_2)) \\
  =& \frac{\lambda}{1+\lambda(a_1+a_2)}(h_1-h_2)^2.
\end{align*}
According to \cref{lem:general_reduction}, a solution $(\nu^{+*},\nu^{-*})$ of (\ref{eq:P-lambda-general}) is given by the marginal distributions of $(y_1^*(Z_1,Z_2),y_2^*(Z_1,Z_2))$ where $(Z_1,Z_2)\sim \pi^*$ and $\pi^*$ is as in \eqref{Eqref:OT_W2_relaxed}. Our next goal is to show that under \cref{ass_non_atomic}, the random variable $y_1^*(Z_1,Z_2)$ does not depend on $Z_2$. To do so, we closely follow the arguments of the Step 2 of Theorem 4 in \cite{divol2024demographicparityregressionclassification}. The only difference between the two settings is that the cost $C_{\infty,c_2}(z_1,z_2) = (h_1-h_2)^2/(a_1+a_2)$ is replaced by the cost $C_{\lambda,c_2}(z_1,z_2)$. These two costs share the essential property that they are quadratic whenever $a_1$ and $a_2$ are fixed, which is enough to make the arguments of the proof in  \citet{divol2024demographicparityregressionclassification} work.

Let us give some additional details. The proof is based on  Kantorovich duality, see Section 2.1 in \cite{divol2024demographicparityregressionclassification} for an introduction to this notion. There exists a function $\psi$, called a  Kantorovich potential, that satisfies the following properties. Let $\phi:\zu_2\mapsto \sup_{\zu_1\in \mathcal Z} (\psi(\zu_1)-C_{\lambda,c_2}(z_1,z_2))$ be the $C_{\lambda,c_2}$-transform of $\psi$. Then, 
\begin{enumerate}
    \item for all $\zu_1$, $\psi(\zu_1)=\inf_{\zu_2} (\phi(\zu_2)+C_{\lambda,c_2}(z_1,z_2))$;
    \item $\pi^*$ is supported on the $C_{\lambda,c_2}^u$-subdifferential of a Kantorovich potential $\psi$, that is the set
\[
\Gamma = \{(\zu_1,\zu_2)\in \mathcal Z\times \mathcal Z:\ \psi(\zu_1)-\phi(\zu_2)= C_{\lambda,c_2}(z_1,z_2)\}.
\]
\end{enumerate}

The arguments used to prove Lemma 3 in \citet{divol2024demographicparityregressionclassification} can be reproduced to show that under \cref{ass_non_atomic},  the Kantorovich potential $\psi$ is differentiable in the direction $h_1$ for  $\muU^+$-almost every $\zu_1=(h_1,d_1)$ (this follows from the fact that, when $(d_1,d_2)$ is fixed, the cost $C_{\lambda,c_2}(z_1,z_2)$ is proportional to $(h_1-h_2)^2$).

Let $\zu_1$ be a point in the support of $\mu_{\eta,\Delta}^+$ where $\partial_{h_1}\psi$ exists, and let $\zu_2$ be such that $(\zu_1,\zu_2)$ is in the support of $\pi^*$ (and therefore in $\Gamma$). The definitions of $\phi$ and $\Gamma$ imply that the function $z\mapsto \psi(\zu)-C_{\lambda,c_2}(z,z_2)$ attains its maximum at $\zu_1$. Differentiating, we find that
\begin{align*}
    \partial_{h_1}\psi(\zu_1) =& \partial_{h_1}C_{\lambda, c_2}^u(\zu_1, \zu_2)  =\frac{2\lambda(h_1-h_2)}{1+\lambda(a_1+a_2)}= 2\lambda u.
\end{align*}
Isolating $u = \frac{1}{2\lambda}\partial_{h_1}\psi(\zu_1)$ and substituting it into the optimality condition for $y_1^*$ gives
\begin{align*}
    y_1^*(\zu_1,\zu_2) =& h_1 - \lambda a_1 u = h_1 - \frac{\vert d_1 \vert}{2} \partial_{h_1}\psi(\zu_1).
\end{align*}
In other words, the value $y_1^*(\zu_1,\zu_2)$ does not depend on the choice of $\zu_2$ with $(\zu_1,\zu_2)$ in the support of $\pi^*$ for $\mu_{\eta,\Delta}^+$-almost all $\zu_1$, which is what we wanted to show. A symmetric argument holds for $y_2^*$, concluding to the existence of deterministic maps $T_{\lambda, \mathcal W_2^2}^+$ and $T_{\lambda, \mathcal W_2^2}^-$ solving (\ref{eq:P-lambda-general}).

Together with \cref{lem:equiv_erm_ot}  and \cref{lem:general_reduction}, this proves \cref{lem:reduc_OT} for the $\mathcal{W}_2^2$ unfairness penalty. Moreover, the proof of \cref{prop:relaxed_reduction} is concluded by noticing that the expression of the optimal transport maps are given by \cref{eq:expression_OT_maps_W2}. 
\end{proof}

Let us note that \Cref{prop:relaxed_reduction} implies in particular the result of \Cref{lem:reduc_OT} in the case $c_{\mathcal D}=c_2$.

\begin{remark}[Asymptotic compatibility with exact fairness]\label{rem:lambda_to_infty}
As a sanity check, let us  verify that the solution of the problem (\ref{eq:P-lambda-general}) converge to the solution of the fair regression problem  \eqref{eq_Pinfinity} as $\lambda\to\infty$. 
  
Consider the relaxed objective (\ref{eq:P-lambda-general}) and its reduced form \cref{eq:OT-reduced} with the pointwise cost $C_{\lambda, c_2}^u(\zu_1,\zu_2)$.
  Let $\zu_i=(h_i,d_i)$ with $a_i=\vert d_i\vert>0$.
  The inner minimization admits the explicit map evaluations given in \cref{eq:y-star}, the gap between the predictions strictly vanishes
  \begin{equation}\label{eq:collapse_y1y2}
  T^+_{\lambda,\mathcal W_2^2}(\zu_1) - T^-_{\lambda,\mathcal W_2^2}(\zu_2) = \frac{h_1-h_2}{1+\lambda(a_1+a_2)} \xrightarrow[\lambda\to\infty]{}0,
  \end{equation}
  and moreover, both maps naturally converge to the weighted barycenter
  \begin{equation}\label{eq:collapse_to_bary}
  T^+_{\lambda,\mathcal W_2^2}(\zu_1),\ T^-_{\lambda,\mathcal W_2^2}(\zu_2)
  \xrightarrow[\lambda\to\infty]{}
  \frac{a_2h_1+a_1h_2}{a_1+a_2}
  =:y^*(\zu_1,\zu_2),
  \end{equation}
  where $y^*(\zu_1,\zu_2)$ is the minimizer of the unrelaxed two-to-one cost $\inf_{y}\{c_u(\zu_1,y)+c_u(\zu_2,y)\}$.
  
  At the level of pairwise costs, the closed form expression of the cost yields the pointwise convergence
  \begin{equation*}
  C_{\lambda, c_2}^u(\zu_1,\zu_2)=\frac{\lambda}{1+\lambda(a_1+a_2)}(h_1-h_2)^2
  \;\xrightarrow[\lambda\to\infty]{}\;
  \frac{(h_1-h_2)^2}{a_1+a_2}
  =:C_\infty(\zu_1,\zu_2),
\end{equation*}
  which is the unrelaxed OT cost. 

  Let now $\pi_\lambda^*\in\Pi(\muU^+,\muU^-)$ be an optimal coupling for the reduced problem in \cref{eq:OT-reduced}. 
  Denote by $V_\lambda$  the optimal value in \cref{eq:OT-reduced}. Then $V_\lambda$ is nondecreasing in $\lambda$ and  satisfies $V_\lambda\leq V_\infty$, where $V_\infty$ is the optimal value of the unrelaxed problem (since $C_{\lambda,c_2}^u\leq C_\infty$). We claim that we actually have
  \[
V_\lambda  \;\xrightarrow[\lambda\to\infty]{}\; V_\infty.
  \]
  Indeed, if were not the case, we could find a subsequence $(\lambda_k)$ such that $V_\lambda$ converges to some smaller limit along this subsequence. 
  As the set $\Pi(\muU^+,\muU^-)$ is compact for the weak convergence \cite[Lemma 4.4]{villani2008optimal}, we may find a transport plan $\pi_\infty^*\in\Pi(\muU^+,\muU^-)$ such that a subsequence of $\pi_{\lambda_k}^*$ converges weakly towards $\pi_\infty^*$. Let us still denote this subsequence by $\pi_{\lambda_k}^*$ for the sake of simplicity. We can then apply Fatou's lemma \cite{feinberg2014fatou} to find that
  \[
 V_\infty>\liminf_{k \to \infty}V_{\lambda_k}= \liminf_{k \to \infty} \int C_{\lambda_k,c_2}^u \dd \pi_{\lambda_k}^* \geq \int C_\infty \dd \pi_\infty^* \geq V_\infty,
  \]
  which yields a contradiction.
  
  Finally, define the associated optimal outputs by
\begin{equation*}
Y_{1,\lambda}:=T^+_{\lambda,\mathcal W_2^2}(Z_1),\qquad Y_{2,\lambda}:=T^-_{\lambda,\mathcal W_2^2}(Z_2),
\qquad (Z_1,Z_2)\sim \pi_\lambda^*,
\end{equation*}
so that $\nu_{1,\lambda}^*=\Law(Y_{1,\lambda})=(T^+_{\lambda,\mathcal W_2^2})_\#\muU^+$ and $\nu_{2,\lambda}^*=\Law(Y_{2,\lambda})=(T^-_{\lambda,\mathcal W_2^2})_\#\muU^-$.
Consider moreover the coupling between these two output laws induced by the same construction,
\begin{equation*}
  \gamma_\lambda:=(T^+_{\lambda,\mathcal W_2^2},T^-_{\lambda,\mathcal W_2^2})_\#\pi_\lambda^*\in\Pi(\nu_{1,\lambda}^*,\nu_{2,\lambda}^*).
\end{equation*}

By definition of $\mathcal W_2$, we have
\begin{align*}
 \mathcal W_2^2(\nu_{1,\lambda}^*,\nu_{2,\lambda}^*)
&\;\le\;
\int_{\R^2} (y_1-y_2)^2\,d\gamma_\lambda(y_1,y_2) \;=\;
\E_{\pi_\lambda^*}\!\left[(Y_{1,\lambda}-Y_{2,\lambda})^2\right] \\
&\;=\; \int \frac{(h_1-h_2)^2}{(1+\lambda(|d_1|+|d_2|))^2} \dd \pi_\lambda^*(z_1,z_2) \\
&\leq \frac{2}{\lambda} \left(\int \frac{h_1^2}{|d_1|} \dd \muU^+(z_1) + \int \frac{h_1^2}{|d_1|} \dd \muU^-\right)(z_2),
\end{align*}
where we use the inequality $(h_1-h_2)^2\leq 2h_1^2+2h_2^2$. Since $\int \frac{h^2}{|d|} \dd \muU^{\pm}(z) = \int_{\mathcal X_\pm} \eta(x)^2 \dd \chi(x)\leq \E[\eta(X)^2]<\infty$, we have $\mathcal W_2^2(\nu_{1,\lambda}^*,\nu_{2,\lambda}^*)\to 0$ as $\lambda\to \infty$. 
In particular, the two output laws become arbitrarily close in $\mathcal W_2$.

Therefore the relaxed solutions asymptotically enforce a single common output distribution; if the unrelaxed barycenter is unique, this common limit must coincide with the unrelaxed barycenter law.
\end{remark}

\subsection{Proof of \texorpdfstring{\cref{prop:tv_relaxed_reduction}}{tv-rerlaxed-red} (FAUST-U TV)}  
\label{app:unaware_proof_relaxed_TV}

\proprelaxedTV*

\begin{proof}[Proof of \cref{prop:tv_relaxed_reduction}]
  By \cref{lem:general_reduction}, the global unaware problem reduces to evaluating the pointwise relaxed cost $C_{\lambda, c_0}^u(\zu_1,\zu_2)$. For any fixed pair $\zu_1=(h_1,d_1)$ and $\zu_2=(h_2,d_2)$ with absolute density weights $a_1=|d_1| > 0$ and $a_2=|d_2| > 0$, we seek to minimize the objective
  \begin{equation*}
      \Phi(y_1, y_2) = \frac{(h_1-y_1)^2}{a_1} + \frac{(h_2-y_2)^2}{a_2} + \lambda \mathbbm{1}_{\{y_1 \neq y_2\}}.
  \end{equation*}
  We divide the pointwise optimization into two mutually exclusive cases depending on the indicator function.

  \medskip\noindent
  \textit{Case 1: We restrict the domain to $y_1 \neq y_2$.} \\
  The indicator evaluates to $1$, yielding the objective $\Phi_{y_1 \neq y_2}(y_1, y_2) = \frac{(h_1-y_1)^2}{a_1} + \frac{(h_2-y_2)^2}{a_2} + \lambda$. 
  Because $a_1, a_2 > 0$, the unconstrained minimum of this convex quadratic is trivially achieved at $y_1^* = h_1$ and $y_2^* = h_2$. 
  Assuming $h_1 \neq h_2$, this solution respects the domain restriction, and the minimum cost for this branch is exactly $\lambda$.

  \medskip\noindent
  \textit{Case 2: We restrict the domain to $y_1 = y_2 = y$.} \\
  The indicator evaluates to $0$, simplifying the objective to a single-variable quadratic $\Phi_{y_1 = y_2}(y) = \frac{(h_1-y)^2}{a_1} + \frac{(h_2-y)^2}{a_2}$. 
  The first-order optimality condition is $-\frac{2(h_1-y)}{a_1} - \frac{2(h_2-y)}{a_2} = 0$. Solving for $y$ yields the weighted barycenter
  \begin{equation}\label{eq:def_y*}
      y^* = \frac{(1/a_1)h_1 + (1/a_2)h_2}{1/a_1 + 1/a_2} = \frac{a_2 h_1 + a_1 h_2}{a_1 + a_2}.
  \end{equation}
  We now plug $y^*$ back into the objective 
  \begin{equation*}
      \Phi(y^*, y^*) = \frac{1}{a_1} \left( h_1 - \frac{a_2 h_1 + a_1 h_2}{a_1 + a_2} \right)^2 + \frac{1}{a_2} \left( h_2 - \frac{a_2 h_1 + a_1 h_2}{a_1 + a_2} \right)^2.
  \end{equation*}
  Simplifying the terms inside the squares gives $\frac{a_1(h_1 - h_2)}{a_1 + a_2}$ for the first and $\frac{-a_2(h_1 - h_2)}{a_1 + a_2}$ for the second. 
  Squaring these and multiplying by their respective outer weights yields
  \begin{align*}
      \Phi(y^*, y^*) =& \frac{1}{a_1} \frac{a_1^2 (h_1-h_2)^2}{(a_1+a_2)^2} + \frac{1}{a_2} \frac{a_2^2 (h_1-h_2)^2}{(a_1+a_2)^2} \\
      & = \frac{a_1 + a_2}{(a_1+a_2)^2} (h_1-h_2)^2 \\
      & = \frac{(h_1-h_2)^2}{a_1+a_2}.
  \end{align*}

  \medskip\noindent
  \paragraph{Conclusion.} 
According to \cref{lem:general_reduction}, a solution $(\nu^{+*},\nu^{-*})$ of (\ref{eq:P-lambda-general}) is given by the marginals distributions of $(y_1^*(Z_1,Z_2),y_2^*(Z_1,Z_2))$ where $(Z_1,Z_2)\sim \pi^*$ and $\pi^*$ is as in \eqref{eq:OT_TV_relaxed}. Precisely, the map $(y_1^*,y_2^*)$ is defined by letting $(y_1^*(z_1,z_2),y_2^*(z_1,z_2)) = (h_1,h_2)$ if $C^u_{\lambda,c_0}(z_1,z_2)=\lambda$, and   $(y_1^*(z_1,z_2),y_2^*(z_1,z_2)) =(y^*,y^*)$ with $y^*$ defined in \eqref{eq:def_y*} otherwise.

Our next goal is to show that under \cref{ass_non_atomic}, the random variable $y_1^*(Z_1,Z_2)$ does not depend on $Z_2$. We use once again Kantorovich duality.

Recall that there exists a function $\psi$, called a  Kantorovich potential, that satisfies the following properties. Let $\phi:\zu_2\mapsto \sup_{\zu_1\in \mathcal Z} (\psi(\zu_1)-C_{\lambda,c_0}(z_1,z_2))$ be the $C_{\lambda,c_0}$-transform of $\psi$. Then, 
$\pi^*$ is supported on the $C_{\lambda,c_0}^u$-subdifferential of a Kantorovich potential $\psi$, that is the set
\[
\Gamma = \{(\zu_1,\zu_2)\in \mathcal Z\times \mathcal Z:\ \psi(\zu_1)-\phi(\zu_2)= C_{\lambda,c_0}(z_1,z_2)\}.
\]
Furthermore, the function $\psi$ is $C_{\lambda,c_0}$-concave, meaning that for all $\zu_1$, $\psi(\zu_1)=\inf_{\zu_2}(\phi(\zu_2)+C_{\lambda,c_0}(z_1,z_2))$. Using the definition of $C_{\lambda,c_0}$, we find that
\begin{align*}
\psi(\zu_1) =& \inf_{z_2}(\phi(z_2)+ \min(C_\infty(z_1,z_2), \lambda))= \min\left(\tilde \psi(z_1), \inf_{\zu_2}\phi(\zu_2)+\lambda \right)
\end{align*}
where $\tilde \psi(z_1)=\inf_{\zu_2}(\phi(\zu_2)+ C_\infty(\zu_1,\zu_2))$ defines a $C_\infty$-concave function for $C_\infty(\zu_1,\zu_2) = \frac{(h_1-h_2)^2}{a_1+a_2}$. According to Lemma 3 in \cite{divol2024demographicparityregressionclassification}, under \cref{ass_non_atomic}, the function $\tilde \psi$ is $\mu_{\eta,\Delta}^+$-almost everywhere differentiable at $z_1=(h_1,d_1)$ in the direction $h_1$.

Let $\mathcal Z_< = \{z_1\in\mathcal Z:\ \psi(z_1)= \tilde \psi(z_1)\}$. 
If $(z_1,z_2)\in\Gamma$ with $z_1\in \mathcal Z_<$, then 
\begin{align*}
  C_\infty(z_1,z_2)&\geq  \min(\lambda,C_\infty(z_1,z_2)) =  C_{\lambda,c_0}(z_1,z_2) = \psi(z_1)-\phi(z_2)\\
  =& \tilde \psi(z_1)- \phi(z_2) \geq C_\infty(z_1,z_2).
\end{align*}
Thus, $\psi(z_1)-\phi(z_2)=C_\infty(z_1,z_2)$. In other words, the set $\{(z_1,z_2)\in \Gamma:\ z_1\in \mathcal Z_<\}$ is included in the $C_\infty$-subdifferential of $\tilde\psi$. It is proven in \cite{divol2024demographicparityregressionclassification} (as explained in the previous section) that if $z_1$ is a point where $\tilde \psi$ is differentiable at $z_1$ in the direction $h_1$, and if $(z_1,z_2)$ is in the $C_\infty$-subdifferential of $\tilde\psi$, then the quantity
\[
y_1^*(z_1,z_2)=\frac{a_2h_1+a_1h_2}{a_1+a_2}
\]
does not depend on the choice of $z_2$. Thus,  the map $T_{\lambda,\TV}^+(z_1)= \frac{a_2h_1+a_1h_2}{a_1+a_2}$ is well-defined for $\mu_{\eta,\Delta}^+$-almost all $z_1\in \mathcal Z_<$.

If $z_1\in \mathcal Z_= = \mathcal Z\backslash \mathcal Z_<$, then $\tilde \psi(z_1)>\inf_{z_2}\phi(z_2)+\lambda$. Thus, if $(z_1,z_2)\in \Gamma$, then
\begin{align*}
  \lambda &\geq  \min(\lambda,C_\infty(z_1,z_2)) =  C_{\lambda,c_0}(z_1,z_2) = \psi(z_1)-\phi(z_2)\\
  =& \inf_{z_2'}\phi(z_2')+\lambda- \phi(z_2) \geq \lambda.
\end{align*}
Thus $C_{\lambda,c_0}(z_1,z_2)=\lambda$ and thus $y_1^*(z_1,z_2)=h_1$ clearly only depends on $z_1$. This shows that the map $T_{\lambda,\TV}^+$ is well-defined $\mu_{\eta,\Delta}^+$-almost everywhere.

Similarly, we show that the map $T_{\lambda,\TV}^-$ is well-defined $\mu_{\eta,\Delta}^-$-almost everywhere.

We then conclude using \Cref{lem:equiv_erm_ot} to show that solutions of \Cref{eq:erm_penalized} take the form described in the lemma.
\end{proof}

Let us note that \Cref{prop:relaxed_reduction} implies in particular the result of \Cref{lem:reduc_OT} in the case $c_{\mathcal D}=c_0$.

\section{Proofs for the relaxed aware setting}
\label{app:aware_corollaries}

As discussed in \cref{remark:unifying_aware_unaware}, the aware setting corresponds to the boundary case of Problem (\ref{eq:P-lambda-general}), where the sensitive attribute is observed, that is we observe $(X,S)$ instead of $X$.
In that case, the Jordan decomposition partitions the space into the demographic groups, yielding uniform absolute density weights $a_1 = 1/p^+$ for group $S=+$ and $a_2 = 1/p^-$ for group $S=-$. Then, $\chi^+$ and $\chi^-$ have separate supports, with $\Delta(x,+)=1/p^+$ for $(x,+)$ in the support of $\chi^+$, $\Delta(x,-)=-1/p^-$ for $(x,-)$ in the support of $\chi^-$. The measure $\mu_{\eta,\Delta}^{+}$ is simply the law of $(\eta(X,+),1/p^+)$ where $X\sim \chi^+$, and $\mu_{\eta,\Delta}^{-}$ is  the law of $(\eta(X,-),-1/p^-)$ where $X\sim \chi^-$. 

The two-dimensional costs described in \Cref{prop:relaxed_reduction} and \Cref{prop:tv_relaxed_reduction} now depend effectively only on $(h_1,h_2)$, that is these costs  become one-dimensional. In that case, the solution of the corresponding optimal transport problem can be formulated in a simpler fashion. Let us describe it  more precisely.


\paragraph{Exact Aware Solution.} 
First, in the limit $\lambda\to\infty$ (no penalization), as established by \citet{chzhen2020fairregressionwassersteinbarycenters}, the exact Wasserstein barycenter for the 1D aware setting admits a closed-form. 
The fair prediction $f^*_{\infty}(x, s)$ for an individual is obtained by evaluating their relative rank within their own demographic group via the cumulative distribution function $F_{\muA^s}$, and mapping it to the weighted quantiles of both groups
\begin{equation*}
    f^*_{\infty}(x, s) =q^*\circ F_{\muA^s}\left(\eta(x, s)\right) :=\left( p^+ Q_{\muA^+} + p^- Q_{\muA^-} \right) \circ F_{\muA^s}\left(\eta(x, s)\right)
\end{equation*}
where $q^*:t\in (0,1)\mapsto p^+ Q_{\muA^+}(t) + p^- Q_{\muA^-}(t)$ is the barycenter map.


\subsection{Aware $\mathcal{W}_2$ relaxation}

Here we consider the fairness penalty $\mathcal{D}(\cdot, \cdot) = OT_{c_2}(\cdot, \cdot)$ for $c_2(y, y') = (y-y')^2$.

\begin{corollary}[\textbf{FAUST-A \textbf{$\mathcal W_2$}}]
  \label{cor:aware_w2}
  Under \cref{ass_non_atomic}, in the awareness setting $S$, the solution to \eqref{eq:erm_penalized} is given by the optimal predictor $f^*_{\lambda, \mathcal{W}_2}(x,s)$, which interpolates between the unconstrained predictor $\eta(x,s)$ and the exact fair predictor $f^*_{\infty}(x,s)$ 
  \begin{equation*}
      f^*_{\lambda, \mathcal{W}_2}(x, s) = (1 - \alpha) f^*_{\infty}(x, s) + \alpha \eta(x, s), \quad \text{where} \quad \alpha = \frac{p^+ p^-}{p^+ p^- + \lambda}.
  \end{equation*}
  This explicitly recovers the closed-form geodesic interpolation derived by \citet{chzhen2022minimaxframeworkquantifyingriskfairness}.
\end{corollary}
  
\begin{proof}
  In the framework of \citet{chzhen2022minimaxframeworkquantifyingriskfairness}, the fairness penalty is defined as the weighted variance $p^+ p^- \mathcal{W}_2^2(\nu^+, \nu^-)$. Thus, our $\lambda$-penalized objective is  equivalent to their formulation under the effective penalty $\lambda_{C} = \lambda / (p^+ p^-)$. 
  
  According to their Proposition 4.7, this problem admits a relative improvement governed by an interpolation parameter $\rho = (1 + \lambda_{C})^{-2}$. 
  Following their Proposition 4.1, the optimal relaxed predictor is uniquely defined as the convex combination of the perfectly fair projection and the unconstrained base predictor, weighted by $1 - \sqrt{\rho}$ and $\sqrt{\rho}$ respectively. 
  Substituting our penalty yields exactly
  \begin{align*}
      \sqrt{\rho} 
      =& \left(1 + \frac{\lambda}{p^+ p^-}\right)^{-1} = \frac{p^+ p^-}{p^+ p^- + \lambda} = \alpha.
  \end{align*}

  We now show that our generalized unaware reduction perfectly recovers this geometric interpolation from the maps defined in \cref{prop:relaxed_reduction}. 
  Substituting the deterministic aware weights $a_1 = 1/p^+$ and $a_2 = 1/p^-$, the interpolation denominator becomes
  \begin{align*}
      1 + \lambda(a_1 + a_2) 
      =& 1 + \lambda\left(\frac{p^+ + p^-}{p^+ p^-}\right) ,
  \end{align*}
  and because $p^+ + p^- = 1$, this simplifies to
  \begin{align*}     
     1 + \lambda(a_1 + a_2)
      =& 1 + \frac{\lambda}{p^+ p^-} = \alpha^{-1}.
  \end{align*}
  Thus,  $C_{\lambda,c_2}^u(z_1,z_2)= \alpha\lambda (h_1-h_2)^2$ for $z_1=(h_1,1/p^+)$ and $z_2=(h_2,-1/p^-)$. 

The optimal transport problem minimizing this cost is the monotone matching, obtained by letting $h_1 = Q_{\muA^+}(t)$ and $h_2 = Q_{\muA^-}(t)$ for a given $t\in (0,1)$.
  
  We then find from \cref{eq:y-star}
  \begin{align*}
      T^+_{\lambda,\mathcal{W}_2^2}(\zu_1) =& \alpha \left[ \left(1 + \frac{\lambda}{p^-}\right) Q_{\muA^+}(t) + \frac{\lambda}{p^+} Q_{\muA^-}(t) \right] \\
            =& \alpha Q_{\muA^+}(t) + \alpha \lambda \left( \frac{Q_{\muA^+}(t)}{p^-} + \frac{Q_{\muA^-}(t)}{p^+} \right).
  \end{align*}
  Recognizing from the definition of $\alpha$ that $\alpha \lambda = (1-\alpha)p^+ p^-$, the second term simplifies to $(1-\alpha)(p^+ Q_{\muA^+}(t) + p^- Q_{\muA^-}(t))$, which corresponds to $(1-\alpha)q^*(t)$, where $q^*$ is the strict Wasserstein barycenter.

Evaluating this optimal quantile mapping at the individual's relative rank $t_{x,+} = F_{\muA^+}(\eta(x, +))$ recovers the predictor formulation for the positive group
  \begin{equation*}
      f^*_{\lambda, \mathcal{W}_2}(x, +) = (1-\alpha)f^*_{\infty}(x, +) + \alpha \eta(x, +),
  \end{equation*}
  and symmetrically for group $S=-$ using $T^-$. Thus, the unaware mapping retrieves the aware geodesic interpolation.
\end{proof}

\subsection{Aware TV relaxation}
Here we consider the fairness penalty $\mathcal{D}(\cdot, \cdot) = OT_{c_0}(\cdot, \cdot)$ for $c_0(y, y') = \mathbbm{1}_{\{y \neq y'\}}$.


  
  \begin{restatable}[\textbf{FAUST-A TV}]{corollary}{corAwareTV}
    \label{cor:aware_tv}
    Under \cref{ass_non_atomic}, in the awareness setting $S \in \{+, -\}$, the solution to \eqref{eq:erm_penalized} is characterized by the optimal transport plan $\pi^* \in \Pi(\muA^+, \muA^-)$ that minimizes the cost
    \begin{equation}
        \pi^* \in \arg\min_{\pi \in \Pi(\muA^+, \muA^-)} \int_{\mathbb{R} \times \mathbb{R}} \min \left( \lambda, \; p^+ p^- (h_1 - h_2)^2 \right) \dd\pi(h_1, h_2).
    \end{equation}
    For any pair $(h_1, h_2)$ in the support of $\pi^*$, the solution to \cref{eq:erm_penalized} is given by 
    \begin{equation*}
        \begin{cases}
            f^*_{\lambda, \TV}(h_1, +) = f^*_{\lambda, \TV}(h_2, -) = p^+ h_1 + p^- h_2 & \text{if} \quad p^+ p^- (h_1 - h_2)^2 \le \lambda, \\[2ex]
            f^*_{\lambda, \TV}(h_1, +) = h_1, \quad f^*_{\lambda, \TV}(h_2, -) = h_2 & \text{if} \quad p^+ p^- (h_1 - h_2)^2 > \lambda.
        \end{cases}
    \end{equation*}
  \end{restatable}

  \begin{proof}[Proof of \cref{cor:aware_tv}]
In the aware setting, the absolute density weights are constant and strictly determined by the group priors: $a_1 = |d_1| = \frac{1}{p^+}$ for the positive group and $a_2 = |d_2| = \frac{1}{p^-}$ for the negative group. 
We derive the aware solution by substituting these values directly into the results of \cref{prop:tv_relaxed_reduction}.
    
    First, we evaluate the denominator present in both the cost and the barycenter formulations
    \begin{align*}
        a_1 + a_2 =& \frac{1}{p^+} + \frac{1}{p^-} = \frac{1}{p^+ p^-},
    \end{align*}
    where we used that $p^+ + p^- = 1$. 
    Substituting this into the threshold from \cref{prop:tv_relaxed_reduction} yields
    \begin{align*}
        \frac{(h_1 - h_2)^2}{a_1 + a_2} = p^+ p^- (h_1 - h_2)^2.
    \end{align*}
    This recovers the threshold $p^+ p^- (h_1 - h_2)^2 \le \lambda$ and the truncated cost $\min \left( \lambda, \; p^+ p^- (h_1 - h_2)^2 \right)$ stated in the corollary.
    
    Next, we substitute the weights into the optimal weighted barycenter $y^*$ defined in \cref{prop:tv_relaxed_reduction}
    \begin{align*}
        y^* = \frac{a_2 h_1 + a_1 h_2}{a_1 + a_2} = \left( \frac{1}{p^-} h_1 + \frac{1}{p^+} h_2 \right) p^+ p^- = p^+ h_1 + p^- h_2.
    \end{align*}
    Applying this simplified cost threshold and this simplified barycenter directly into the piecewise formulation of the deterministic maps from \cref{prop:tv_relaxed_reduction} exactly recovers the closed-form predictions for $f^*_{\lambda, \TV}$, concluding the proof.

    Finally, we reconstruct the optimal map. According to \cref{prop:tv_relaxed_reduction}, when the cost exceeds the threshold $\lambda$, the map outputs the unconstrained base prediction $h_s$. When the cost is below the threshold $\lambda$, the maps output the shared barycenter, recovering the formula in the statement of the corollary.
\end{proof}

\section{Unfairness of plug-in approaches}
\label{app:toy_plugin}

Replacing the unknown senstitive attribute $S$ by a proxy $\hat{S}$ is a common practice \cite{vladimirova2024fairjob}. A line of work \citep{bharti2023estimating,10.1287/mnsc.2020.3850,gupta2018proxyfairness,pmlr-v108-awasthi20a} has therefor studied the relationship between fairness with respect to a proxy sensitive attribute and fairness with respect to the true sensitive attribute. With the exception of \cite{10.1287/mnsc.2020.3850}, which establishes impossibility results, these works generally aim to show that, either under suitable assumptions on the proxy or empirically, low proxy-based unfairness translates into low unfairness with respect to the true sensitive attribute. These assumptions are often too strong to hold in practice, as illustrated by the unavoidable level of unfairness observed for the plug-in approach in our simulations.

Below, we state a result that may be considered folklore. We were unable to find a precise reference for it, and therefore include it together with a proof based on counterexamples. These counterexamples show that, in general, fairness with respect to a proxy does not imply fairness with respect to the true sensitive attribute.
\begin{proposition}\label{prp:plug-in}
    Let $\hat{S}(X) = \arg max_{s \in \{+,-\}} \mathbb{P}(S=s\vert X)$ be the maximum a posteriori estimator of the sensitive attribute (ties broken arbitrarily). Then, for $f : \mathcal{X} \mapsto \mathbb{R}$,
    \begin{align*}
        f(X) \perp\!\!\!\perp \hat{S}(X)\not\Rightarrow f(X) \perp\!\!\!\perp S. 
    \end{align*}
    Moreover, for $f : \mathcal{X}\times S \mapsto \mathbb{R}$,
    \begin{align*}
        f(X, S) \perp\!\!\!\perp S\not\Rightarrow f(X, \hat{S}(X)) \perp\!\!\!\perp S. 
    \end{align*}
\end{proposition}

\Cref{prp:plug-in} proves that plug-in approaches, either enforcing DP with regards to $\hat{S}$ or replacing $S$ by $\hat{S}$ in aware predictors generally fail to enforce demographic parity.

\begin{proof} We prove \Cref{prp:plug-in} by providing simple counter-examples.
(\textbf{First claim})
    Consider the following distributions : $S\sim$ Rademacher, $X\vert S = + \sim 0.8\times \mathcal{U}([0,0.5]) + 0.2\times \mathcal{U}([0.75,1])$, $X\vert S = - \sim 0.6\times \mathcal{U}([0.5,0.75]) + 0.4\times \mathcal{U}([0.75,1])$, and $Y = X$. Then, $\hat{S}(x) = +$ if $X\leq 0.5$, and $\hat{S}(x) = -$ if $x > 0.5$. The optimal prediction function that is fair with regards to $\hat{S}(X)$ (i.e., the prediction function enforcing $f(X) \perp\!\!\!\perp \hat{S}(X)$ with minimal squared risk) is given by the maps to the Wasserstein barycenter of the distributions $Y\vert \hat{S}(X) = + \sim \mathcal{U}([0,0.5])$ and  $Y\vert \hat{S}(X) = - \sim \mathcal{U}([0.5,1])$ with weights $\mathbb{P}(\hat{S}(X) = +) = 0.4$ and $\mathbb{P}(\hat{S}(X) = -) = 0.6$ \cite{chzhen2020fairregressionwassersteinbarycenters}. Its distribution is therefore $\mathcal{U}([0.3, 0.8])$, and the optimal fair prediction is 
\begin{align*}
    f(x) = 
    \begin{cases}
        x + 0.3 \qquad \text{if} \qquad x\leq 0.5\\
        x - 0.2 \qquad \text{else.}
    \end{cases}
\end{align*}
By construction, $f(X) \perp\!\!\!\perp \hat{S}(X)$. However, it is easy to see that 
\begin{align*}
    \mathbb{P}(f(X) \in [0.3, 0.55] \vert S = +) &= \mathbb{P}(f(X) \in [0.3, 0.55], \hat{S}(X) = + \vert S = +) \\
    &= \mathbb{P}(\hat{S}(X) = +\vert S = +)\mathbb{P}(f(X) \in [0.3, 0.55] \vert S = +, X \leq 0.5)\\
    &= 0.8\mathbb{P}(X + 0.3 \in [0.3, 0.55]  \vert S = +, X \leq 0.5) \\
    &=0.4.
\end{align*}
On the other hand, 
\begin{align*}
    \mathbb{P}(f(X) \in [0.3, 0.55] \vert S = -) &= \mathbb{P}(X - 0.2 \in [0.3, 0.55] \vert S = -)\\
    &=  \mathbb{P}(X \in [0.5, 0.75] \vert S = -)\\
    &= 0.6.
\end{align*}
Thus, $f(X) \not \perp\!\!\!\perp S$, which concludes the proof.

(\textbf{Second claim}) Consider the following distribution : $S\sim$ Rademacher, $X\vert S = + \sim \mathcal{U}([0,1])$, $X\vert S = - \sim \mathcal{U}([0.5,1])$, and $Y = X$. Then, $\hat{S}(x) = +$ if $X\leq 0.5$, and $\hat{S}(x) = -$ if $x > 0.5$. The optimal fair prediction function in the awareness setting is given by the maps to the Wasserstein barycenter, which is $\mathcal{U}([0.25, 1])$, and the optimal fair prediction is 
\begin{align*}
    f(x,s) = 
    \begin{cases}
        0.75x + 0.25 \qquad \text{if} \qquad s= +\\
        1.5x - 0.5 \qquad \text{else.}
    \end{cases}
\end{align*}
By construction, we have $f(X, S) \perp\!\!\!\perp S$. Yet, on the one hand,
\begin{align*}
    \mathbb{P}(f(X,\hat{S}(X)) \in [0.625, 1] \vert S = +) &=\mathbb{P}(f(X,-) \in [0.625, 1] \vert S = +) \\
    &=\mathbb{P}(1.5X - 0.5 \in [0.625, 1] \vert S = +) \\
    &=\mathbb{P}(X \in [0.75, 1] \vert S = +) \\
    &=0.25.
\end{align*}
On the other hand,
\begin{align*}
    \mathbb{P}(f(X,\hat{S}(X)) \in [0.625, 1] \vert S = -) &=\mathbb{P}(X \in [0.75, 1] \vert S = -)\\
    &=0.5.
\end{align*}
Thus, $f(X, \hat{S}(X))\not\perp\!\!\!\perp S$.
\end{proof}

\section{Algorithm pseudocode} \label{app:algorithmic_implementation}
We provide the  pseudocode of our algorithm given in \cref{alg:fair_unaware}, as well as a justification of the reweighting of the empirical measures. 

\begin{algorithm}[h!]
  \caption{Unaware Fair Regression}
  \label{alg:fair_unaware}
  \begin{algorithmic}[1]
      \State \textbf{Input:} Base regressor $\hat{\eta}(x)$, sensitive attribute's classifier $\hat{p}(s\vert x)$, priors $p^+$ and $p^-$, dataset $\{x_k\}_{k=1}^{n}$
      \State \textbf{Output:} Fair continuous map $\hat{f}_{fair}$
      \Statex
      \State \textbf{Step 1: Estimation and Jordan decomposition}
      \State For all $k \in \{1, \dots, n\}$, compute unconstrained predictions $\hat{\eta}_k \gets \hat{\eta}(x_k)$.
      \State For all $k \in \{1, \dots, n\}$, compute density ratio differences: $\hat{\Delta}_k \gets \frac{\hat{p}(+\vert x_k)}{p^+} - \frac{\hat{p}(-\vert x_k)}{p^-}$
      \State Define index sets: $\mathcal{I}^+ = \{i : \hat{\Delta}_i > \tau\}$ and $\mathcal{I}^- = \{j : \hat{\Delta}_j < -\tau\}$.
      \State Compute normalized probability weights $a_i$ (for $i \in \mathcal{I}^+$) and $b_j$ (for $j \in \mathcal{I}^-$) so $\sum a_i = \sum b_j = 1$.
      \Statex
      \State \textbf{Step 2: Relaxed Optimal Transport}
          \State Cost $\mathbf{C}_{i,j} \gets C_{\lambda, \mathcal{D}}\left((\hat{\eta}_i, \hat{\Delta}_i), (\hat{\eta}_j, \hat{\Delta}_j)\right)$ \hfill \textit{(e.g., via \cref{prop:relaxed_reduction} or \cref{prop:tv_relaxed_reduction})}
          \State Pairwise targets $m^+_{i,j} \gets T_{\lambda, \mathcal{D}}^+\left((\hat{\eta}_i, \hat{\Delta}_i), (\hat{\eta}_j, \hat{\Delta}_j)\right)$ and $m^-_{i,j} \gets T_{\lambda, \mathcal{D}}^-\left((\hat{\eta}_i, \hat{\Delta}_i), (\hat{\eta}_j, \hat{\Delta}_j)\right)$

      \State Solve discrete OT plan $\mathbf{\Pi}$ between $\hat{\mu}^+ (\text{weights } a_i)$ and $\hat{\mu}^- (\text{weights } b_j)$ using cost matrix $\mathbf{C}$.
      \Statex
      \State \textbf{Step 3: Construct the fair targets}
      \State Initialize pseudo-label vector $f = (f_1, \dots, f_n)$.
      \For{$k = 1$ to $n$}
          \If{$k \in \mathcal{I}^+$} 
              \State $f_k \gets \frac{1}{a_k} \sum_{j \in \mathcal{I}^-} \mathbf{\Pi}_{k,j} m^+_{k,j}$
          \ElsIf{$k \in \mathcal{I}^-$} 
              \State $f_k \gets \frac{1}{b_k} \sum_{i \in \mathcal{I}^+} \mathbf{\Pi}_{i,k} m^-_{i,k}$
          \Else 
              \State $f_k \gets \hat{\eta}_k$
          \EndIf
      \EndFor
      \Statex
      \State \textbf{Step 4: Fit final regressor}
      \State Train the continuous map $\hat{f}_{fair}$ on inputs $\{(\hat{\eta}_k, \hat{\Delta}_k)\}$ to predict targets $\{f_k\}$.
  \end{algorithmic}
\end{algorithm}

\paragraph{Justification of the reweighting step} Recall that $(X_i)_{i\leq n} \overset{i.i.d.}{\sim}\chi$ for $\chi = p_+ \chi_+ + p_- \chi_-$. To estimate the distributions $\muU^+ = \left(\eta(\cdot), \Delta(\cdot)\right)_{\#} \left(\chi^+ - \chi^-\right)_+$ and $\muU^- = \left(\eta(\cdot), \Delta(\cdot)\right)_{\#} \left(\chi^+ - \chi^-\right)_-$, 
we note that for $x$ such that $\Delta(x)>0$, we have by Bayes rule $\mathbb{P}(S=s\vert X=x)\dd\chi(x) = p_s \dd\chi^s(x)$ 
\begin{align*}
  \Delta(x)\dd\chi(x) &= \left(\frac{\mathbb{P}(S = +\vert X = x)}{p_+} - \frac{\mathbb{P}(S = -\vert X = x)}{p_-}\right)\dd\chi(x)\\
  &\propto \dd\left(\chi^+ - \chi^-\right)_+(x),
\end{align*} 

where the proportionality comes from defining $\left(\chi^+- \chi^-\right)_+$ and $\left(\chi^+- \chi^-\right)_-$ as the \textit{normalized} Jordan decomposition of $\left(\chi^+- \chi^-\right)$. Thus, the weighted empirical measure $(X_i)_{i \leq n}$ with weights $\Delta(X_i)_+$ is a empirical (unbiased) version of the measure $\left(\chi^+- \chi^-\right)_+$, so $(\hat{\eta}_j, \hat{d}_j)_{j \in \mathcal{I}^+}$ with corresponding weights $(\hat{d}_i)_{i \in \mathcal{I}^+}$ is an estimate for the measure  $\muU^+$.

Similarly, for $x$ such that $\Delta(x)<0$, we have $\vert \Delta(x)\vert \dd\chi(x) \propto \dd\left(\chi^+ - \chi^-\right)_-(x)$, and the weighted empirical measure $(X_i)_{i \leq n}$ with weights $-\Delta(X_i)_-$ is a empirical (unbiased) version of the measure $\left(\chi^+ - \chi^-\right)_-$, so $(\hat{\eta}_j, \hat{d}_j)_{j \in \mathcal{I}^-}$ with corresponding weights $(\hat{d}_i)_{i \in \mathcal{I}^-}$ is an estimate for the measure  $\muU^-$.

\section{Experimental setup and implementation details}
\label{app:experimental_setup}

To ensure a rigorous, reproducible, and equitable comparison between our Optimal Transport framework and the FairReg baseline, we strictly to the evaluation protocol established by \citet{Fairreg}. 

\subsection{Details on the Unaware Baselines (FairReg and FRAPPÉ)}
\label{app:baseline_details}

We contextualize the performance and computational profiles of FairReg \citep{Fairreg} and FRAPPÉ \citep{tifrea2024frappe}.

\paragraph{FairReg: Grid discretization and dual optimization.}
FairReg operates as an unlabeled post-processing algorithm that does not map inputs to a deterministic continuous prediction. 
Instead, it relies on discretization, defining a uniform grid $\hat{\mathcal{Y}}_L$ over a bounded interval $[-B, B]$ consisting of $2L+1$ discrete points. 
The algorithm learns a randomized prediction function over this support, parameterized by a softmax function where the hyperparameter $\beta$ controls the degree of entropic regularization. 
To enforce DP, they formulate a dual stochastic convex program from this entropically regularized objective. 
This dual problem is then minimized using a stochastic first-order oracle, typically implemented via Accelerated Stochastic Approximation (AC-SA2) to control the expected norm of the gradient mapping. 

\paragraph{FairReg: Iterations ($T$).}
In their theoretical framework and codebase, the hyperparameter $T$ defines the total number of stochastic gradient evaluations. 
Because the stochastic gradient of the dual objective is computed by sampling one independent unlabeled feature vector $X$ per step, $T$ represents individual iterations rather than full passes over the dataset. 
To guarantee a fair, scale-invariant convergence analysis across datasets of varying sizes ($N$), we parameterize the baseline's optimization in terms of training epochs. 
Specifically, we set $T = 60 \times N$ for all experiments. 
This ensures the FairReg algorithm performs 60 epochs over the unlabeled data.

\paragraph{FRAPPÉ: Additive post-hoc modules and SGD.}
FRAPPÉ is a post-processing framework that achieves fairness by adding a learned post-hoc module, denoted $T_{PP}(X)$, to the unfair scores output by a pre-trained base model. Because this additive transformation depends only on the covariates $X$, it does not require access to the sensitive attribute at inference time. In our experiments, matching the authors implementation, we model this post-hoc transformation using a 1-layer Multi-Layer Perceptron (MLP) with 64 hidden units. The module is trained using stochastic gradient descent (specifically the Adam optimizer with a learning rate of 0.01) to minimize a combination of prediction loss and a Macximum Mean Discrepency (MMD) fairness regularizer with a Gaussian kernel. Because FRAPPÉ relies on gradient-based optimization rather than closed-form mappings, setting the fairness penalty parameter ($\lambda$) to high values to enforce parity causes optimization instability, leading to the degraded predictive utility observed in our trade-off analyses.

\subsection{General experimental details}

\paragraph{Data preprocessing and splits.}
For all real-world datasets, the continuous target variable $Y$ is scaled linearly to the interval $[-1, 1]$. 
For all experiments, we evaluate out-of-sample generalization by splitting the data into a training set and a hold-out test set (using an $80\% - 20\%$ split), previously stratified by the sensitive attribute $S$ to preserve group prior probabilities. 

\paragraph{Base models and estimators.}
To isolate the effect of the fairness interventions and ensure differences in performance are attributable to the post-processing algorithms, all methods share the same base estimators. 
Specifically, the unconstrained base regressor $\hat{\eta}(X)$ is instantiated as a standard Ordinary Least Squares (OLS) Linear Regression model (we provide additional base model ablations study in \cref{app:base_models_benchmark}). 
The probabilistic sensitive attribute classifier $\hat{\mathbb{P}}(S \mid X)$, which is required by both FairReg and our Unaware OT framework to estimate the local density ratios, is instantiated as a Logistic Regression model trained for up to 2000 iterations to ensure convergence.

\paragraph{Hyperparameters.}
For the FairReg and FRAPPÉ baselines, we follow their original implementations and optimization heuristics, with full configuration details deferred to \cref{app:baseline_details}.
For our OT framework, the discrete optimal transport plan $\pi^*$ is solved exactly. 
To construct the final continuous, out-of-sample predictor $\hat{f}_{fair}$ from the discrete OT assignments, we train a Random Forest regressor (with 200 trees) to map the 2-dimensional feature space $(\hat{\eta}(x), \hat{\Delta}(x))$ to the optimal fair targets $y^*$. The parameter $\tau$ in \ref{alg:fair_unaware} is set to the value $10^{-6}$. 

\paragraph{Statistical significance and error bars.}
To assess the stability and statistical significance of our results, all experiments are aggregated over multiple independent trials (10 runs unless otherwise specified). 
The source of random variability depends on the experimental setting: for the synthetic dataset (\cref{app:synthetic_dataset}), the random seed dictates the underlying data generation process; for real-world datasets (\cref{app:adult_dataset,sec:experiments}), the random seed controls the stratified train/validation/test splits. 
In all summary tables, the reported $\pm$ values represent the empirical standard deviation calculated across these independent runs (\cref{app:fully_constrained,app:computation_times,app:base_models_benchmark}). 
Similarly, in all visualizations (such as the accuracy-fairness trade-off and confounding sweep curves), the shaded regions represent $\pm 1$ standard deviation. 

    \section{Experiments on synthetic datasets}

We use two distinct synthetic datasets to empirically illustrate the exact geometric behavior of the optimal transport relaxations in both the aware and unaware settings: a non-linear confounded dataset designed to visualize the effect of the relaxations on the individual predictions, and a linearly confounded dataset designed to sweep the confounding strength and evaluate the accuracy-fairness trade-off across different levels of bias.

\subsection{Empirical illustration of the geometric relaxations} \label{app:empirical_illustrations}

\paragraph{Experimental setup and data generation.}
To isolate and visualize the geometric behavior of the optimal transport relaxations, we use a 1-dimensional synthetic dataset designed to introduce a non-linear confounding effect. We generate $N=5000$ independent samples according to the following data generating process:
\begin{itemize}
    \item \textbf{Sensitive attribute:} We sample the binary group uniformly at random $S \sim \text{Bernoulli}(0.5)$, and define the corresponding Rademacher variable $Z_S = 2(S - 0.5) \in \{-1, 1\}$.
    \item \textbf{Confounded feature:} We construct the continuous feature $X \in \mathbb{R}$ as a linear combination of the sensitive group and an independent Gaussian noise: $X = \gamma Z_S + \sqrt{1 - \gamma^2} \mathcal{N}(0, 1)$, where the confounding strength is set to $\gamma = 0.6$.
    \item \textbf{Target variable:} The continuous outcome $Y$ follows a non-linear S-curve dependence on $X$ with additive Gaussian noise: $Y = 3.0 \tanh(1.5 X) + \mathcal{N}(0, 0.5^2)$. 
\end{itemize}

To remove finite-sample estimation artifacts and evaluate the pure geometry of the optimal transport framework, we bypass empirical risk minimization for the base models. 
Instead, we use oracles for the continuous conditional expectation $\eta(x) = 3.0 \tanh(1.5 x)$ and the exact posterior probabilities $\mathbb{P}(S=1 \mid X=x) = \sigma\left(\frac{2\gamma x}{1-\gamma^2}\right)$, where $\sigma(z) = \frac{1}{1 + e^{-z}}$
is the sigmoid function. 

Finally, to ensure an equitable comparison between the $\mathcal W_2$ and TV penalties across both the aware and unaware paradigms, we perform a numerical line search to identify the exact $\lambda$ (or exact threshold $\tau$ for the aware TV case) that matches a fairness budget (e.g., preserving exactly $80\%$, $60\%$, or $20\%$ of the initial unconstrained Wasserstein-2 unfairness).

\paragraph{Qualitative analysis of the unaware relaxations}

To complement the analysis of the marginal distributions provided in the main text, we visualize here the effect of the relaxations on the individual predictions. 

As shown in \cref{fig:unaware_composite}, comparing the base ERM predictor against the optimal relaxed unaware predictors (matched to an 80\% unfairness budget) highlights the geometric difference between the two penalties. 
The $\mathcal W_2$ penalty induces a global smooth interpolation, pushing all individuals closer to the fair center. Conversely, the TV penalty creates a flat plateau, resulting in an overlapping density spike in the center of the distribution while leaving the extreme predictions untouched.

\begin{figure*}[h!]
    \centering
    \includegraphics[width=0.9\textwidth]{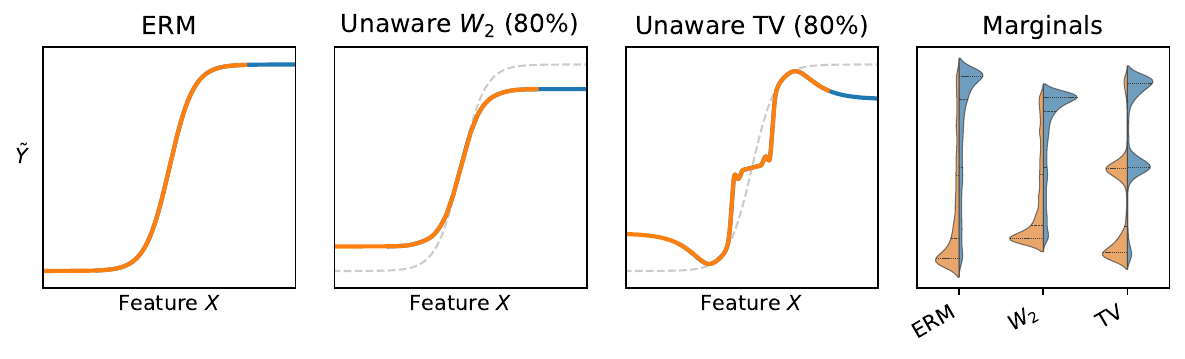}
     \caption{\textbf{Unaware relaxations.} 
     Comparison of the Bayes predictor against the optimal relaxed unaware predictors under $\mathcal W_2$ and TV penalties, matched to an 80\% unfairness budget.}
     \label{fig:unaware_composite}
\end{figure*}

\paragraph{Qualitative analysis of the aware relaxations}

To complement the unaware visualizations above, we empirically evaluate the exact same relaxations in the \emph{aware setting} on our 1D synthetic dataset. 

As shown in \cref{fig:aware_composite}, the aware $\mathcal W_2$ penalty induces a smooth interpolation, shifting predictions homogeneously toward the central compromise. Conversely, the aware TV penalty acts as a hard-thresholding operator. This creates a highly discontinuous mapping: predictions merge if the pointwise disparity gap falls below the tolerated threshold, but keepthe unconstrained biased baseline if the gap is too large. 

By varying the regularization parameter to transition the model from the unconstrained empirical risk minimizer (100\% unfair) to strict Demographic Parity (0\% unfair), we observe the evolution of the marginal distributions in \cref{fig:aware_evolution}. Consequently, the aware TV penalty fragments the intermediate marginal distributions into distinct, multi-modal subgroups.

\begin{figure*}[h!]
  \centering
  \includegraphics[width=0.9\textwidth]{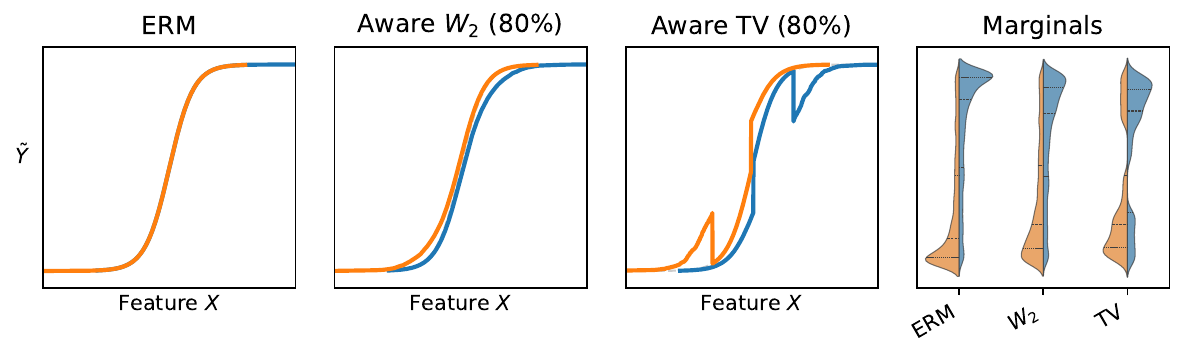}
  \caption{\textbf{Aware relaxations.} Comparison of the base ERM predictor against the optimal relaxed aware predictors under $\mathcal W_2$ and TV penalties. 
  While $\mathcal W_2$ yields a smooth shift, the TV penalty applies a hard-thresholding, creating a discontinuous effect where predictions either perfectly merge or go back to the unfair ERM.}
  \label{fig:aware_composite}
\end{figure*}

\begin{figure*}[h!]
  \centering
  \includegraphics[width=0.9\textwidth]{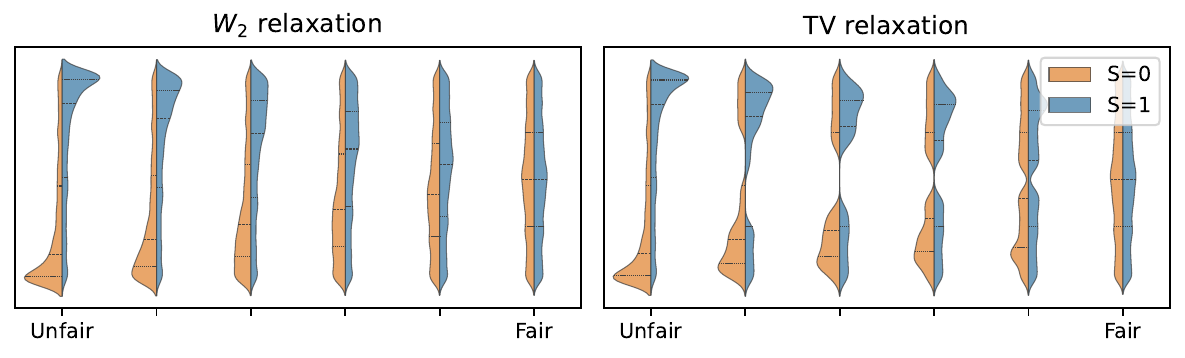}
  \caption{\textbf{Evolution of the aware marginal distributions.} Evolution of the predicted marginal distributions from Unfair to Fair. 
  The discontinuous nature of the aware TV hard-thresholding creates fragmented intermediate distributions, contrasting with the smooth continuous shift of the $\mathcal W_2$ penalty.}
  \label{fig:aware_evolution}
\end{figure*}

    \subsection{Impact of confounding strenght and trade-offs} \label{app:synthetic_dataset}

    We generate $n$ independent samples from a controlled 2-dimensional synthetic dataset.

    \paragraph{Sensitive attribute and features.} 
We sample the binary sensitive attribute uniformly at random $S \sim \text{Bernoulli}(0.5)$, and define the corresponding Rademacher variable $Z_S = 2(S - 0.5) \in \{-1, 1\}$.
We construct a 2-dimensional continuous feature vector $X = (X_1, X_2)$. 
The first feature, $X_1$, acts as the confounded signal and is generated exactly as described in \cref{app:empirical_illustrations}. 
The second feature, $X_2$, acts as a ``safe" predictive signal and is drawn from an independent Gaussian distribution
\begin{equation*}
    X_2 \sim \mathcal{N}(0, 1).
\end{equation*}
        
    \paragraph{Target variable.} 
    The continuous outcome $Y$ is generated as a simple linear function of both features, with an independent additive Gaussian noise $\varepsilon$
    \begin{equation*}
        Y = 2X_1 - X_2 + \varepsilon, \quad \text{where } \varepsilon \sim \mathcal{N}(0, 0.5^2).
    \end{equation*}
    
    \paragraph{Role of the parameter $\gamma$.} 
    This specific design introduces a single source of unfairness governed by the parameter $\gamma \in [0, 1]$, which dictates the \emph{confounding strength}. 
    When $\gamma = 0$, $X_1$ is completely independent of $S$.
     As $\gamma \to 1$, $X_1$ becomes increasingly confounded with the demographic groups. 
    Because $Y$ is a function of $X$, the unconstrained ERM will always converge to the true coefficients $[2.0, -1.0]$, and its MSE will always equal the noise variance $\text{Var}(\varepsilon) = 0.25$, regardless of $\gamma$. 
    This guarantees that any drop in accuracy observed during post-processing is exclusively attributable to the fairness constraints.
    
    \paragraph{Accuracy-fairness trade-offs.}
    We first fix $\gamma = 0.5$ to induce a moderate structural bias and evaluate the relaxed predictors. 
    As shown in \cref{fig:tradeoff_synthetic}, the unconstrained ERM achieves the theoretical optimal MSE of $0.25$ but exhibits high unfairness. 
    As we tighten the fairness constraints, FAUST-U recovers the Pareto frontier of the optimal aware interpolations. 
Notably, the naive plug-in estimator (Plug-In W2) fails to achieve the same fairness levels, demonstrating the risk of relying on a hard proxy classification $\widehat{S}$ rather than the continuous group probabilities.

    \paragraph{Impact of confounding strength.}
    To evaluate robustness across different biases, we fix the constraint to exact DP (no relaxation) and vary the confounding strength $\gamma$ from $10^{-2}$ to $10^0$ (\cref{fig:confounded_sweep_abs,fig:confounded_sweep_rel}). 
    Note that because the empirical $\mathcal{W}_2$ distance is computed between finite samples there is a small positive value due to sampling noise even for fair models, which is why the models do not reach zero on the logarithmic scale.
    This reveals a characteristic of the unaware geometry: it requires an active demographic signal to construct the 2D pseudo-measures. 
    When $\gamma$ is low, $S$ is independent of $X$, causing the estimated density ratio $\hat{\Delta}(x)$ to vanish. 
FAUST-U collapses the predictions toward the global mean (incurring a high MSE) to guarantee fairness. 
    
    However, as bias appears ($\gamma > 0.1$), the unaware geometry fully resolves. 
    When there is high-confounding, FAUST-U suppresses the reliance on the biased feature while preserving the signal from $X_2$, converging to the optimal bounds of FAUST-A. 
    Furthermore, the naive plug-in approach (Plug-In W2) is not robust to this confounding, even at moderate levels.
    Because it forces a hard classification of the sensitive attribute, it applies the optimal transport plan based on misaligned proxy labels. This results in a massive spike in true $\mathcal{W}_2$ unfairness, only finally converging to the Aware bounds when $\gamma \to 1$ and the proxy becomes an oracle. Conversely, the FairReg baseline shows instability, with MSE penalties early in the trajectory.
    
    \paragraph{Variance analysis.}
    This behavior at low confounding is confirmed by analyzing the within-group variance of the predictions (\cref{fig:confounded_sweep_var}). 
    At low $\gamma$, the variance of the FAUST-U predictions collapses significantly compared to the ERM, reflecting the prediction of the constant prediction in the absence of a reliable demographic signal. 
    As the confounding signal strengthens ($\gamma \to 1$), all fair models reduce their variance to suppress $X_1$, but FAUST-U demonstrates the most stable convergence toward the aware optimum.

\begin{figure*}[h!]
    \centering
    \includegraphics[width=\textwidth]{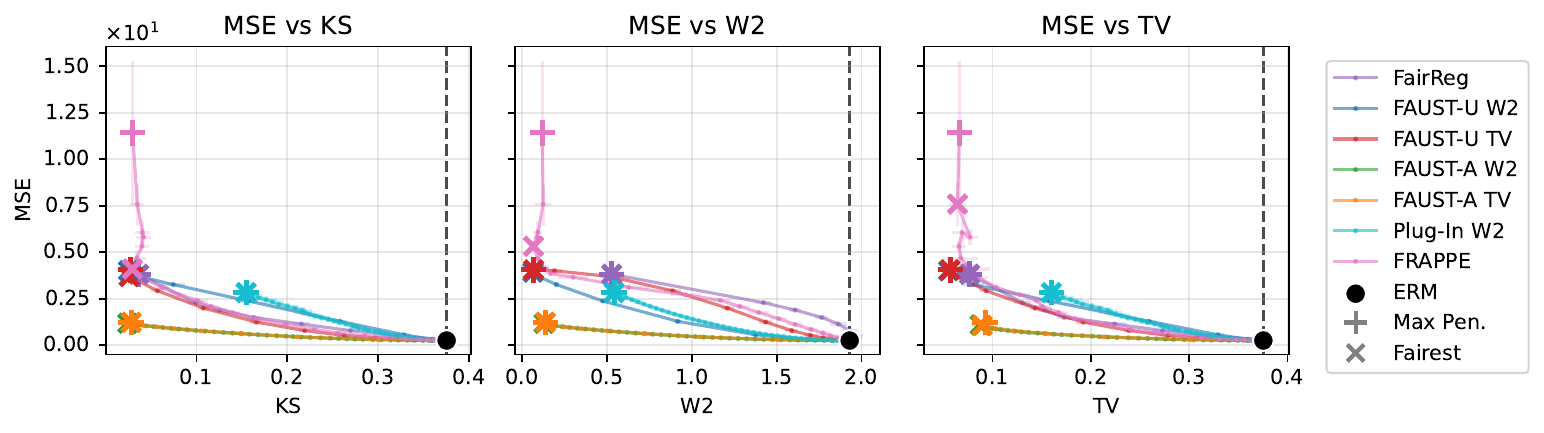}
        \includegraphics[width=\textwidth]{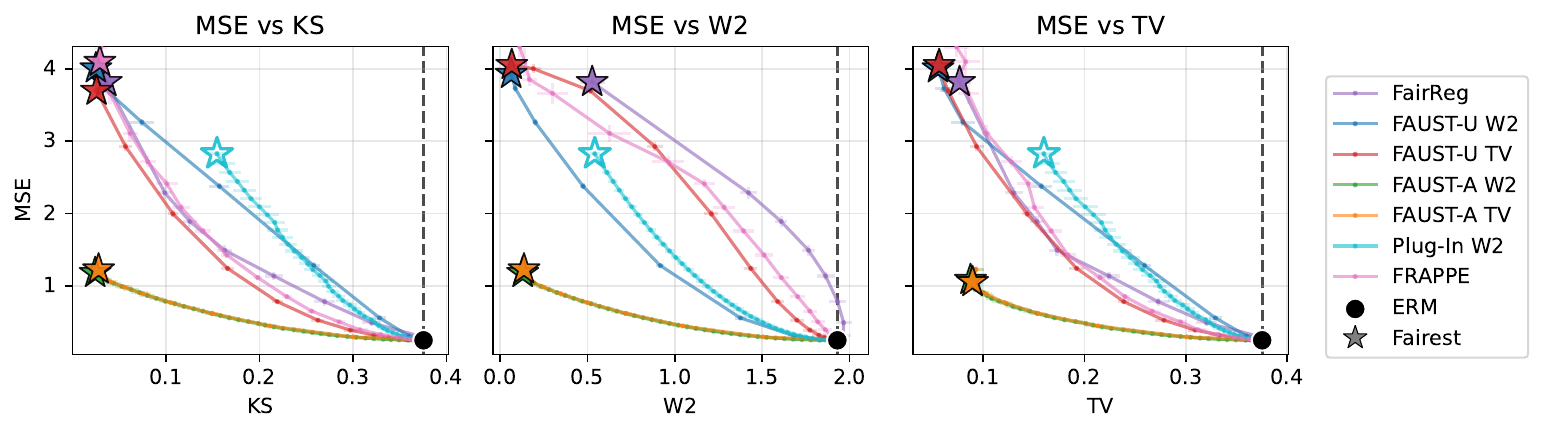}
\caption{\textbf{Accuracy-fairness trade-offs on the synthetic dataset ($\gamma=0.5, n=10,000$).} Relaxation trajectories from the unconstrained ERM (black dot) to the fair models. The top panel shows the full vertical scale capturing FRAPPÉ's extreme MSE degradation; the bottom panel provides a zoom on the remaining baselines. The unconstrained baseline converges to the noise limit ($\text{MSE} = 0.25$). Our exact mappings (FAUST-U $\mathcal{W}_2$ and FAUST-U TV) recover the optimal aware Pareto frontier (FAUST-A) and outperform FairReg and the noisy plug-in proxy (Plug-In W2).}
    \label{fig:tradeoff_synthetic}
\end{figure*}

\begin{figure*}[h!]
  \centering
  \includegraphics[width=\textwidth]{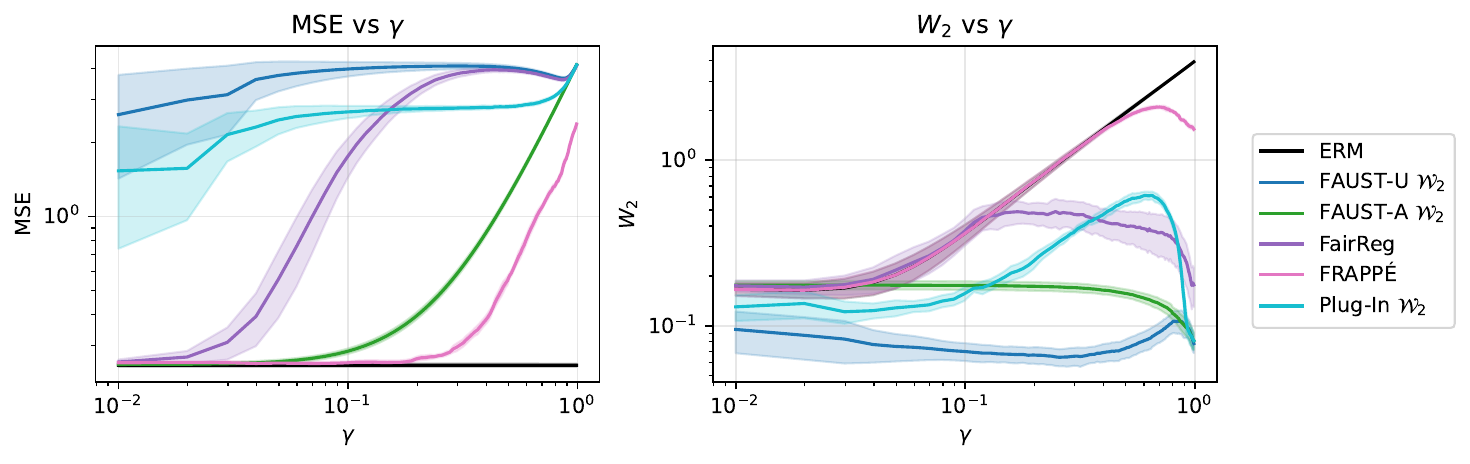}
\caption{\textbf{Impact of confounding strength (Absolute).} Evolution of the MSE (left) and $\mathcal{W}_2$ unfairness (right) under exact DP. As bias increases ($\gamma \to 1$), FAUST-U converges to the theoretical aware optimum (FAUST-A), bounding the MSE degradation. In contrast, FairReg suffers from optimization instability, which degrades the MSE. FRAPPÉ shows instability, not minimize unfairness efficiently while showing a spike in MSE at high confounding levels.}
  \label{fig:confounded_sweep_abs}
\end{figure*}

\begin{figure*}[h!]
  \centering
  \includegraphics[width=\textwidth]{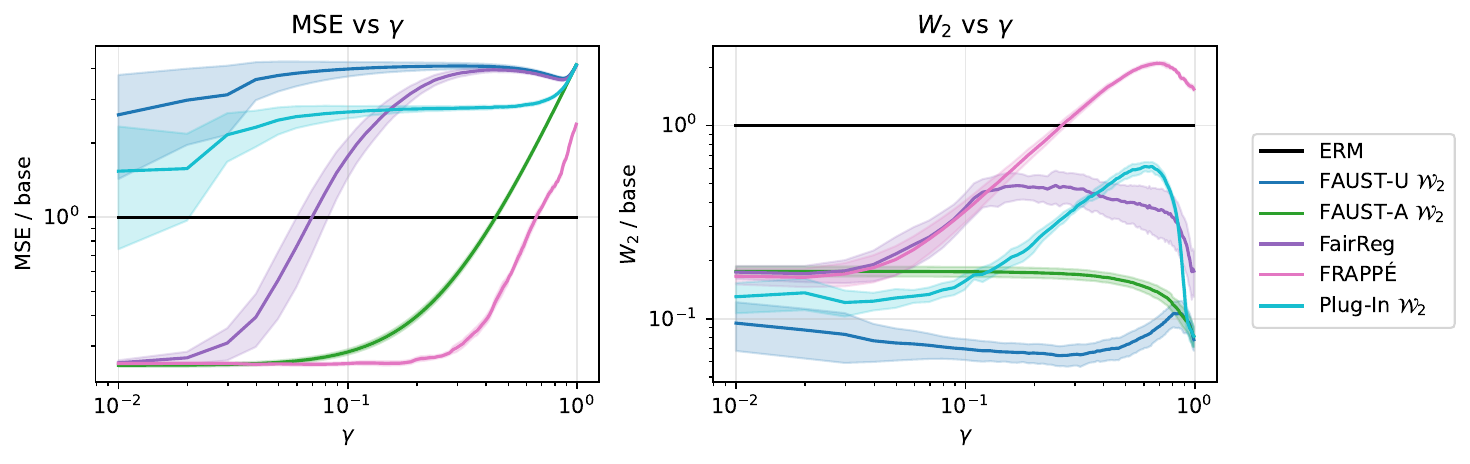}
\caption{\textbf{Impact of confounding strength (Relative).} Evolution of the MSE and $\mathcal{W}_2$ unfairness expressed relative to the unconstrained ERM baseline. It highlight that both FairReg and FRAPPÉ's MSE drop even at moderate confounding levels ($\gamma \approx 0.1$), whereas FAUST-U maintains stability and accurately tracks the theoretical aware bound across all bias intensities.}
  \label{fig:confounded_sweep_rel}
\end{figure*}

\begin{figure*}[h!]
  \centering
  \includegraphics[width=\textwidth]{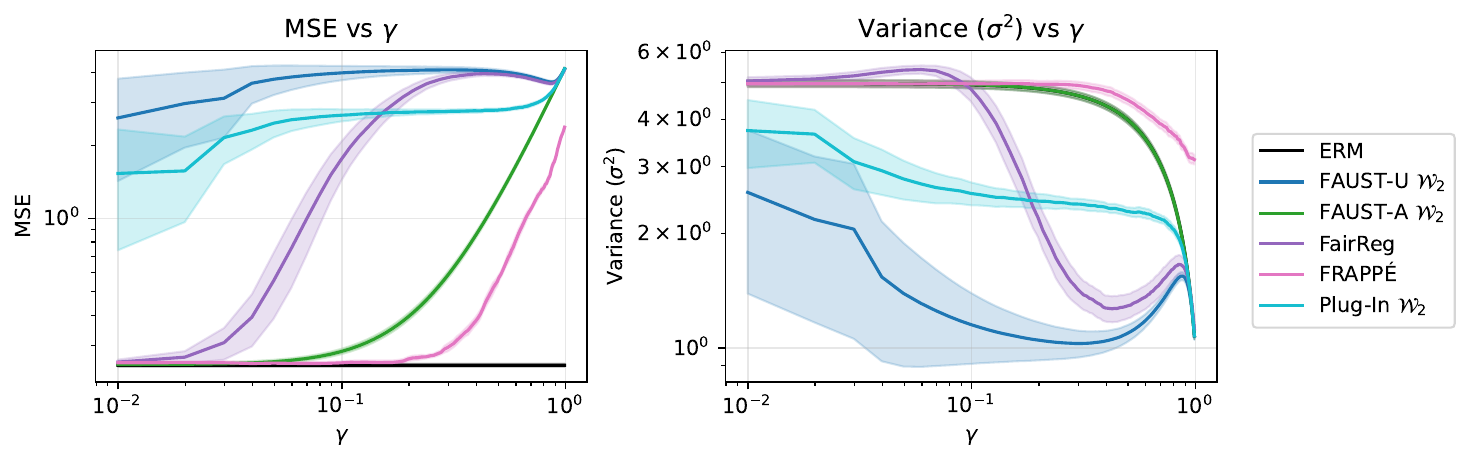}
\caption{\textbf{Within-Group Variance Analysis.} Evolution of the MSE (left) and within-group prediction variance (right) across varying confounding strengths ($\gamma$). At low $\gamma$, FAUST-U collapses the prediction variance to act as a constant predictor, reflecting the uninformative demoigrphic signal $\hat{\Delta}(x)$. As bias increases, FAUST-U resolves the 2D geometry and recovers the variance of the optimal FAUST-A model, whereas baselines like FRAPPÉ show an artificizl variance into the predictions to satisfy gradient penalties.}
  \label{fig:confounded_sweep_var}
\end{figure*}

    \section{Experiments on real-world datasets}
\label{app:adult_dataset}

To complement the real-world evaluations in the main text, we provide the full accuracy-fairness trade-off analysis for the Adult (Census Income)  and plot the full plot to see the whole FRAPPÉ baseline on the Communities \& Crime and Law School datasets .

\subsection{Adult dataset}

\paragraph{Adult dataset detail.}
The Adult dataset comprises individual-level records extracted from the 1994 U.S. Census database. It contains 48,842 total observations characterized by 14 demographic and employment-related attributes, including education level, marital status, occupation, and hours worked per week. While the usual prediction task for this dataset is binary classification (predicting whether an individual earns more than \$50K per year), \cite{Fairreg} convert this into a regression task by using a numeric proxy from the features (age) as the continuous target $Y$. 

The dataset contains demographic disparities: the population is approximately $66.9\%$ male and $33.1\%$ female. The racial composition of the dataset is also unbalanced, comprising approximately $85.4\%$ White, $9.6\%$ Black, $3.2\%$ Asian-Pac-Islander, $1.0\%$ Amer-Indian-Eskimo, and $0.8\%$ Other. For this experiment, the sensitive attribute considered is gender.

Characteristics about the \textbf{Law School} and \textbf{Communities and Crime} datasets are provided in the main text \cref{sec:experiments}.


\subsection{Accuracy-Fairness trade-off}
\label{app:frappe_tradeoffs}

To complement our main evaluations, we present the full accuracy-fairness trade-off curves comparing our framework against the FRAPPÉ baseline \cite{tifrea2024frappe}. As detailed in Section 4.1, FRAPPÉ acts as a relevant unaware post-processing baseline, but it relies on stochastic gradient descent to learn an additive post-hoc correction, rather than relying on exact continuous geometric mappings. We evaluated FRAPPÉ on all four datasets using the exact same base ERM regressors as our models, sweeping its penalty hyperparameter to trace its complete Pareto frontier.

\begin{figure*}[h!]
    \centering
    \includegraphics[width=1\textwidth]{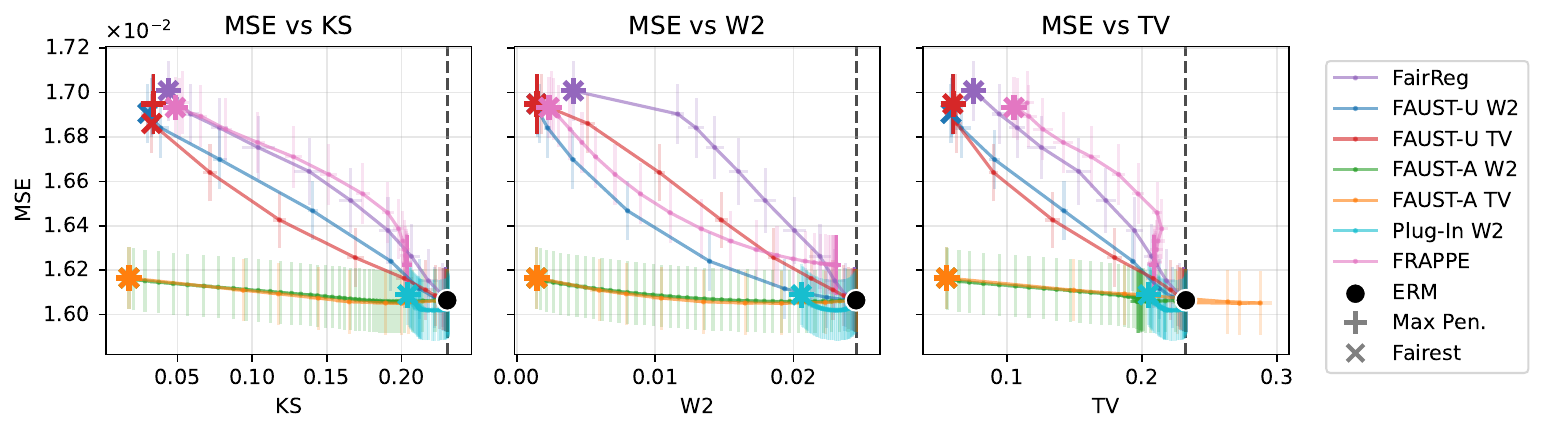}
    \vspace{0.2cm}
    \includegraphics[width=1\textwidth]{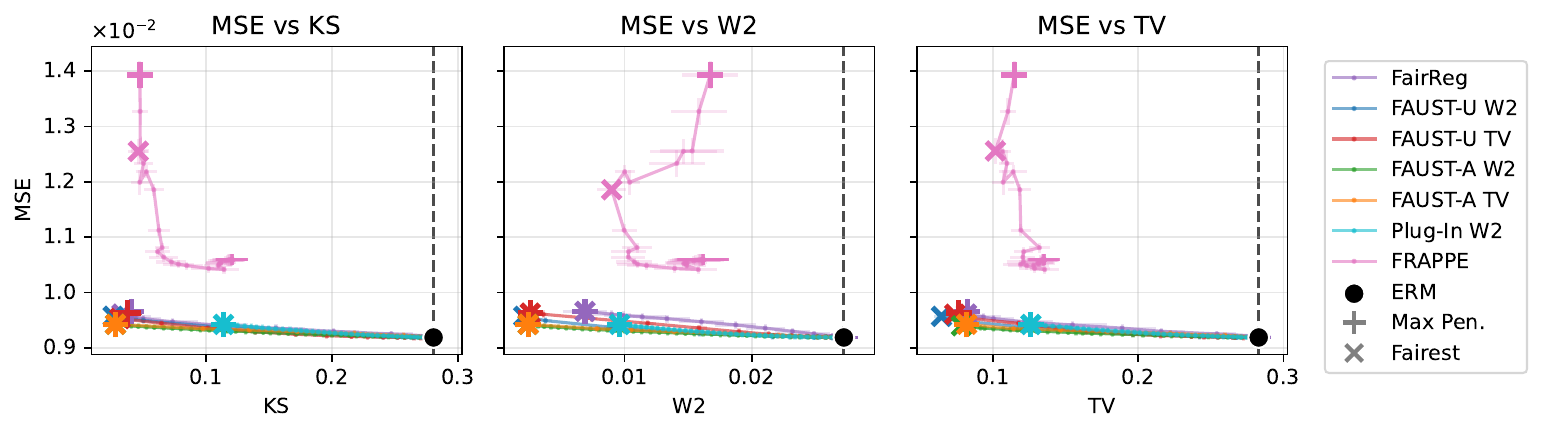}
    \vspace{0.2cm}
    \includegraphics[width=1\textwidth]{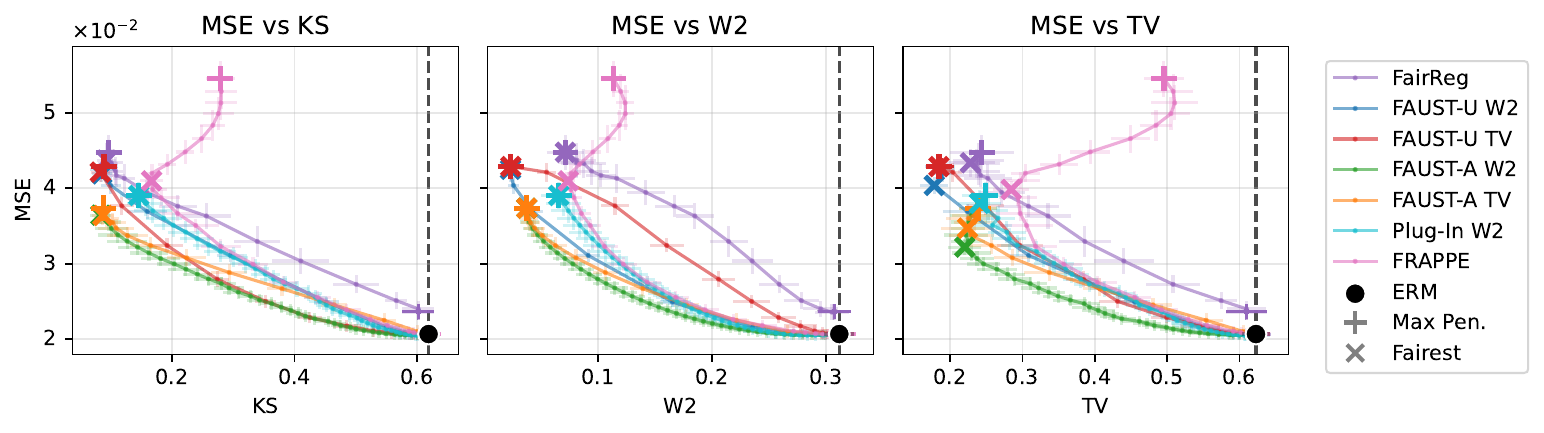}
    \caption{\textbf{Extended Accuracy-Fairness Trade-offs.} Relaxation trajectories across the evaluated datasets (from top to bottom: Adult, Law School and Communities and Crime). The light pink curve represents the FRAPPÉ baseline. Across all datasets, FRAPPÉ struggles to reach the theoretical fairness optimum (the star markers) without severe predictive degradation.}
    \label{fig:frappe_tradeoffs}
\end{figure*}

\paragraph{Analysis of the trade-off geometry.}
As illustrated by the Pareto frontiers in \cref{fig:frappe_tradeoffs}, the generated curves show that FRAPPÉ suffers from optimization instability. As the fairness constraint tightens, FRAPPÉ's frontier forms a "hook", its MSE degrades without enforcing parity. This empirically validates the advantage of using closed-form mappings over gradient-based neural network approximations to enforce fairness constraints.

    \subsection{Evaluation of the fully constrained models} \label{app:fully_constrained}

    \paragraph{Experimental setup.}
    \cref{tab:summary_relative_merge_train_unlab_adult,tab:summary_relative_merge_train_unlab_lawschool,tab:summary_relative_merge_train_unlab_communities} present the performance of the fully constrained models relative to the unconstrained ERM (MSE $> 1.00$ indicates degradation; fairness metrics near $0.00$ indicate perfect parity). 
    Alongside FairReg, FRAPPÉ and FAUST-U, we evaluate two naive plug-in baselines. 
    If $S$ is unavailable at inference, an intuitive solution is to substitute it with an estimator $\widehat{S}$ (or $\widehat{p}(s\mid x)$) inside the aware optimal transport maps $T_s$, yielding
    \begin{align}
      \widehat Y_{\mathrm{hard}}(x)=&T_{\widehat S(x)}(\eta(x)), \label{eq:y_hard}\\
      \widehat Y_{\mathrm{soft}}(x)=&\sum_{s}\widehat p(s\mid x)\,T_s(\eta(x)) \label{eq:y_soft}.
    \end{align}

    \paragraph{Results.} 
    On the Adult dataset (\cref{tab:summary_relative_merge_train_unlab_adult}), FAUST-U consistently matches or outperforms FairReg across all metrics. 
    While FairReg achieves the lowest fairness measure in terms of KS-distance between the predictions of the group, at equality with our FAUST-U method, it still suffers from a high $\mathcal{W}_2^2$-unfairness (23\% relative to the baseline), whereas FAUST-U reduces it to 7\% for the exact same relative MSE cost ($1.06$). 
    FAUST-U also matches FairReg on TV and KS metrics, proving that exact optimal transport naturally guarantees statistical parity without the discretization.

    \paragraph{Failure of naive plug-in estimators.}
    The naive plug-in approaches completely fail to enforce fairness. 
    On the Adult dataset (\cref{tab:summary_relative_merge_train_unlab_adult}), the $\mathcal W_2^2$ and TV metrics for $\widehat Y_{\mathrm{hard}}$ and $\widehat Y_{\mathrm{soft}}$ remain between $0.85$ and $1.10$, making them virtually indistinguishable from the biased ERM. 
    Because aware optimal transport relies on exact group rankings, classification errors from $\widehat{S}$ prevent achieving fairness. 
    In contrast, FAUST-U successfully enforces parity by mathematically absorbing this classification uncertainty directly into its global density weights, succeeding where plug-in estimators fails.

    \input{images/relaxed_curves/adult/summary_relative_merge_train_unlab.tex}
    
    \input{images/relaxed_curves/lawschool/summary_relative_merge_train_unlab.tex}

    \input{images/relaxed_curves/communities/summary_relative_merge_train_unlab.tex}

\subsection{Robustness to estimation error}\label{app:estimation_error}

Because our framework is a post-processing algorithm operating in the unaware setting, its performance relies on the base estimators for the conditional expectation $\hat{\eta}(x)$ and the posterior group probabilities $\hat{\mathbb{P}}(S \mid X)$ used to construct $\hat{\Delta}(x)$. In practice, estimation errors in these quantities are coupled and stem from limited training data. 

\paragraph{Finite-sample study.}
To quantify the sensitivity to these finite-sample estimation errors, we restricted the training sample size $N$ used to fit the base estimators and the optimal transport map. We evaluated the evolution of the test MSE, the strict $W_2$ unfairness, and the tradeoff frontier across training sizes ranging from $N = 100$ up to the full dataset limits. All points report the mean and standard deviation across 10 independent random data splits.

\begin{figure*}[h!]
    \centering
    \includegraphics[width=1\textwidth]{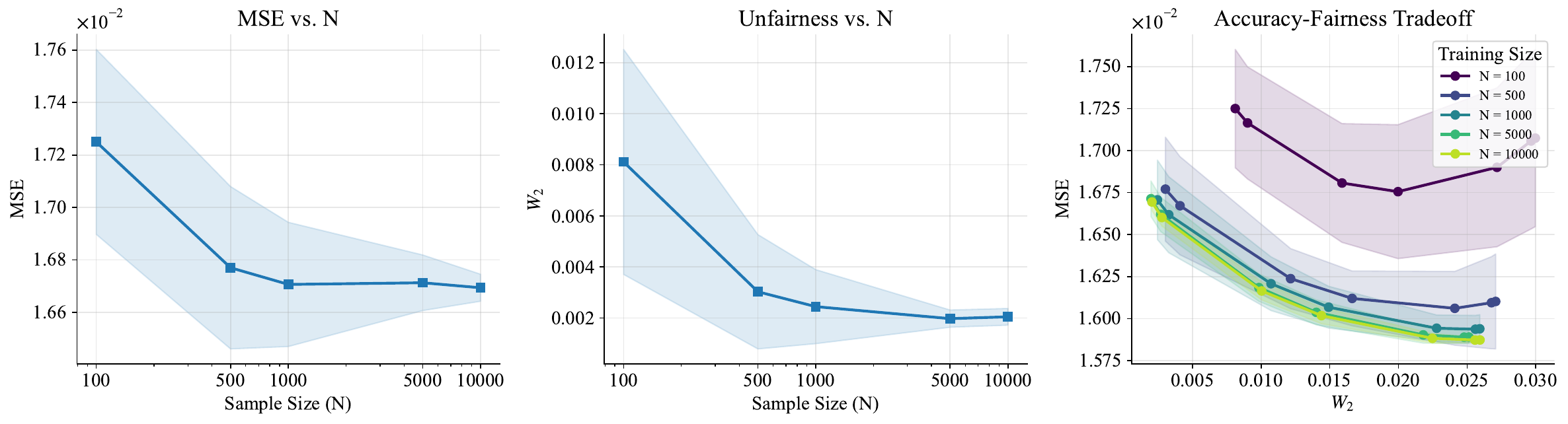}
    \vspace{-0.6cm}
    \caption{\textbf{Finite-sample performance on the Adult dataset.} 
    \emph{Left \& Middle:} Test MSE and $W_2$ unfairness gap of the strictly fair FAUST-U model as a function of $N$. 
    \emph{Right:} Accuracy-fairness tradeoff Pareto frontiers. The Adult dataset exhibits stable corrections, the Pareto frontier going toward the optimal lower-left region as data increases.}
    \label{fig:finite_sample_adult}
\end{figure*}

\begin{figure*}[h!]
    \centering
    \includegraphics[width=1\textwidth]{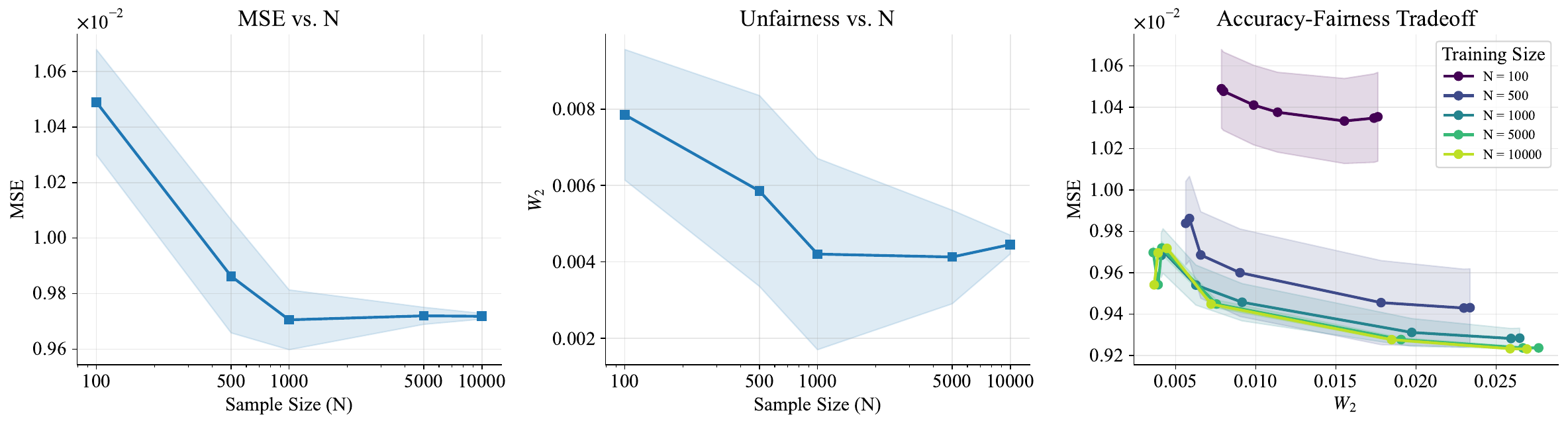}
    \vspace{-0.6cm}
    \caption{\textbf{Finite-sample performance on the Law School dataset.} 
    Noticeably, under low-data regimes and strict fairness constraints (the rightmost points of the Pareto frontier), the lack of accurate density estimations causes a slight geometric distortion (a "hook" where MSE degrades without improving $W_2$). As $N$ increases to 10,000 (light green curve), this artifact completely vanishes, yielding a strictly monotonic and optimal frontier.}
    \label{fig:finite_sample_lawschool}
\end{figure*}

\begin{figure*}[h!]
    \centering
    \includegraphics[width=1\textwidth]{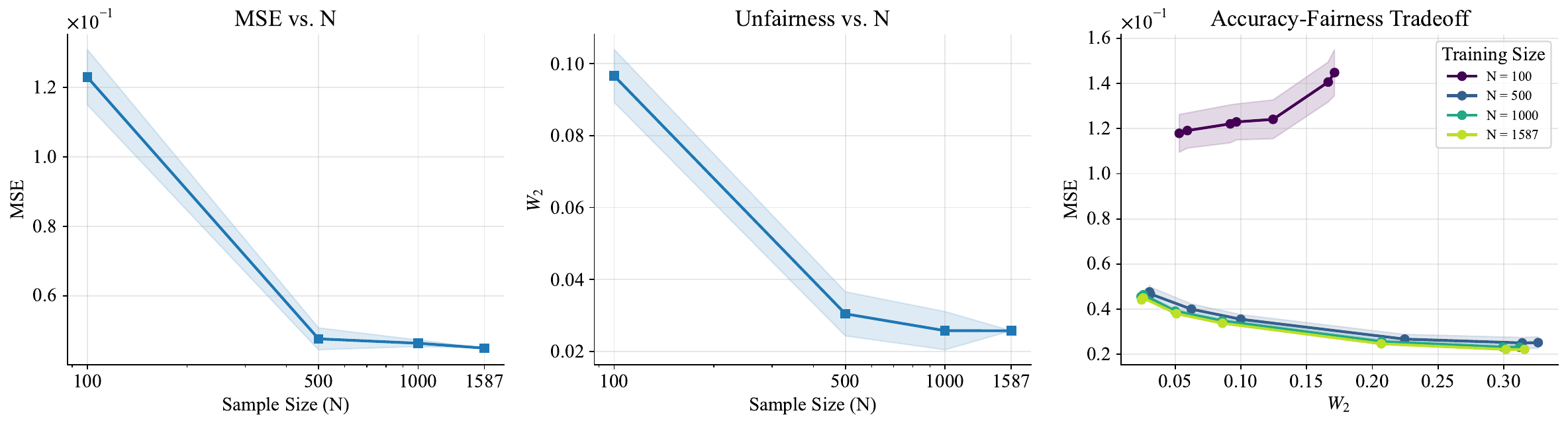}
    \vspace{-0.6cm}
    \caption{\textbf{Finite-sample performance on the Communities and Crime dataset.} 
    Because this dataset contains 128 features, the $N=100$ regime is nearly underdetermined, resulting in high estimation variance and a severely degraded frontier. However, with just $N=500$ samples, the base estimators stabilize, and the optimal transport maps successfully enforce parity, rapidly converging to the optimal frontier up to the full dataset size ($N=1587$).}
    \label{fig:finite_sample_communities}
\end{figure*}

\paragraph{Analysis of the results.}
Across all datasets, the results demonstrate that our OT framework is highly robust to estimation errors. Even with low data (e.g., $N=100$) where the base estimators $\hat{\eta}$ and $\hat{\Delta}$ are noisy, the discrete optimal transport map enforces corrections, reducing the $W_2$ unfairness compared to the unconstrained ERM without impactong utility. While extreme noise can induce loss in performance under exact fairness constraints, such in the tail of the Law School Pareto frontier (\cref{fig:finite_sample_lawschool}),the overall tradeoffs remain resilient. Furthermore, the framework exhibits sample efficiency: it requires very little data to overcome this noise and recover the optimal population-level mapping. For instance, despite the Communities dataset containing 128 features (\cref{fig:finite_sample_communities}), $N=500$ samples are sufficient for the 2D pseudo-measures to resolve and match the full-dataset Pareto frontier.

\paragraph{Noise injection study.}
To further isolate the sensitivity of the post-processing framework to estimation errors, we conducted an artificial noise injection study. We introduced increasing levels of independent Gaussian noise ($\sigma \in [0.0, 0.5]$) into the base estimators $\hat{\mathbb{P}}(S \mid X)$ and $\hat{\eta}(X)$ separately, while keeping the fairness constraint strict ($\lambda \to \infty$). As shown in \cref{fig:noise_injection_metrics}, the framework demonstrates strong robustness to miscalibration.

\begin{figure*}[h!]
    \centering
    \includegraphics[width=0.9\textwidth]{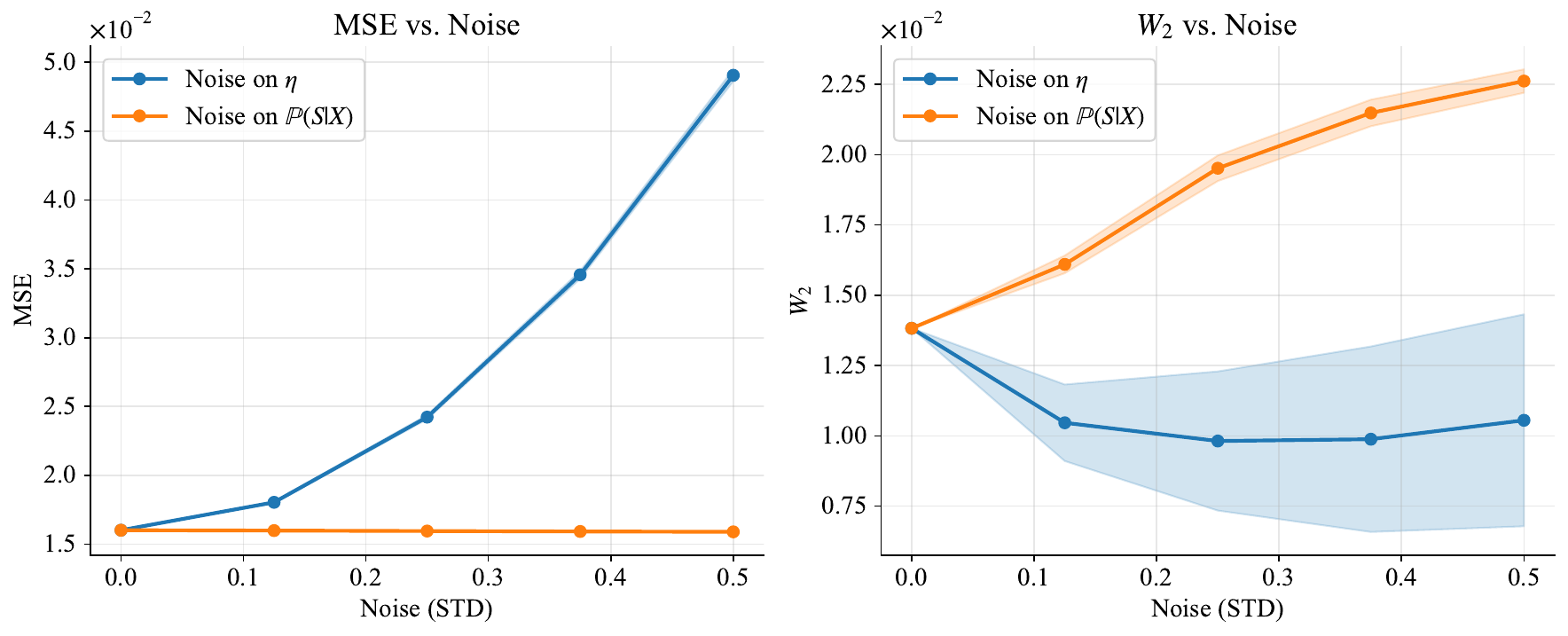}
    \vspace{0.2cm}
    \includegraphics[width=0.9\textwidth]{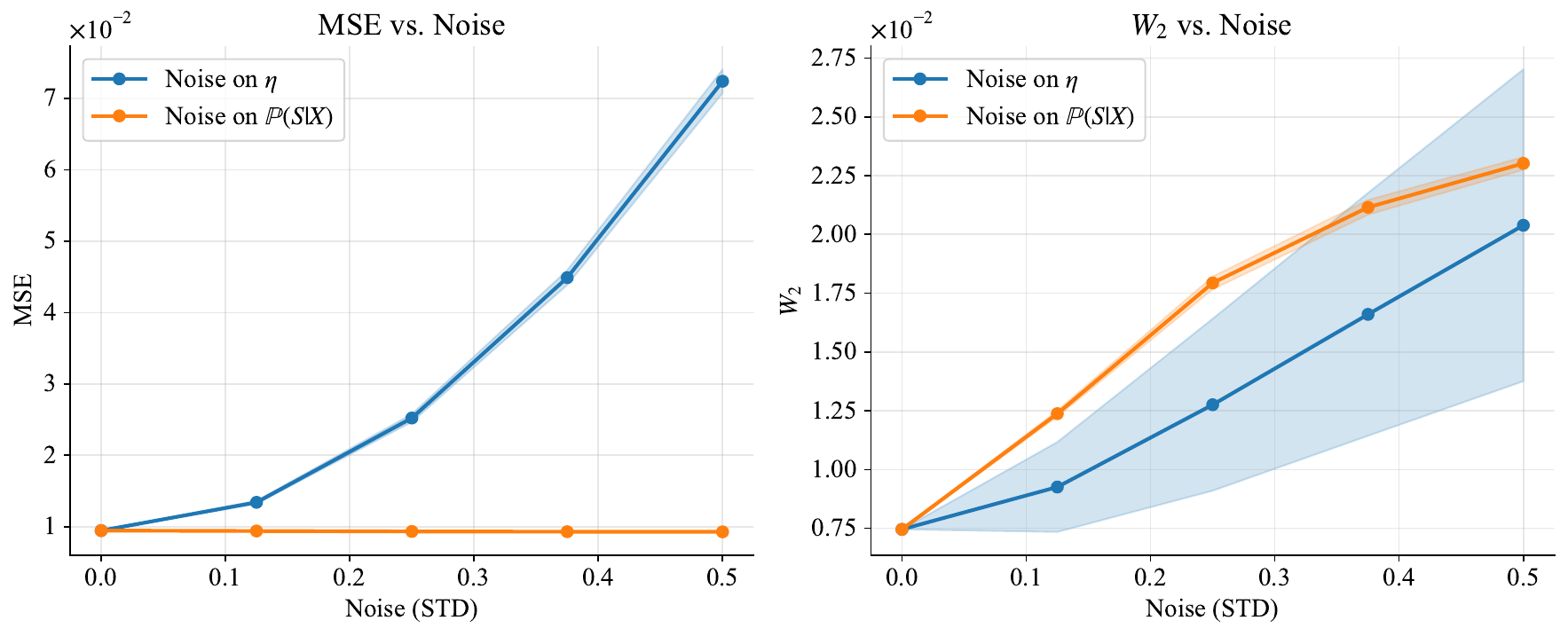}
    \vspace{0.2cm}
    \includegraphics[width=0.9\textwidth]{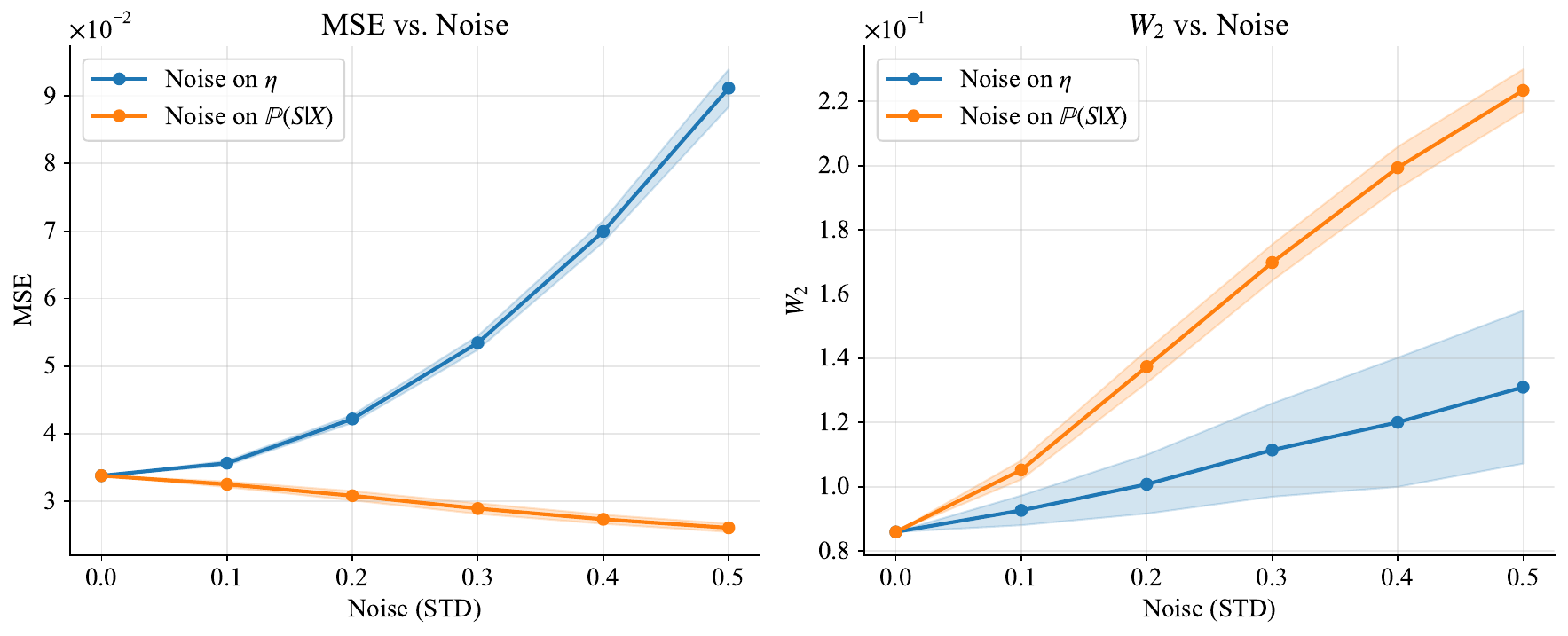}
    \caption{\textbf{Isolated metric degradation under estimator noise.} Evolution of the MSE (left column) and $\mathcal{W}_2$ gap (right column) of the fair FAUST-U model when artificial Gaussian noise is injected into either $\hat{\eta}$ (blue) or $\hat{\mathbb{P}}(S \mid X)$ (orange). The rows correspond to the Adult, Law School and, Communities and Crime datasets, respectively. Shaded regions represent $\pm 1$ standard deviation across 5 independent runs.}
    \label{fig:noise_injection_metrics}
\end{figure*}

When noise is applied to the posterior $\hat{\mathbb{P}}(S \mid X)$ (orange curves), the vertical density of the 2D space $(\eta, \Delta)$ is artificially smoothed. Consequently, the empirical pseudo-distributions appear more aligned prior to transport, causing the algorithm to under-correct. Fairness smoothly degrades (the $\mathcal{W}_2$ gap increases toward the unconstrained baseline), but this error does not significantly alter the predictive utility, leaving the MSE stable. 

Conversely, when noise affects the base predictions $\hat{\eta}$ (blue curves), the OT map aligns the noisy distributions, maintaining a stable $\mathcal{W}_2$ gap near zero. However, the MSE degrades because the optimal transport map propagates the noisy base signals to satisfy the constraint. 

Crucially, in both scenarios and across all datasets, the degradation is contzined. The framework does not fail under significant perturbation (up to $\sigma=0.5$), confirming its viability and stability for real-world applications where base estimators are imperfect.

\subsection{Sensitivity of fairness metrics to discretization grids} \label{app:metric_sensitivity}

To evaluate DP on continuous predictions, metrics such as the Kolmogorov-Smirnov (KS) distance and the discrete Total Variation (TV) require binning the predictions into a finite grid. To ensure our empirical conclusions are robust to this evaluation hyperparameter, we tested the sensitivity of the metrics across fixed linear grids of 20, 50, and 100 bins, as well as an adaptive 50-bin setup (quantile-based binning to ensure equal probability mass per bin).

\begin{figure*}[h!]
    \centering
    \includegraphics[width=\textwidth]{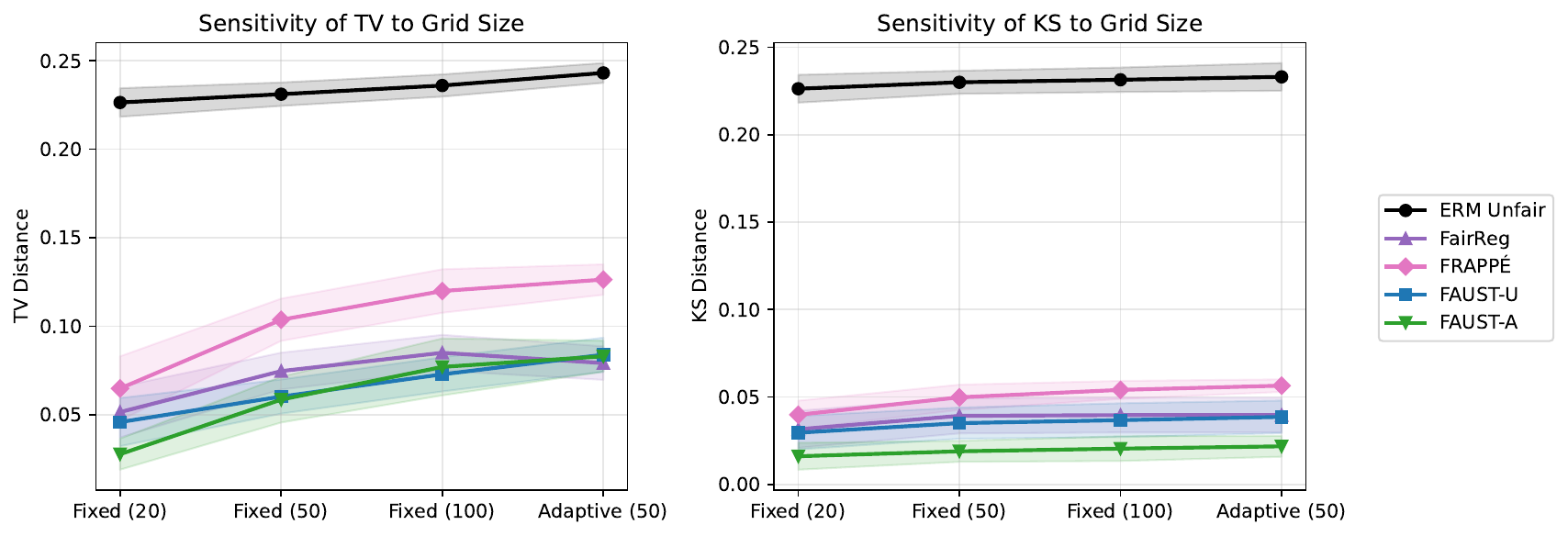}
    \includegraphics[width=\textwidth]{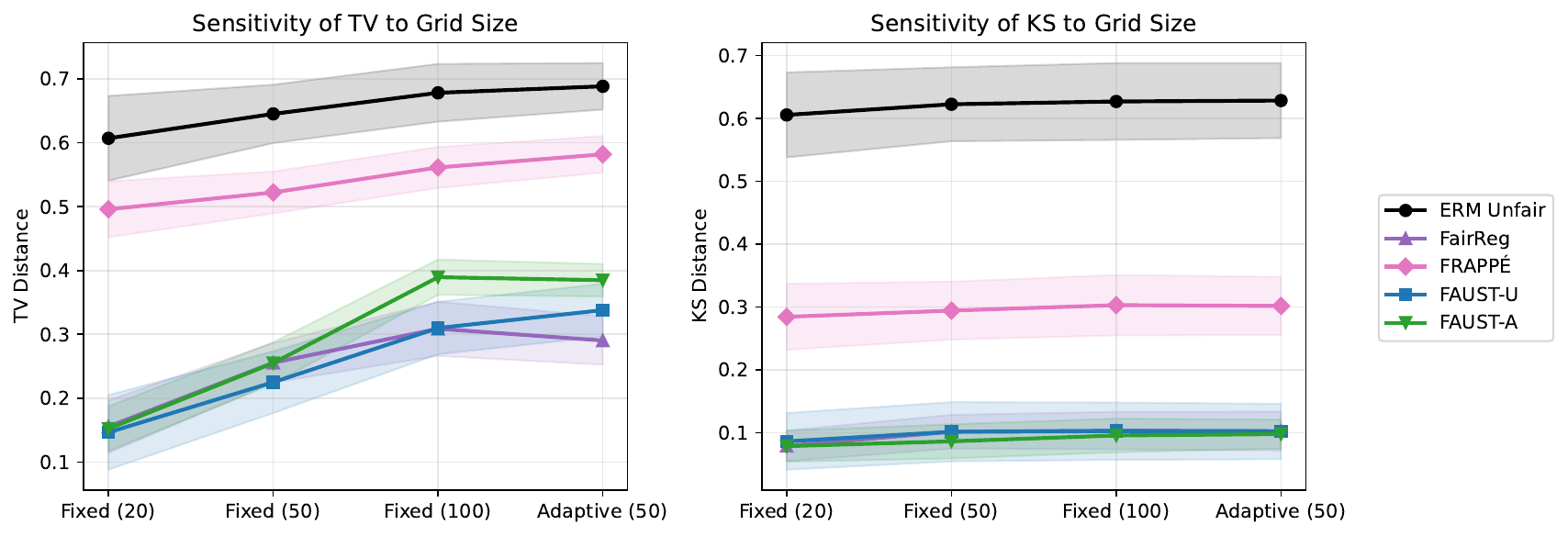}
    \includegraphics[width=\textwidth]{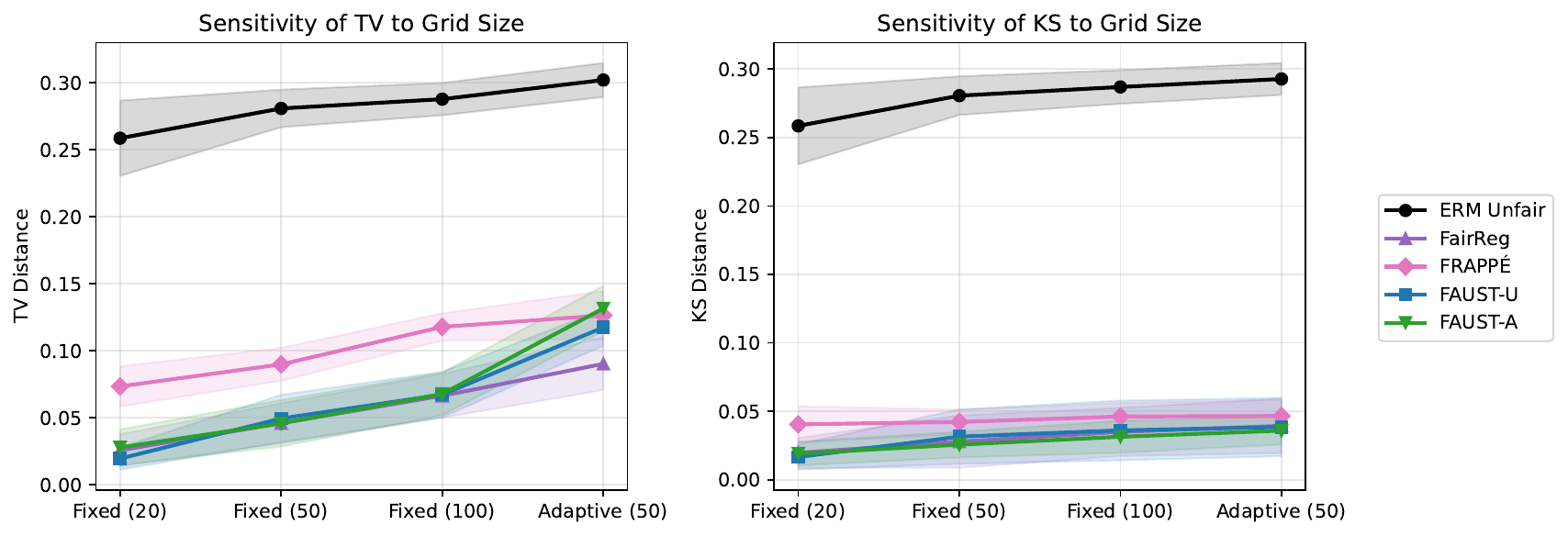}
    \includegraphics[width=\textwidth]{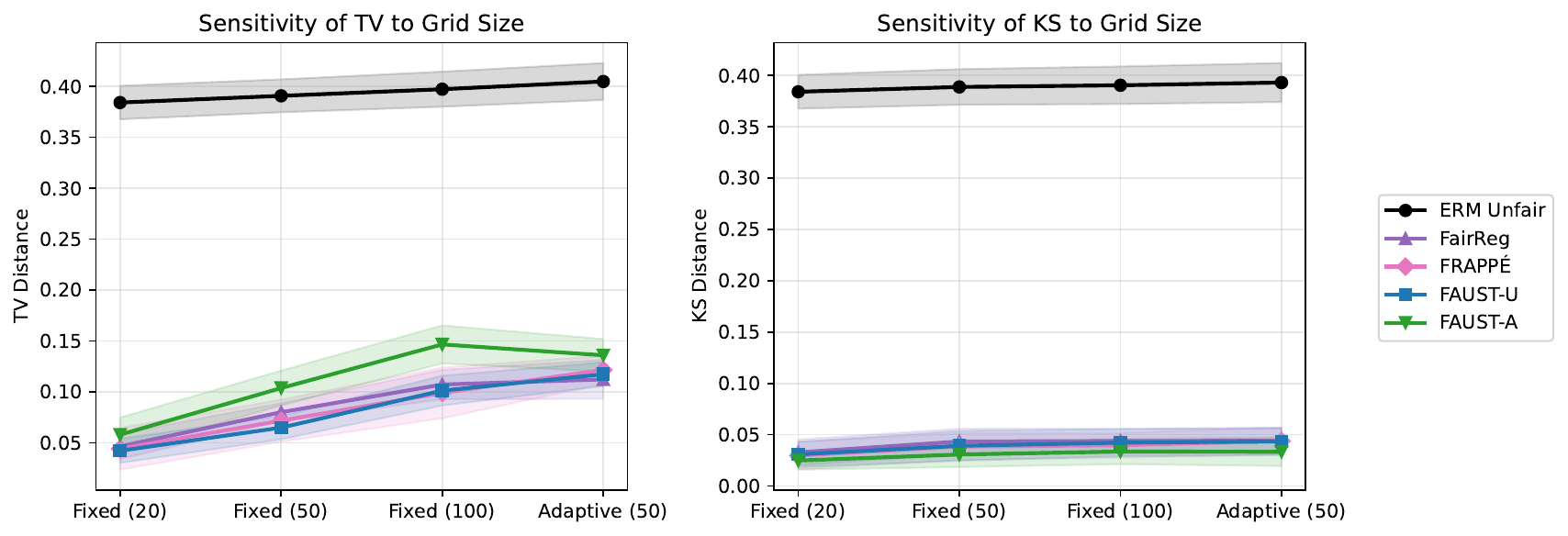}
\caption{\textbf{Sensitivity of TV and KS metrics to evaluation bins.} While the absolute values of the metrics (especially the discrete TV) change with grid resolution, the relative performance of the methods remains stable across all configurations. The rows show the Adult, Communities and Crime, Law School, and Synthetic datasets, from top to bottom. Across all datasets, FAUST-U outperforms FRAPPÉ and matches or exceeds the FairReg baseline. Crucially, FAUST-U tracks the theoretical optimal bound (FAUST-A) regardless of the chosen evaluation grid.}
    \label{fig:metric_sensitivity_all}
\end{figure*}

As illustrated in \cref{fig:metric_sensitivity_all}, the absolute values of the metrics, particularly TV, are sensitive to the chosen discretization, increasing as the evaluation grid becomes finer. However, despite this absolute variance, the relative ranking of the methods is invariant across all datasets. Across every tested configuration, our FAUST-U framework dominates both the unconstrained ERM and the FRAPPÉ baseline. Furthermore, it consistently matches or outperforms the grid-based baseline (FairReg), which explicitly optimizes for a discretized grid, remaining close to the theoretical optimal bound established by the Aware OT model. This confirms that our performance gains are highly robust to the choice of the evaluation grid.

    \subsection{Robustness across different base estimators}
    \label{app:base_models_benchmark}
    
    To ensure the performance of the post-processing methods is not an artifact of the initial model class, we benchmark the fully constrained frameworks across three distinct unconstrained base estimators: Ordinary Least Squares (Linear), Random Forests (RF), and Gradient Boosting Machines (GBM). 
    
    \paragraph{Methodology.}
    For each base estimator, we extract the unconstrained predictions (ERM) and process them using both the FairReg baseline and our FAUST-U framework. 
    To ensure a rigorous comparison, we report the mean and standard deviation across 10 random train/test splits of size.
    For the three unaware methods, the strictly superior performance is highlighted in \textbf{bold}. 
    To account for variance, we apply a Wilcoxon signed-rank test; if a method underperforms the best method but the difference is not statistically significant ($p > 0.05$), it is highlighted in \textcolor{blue}{blue}.

    \paragraph{Metrics.}
    We recall that $\text{KS}^{\mathcal{T}}$ is a discretized approximation of the Kolmogorov-Smirnov (KS) distance, computed as the maximum absolute c.d.f. difference across a discrete 50-bin grid $\mathcal{T}$. 
    We also report the continuous Wasserstein-2 ($\mathcal{W}_2$) and Total Variation (TV) distances.
    The Total Variatioin is computed over 50 bins.

    \paragraph{Base model configurations.}
    To ensure reproducibility, we use the following hyperparameters for our unconstrained base estimators (implemented via scikit-learn):
    \begin{itemize}
        \item \textbf{Linear}: Ordinary Least Squares with an intercept (and Logistic Regression with \texttt{max\_iter=2000} for auxiliary classification).
        \item \textbf{Random Forest (RF)}: 100 trees restricted to a minimum of 20 samples per leaf (\texttt{min\_samples\_leaf=20}) to prevent severe overfitting.
        \item \textbf{Gradient Boosting Machine (GBM)}: 100 boosting stages with a maximum tree depth of 3 (\texttt{max\_depth=3}).
    \end{itemize}

\paragraph{Performance and interpretation.}
    As shown in \cref{tab:lawschool_bases,tab:adult_bases,tab:communities_bases}, FAUST-U demonstrates robust generalization across all model classes. 
    On the individual-level datasets (Adult and Law School), FAUST-U systematically achieves superior $W_2$ and KS$^{\mathcal{T}}$. 
    
    On the Communities dataset (\cref{tab:communities_bases}), FAUST-U proves even more dominant. 
    While it maintains highly competitive predictive utility (e.g., an MSE of $0.045$ vs. $0.043$ for FairReg on the Linear model), it successfully enforces the geometric fairness constraint. FairReg struggles to align the distributions, leaving a $W_2$ gap of $0.073$ for the Linear model, whereas FAUST-U successfully reduces it to $0.031$.

    \paragraph{Theoretical validation of the MSE degradation.}
    The main observation in the Adult and Law School results is that the empirical Mean Squared Error (MSE) appears unchanged between the Unfair ERM and the exact fair OT predictors (e.g., remaining at $0.01$ or $0.02$). 
    This is not an anomaly, but a direct empirical validation of the theoretical bounds of Wasserstein fair regression.
    
    The theoretical maximum MSE degradation required to achieve exact DP is bounded by the squared Wasserstein distance ($W_2^2$) between the initial biased distributions. 
    For the Law School dataset, the initial $W_2$ gap is approximately $0.03$. 
    The exact geometric cost to repair this disparity is therefore bounded by $0.03^2 = 0.0009$. 
    Because the targets are scaled to $[-1, 1]$, the base ERM risk is around $0.010$. 
    Adding the absolute maximum fairness penalty yields an expected fair risk of $0.0109$. 
    Consequently, the degradation is practically invisible at two decimal places. 
    This proves that Optimal Transport successfully enforces parity with the absolute minimum possible degradation to accuracy.
    
    Conversely, for the Communities and Crime dataset, the initial baseline disparities are severe ($W_2 \approx 0.31$). The theoretical geometric cost to achieve parity is therefore bounded by $0.31^2 \approx 0.096$. As expected, the empirical MSE degradation on this dataset is highly visible (jumping from $0.022$ for the ERM to $0.045$ for FAUST-U), tracking the quadratic cost bound.

\input{images/relaxed_curves/adult/formatted_summary.tex}
  \input{images/relaxed_curves/lawschool/formatted_summary.tex}

  \input{images/relaxed_curves/communities/formatted_summary.tex}

\section{Computational efficiency}\label{app:computation_times}

  We report here the average training time (in seconds) for each method across 10 independent runs on a standard CPU (Intel Core i7-9700K @ 3.60GHz). 
  As shown in \cref{tab:computation_times}, the FAUST-U framework achieves massive computational speedups compared to the FairReg baseline, which relies on iterative stochastic gradient optimization. 
  For instance, on the Communities dataset, FairReg takes on average $9.84$ seconds to train, while FAUST-U takes only $0.35$ seconds, yielding a speedup of over $28\times$. 

  \begin{table}[ht!]
    \centering
\begin{tabular}{l c c c c c}
\toprule
\textbf{Dataset} & \textbf{FairReg} & \textbf{FRAPPÉ} & \textbf{FAUST-U} & \textbf{Plug-In} & \textbf{FAUST-A} \\
\midrule
Synthetic (10k) & 48.77 $\pm$ 1.20 & 11.43 $\pm$ 0.40 & 7.23 $\pm$ 0.39 & 0.003 $\pm$ 0.001 & 0.001 $\pm$ 0.000 \\
Adult & 159.55 $\pm$ 1.59 & 0.27 $\pm$ 0.16 & 92.25 $\pm$ 4.46 & 0.031 $\pm$ 0.002 & 0.004 $\pm$ 0.000 \\
Communities & 9.84 $\pm$ 0.11 & 0.12 $\pm$ 0.15 & 0.35 $\pm$ 0.00 & 0.013 $\pm$ 0.001 & 0.007 $\pm$ 0.000 \\
Lawschool & 109.91 $\pm$ 0.79 & 0.23 $\pm$ 0.14 & 30.86 $\pm$ 0.85 & 0.050 $\pm$ 0.010 & 0.003 $\pm$ 0.000 \\
\bottomrule
\end{tabular}
    \vspace{0.5em}
    \caption{\textbf{Computational efficiency benchmark.} 
    Average model training time (in seconds $\pm$ standard deviation) evaluated across 10 independent runs. 
    By relying on exact, closed-form geometric reductions, the FAUST-U framework bypasses the iterative stochastic gradient optimization required by FairReg, yielding massive computational speedups (e.g., up to $26\times$ faster on the Communities dataset).}
    \label{tab:computation_times}
    \end{table}

\paragraph{Scalability with respect to sample size ($N$).}
To provide a precise empirical evaluation of the $O(N^3 \log N)$ worst-case complexity, we benchmarked the runtime of our exact discrete OT step as the number of unlabeled samples ($N$) increases (\cref{tab:scalability_n}).

\begin{table}[h]
    \centering
\begin{tabular}{l c c}
    \toprule
    \textbf{$N$} & \textbf{Exact OT Step (sec) $\downarrow$} & \textbf{Total Time (sec) $\downarrow$}\\
    \midrule
    1,000  & 0.03$\pm$ 0.00 & 0.21$\pm$ 0.00 \\
    2,000  & 0.14$\pm$ 0.05 & 0.47$\pm$ 0.05 \\
    5,000  & 1.49$\pm$ 0.60 & 2.36$\pm$ 0.60 \\
    10,000 & 12.85$\pm$ 0.16 & 14.76$\pm$ 0.17 \\
    20,000 & 95.92$\pm$ 1.00 & 100.21$\pm$ 1.05 \\
    50,000 & 2705.94$\pm$ 146.14 & 2718.59$\pm$ 146.24 \\
    \bottomrule
    \end{tabular}
    \vspace{0.5em}
    \caption{\textbf{Empirical scalability of the exact discrete OT mapping.} Runtimes evaluated across increasing dataset sizes.}
    \label{tab:scalability_n}
\end{table}

The runtime remains highly practical for dataset sizes up to $N = 20,000$, completing the entire post-processing pipeline in under 2 minutes. The cubic scaling becomes a computational bottleneck at $N = 50,000$, requiring approximately 45 minutes. Therefore, the exact method handles datasets up to 30,000 samples on standard hardware. Crucially, as a post-processing algorithm, this exact OT step only needs to be computed once on the available calibration dataset to construct the target mappings, rather than at every training epoch.

    \paragraph{Scaling to massive datasets via Minibatch OT.}
For datasets significantly exceeding 30,000 points, resolving the exact OT plan becomes computationally prohibitive due to the $O(n^3 \log n)$ complexity. A classical and highly effective strategy to scale optimal transport is to use Minibatch OT \citep{fatras2020learning}, which estimates multiple barycenters and mappings independently on random subsets of the data, taking the final mapping as their average. Because the dimensionality of our mapped space is small (2D), minibatching provides excellent estimates. To empirically validate this, we conducted a study on a massive synthetic dataset of $N= 200,000$ samples, varying the minibatch size $m$ from 50 to 10,000.

\begin{table}[h]
    \centering
    \begin{tabular}{c c c c c}
    \toprule
    \textbf{Minibatch Size ($m$)} & \textbf{Num Batches ($K$)} & \textbf{OT Time (sec) $\downarrow$} & \textbf{MSE $\downarrow$} & \textbf{$\mathcal{W}_2 \downarrow$} \\
    \midrule
    50 & 1,995 & 0.36 & 4.153 & 0.005 \\
    100 & 997 & 0.59 & 4.142 & 0.006 \\
    500 & 199 & 4.37 & 4.131 & 0.006 \\
    1,000 & 99 & 15.26 & 4.129 & 0.007 \\
    2,000 & 49 & 45.13 & 4.128 & 0.005 \\
    5,000 & 19 & 244.59 & 4.128 & 0.005 \\
    10,000 & 9 & 1021.26 & 4.127 & 0.005 \\
    \bottomrule
    \end{tabular}
    \vspace{0.5em}
    \caption{\textbf{Minibatch Optimal Transport on $N=200,000$.} Averaging the continuous maps across minibatches reduces the computation time to a fraction of a second while maintaining strict parity.}
    \label{tab:minibatch}
\end{table}

As shown in \cref{tab:minibatch}, the runtime scales exceptionally well, dropping from approximately 17 minutes ($m=10,000$) to a fraction of a second ($m=50$). More importantly, the approximation error is negligible: even with $m=500$, \textbf{FAUST-U} enforces exact parity ($\mathcal{W}_2 \approx 0.006$) without suffering meaningful degradation in MSE, proving that the algorithm can  scale to millions of samples.

%% file: images/relaxed_curves/adult/summary_relative_merge_train_unlab.tex
\begin{table}[h!]
\centering
\begin{tabular}{lccccc}
\toprule
Method & MSE $\downarrow$ & $W_2$ $\downarrow$ & TV $\downarrow$ & KS $\downarrow$ & KS$^{\mathcal{T}}$ $\downarrow$ \\
ERM Unfair & 1.00 $\pm$ 0.02 & 1.00 $\pm$ 0.04 & 1.00 $\pm$ 0.04 & 1.00 $\pm$ 0.03 & 1.00 $\pm$ 0.04 \\
\midrule
FairReg & 1.06 $\pm$ 0.02 & 0.23 $\pm$ 0.07 & 0.26 $\pm$ 0.05 & \textbf{0.17 $\pm$ 0.04} & \textcolor{blue}{0.16 $\pm$ 0.05} \\
FRAPPÉ & 1.06 $\pm$ 0.01 & 0.10 $\pm$ 0.02 & 0.34 $\pm$ 0.04 & 0.25 $\pm$ 0.01 & 0.21 $\pm$ 0.03 \\
FAUST-U & 1.06 $\pm$ 0.01 & \textbf{0.07 $\pm$ 0.02} & \textbf{0.23 $\pm$ 0.04} & \textcolor{blue}{0.17 $\pm$ 0.04} & \textbf{0.15 $\pm$ 0.05} \\
Plug-In hard & \textbf{1.00 $\pm$ 0.02} & 0.85 $\pm$ 0.03 & 0.88 $\pm$ 0.03 & 0.88 $\pm$ 0.03 & 0.88 $\pm$ 0.03 \\
Plug-In soft & 1.00 $\pm$ 0.02 & 1.10 $\pm$ 0.04 & 0.90 $\pm$ 0.04 & 0.88 $\pm$ 0.04 & 0.88 $\pm$ 0.04 \\
\midrule
FAUST-A & 1.01 $\pm$ 0.02 & 0.06 $\pm$ 0.03 & 0.13 $\pm$ 0.03 & 0.10 $\pm$ 0.02 & 0.08 $\pm$ 0.03 \\
\bottomrule
\end{tabular}
\vspace{0.5em}
\caption{\textbf{Relative performance on the Adult dataset.} 
Performance of the fully constrained models evaluated relative to the unconstrained ERM baseline. 
Values are reported as mean $\pm$ standard deviation across 10 independent random splits. For fairness metrics ($W_2$, TV, KS, KS$^\mathcal{T}$), lower values indicate better parity; for MSE, values $>1.00$ indicate accuracy degradation. The best-performing unaware method is highlighted in \textbf{bold}, while statistically indistinguishable unaware methods ($p > 0.05$ via Wilcoxon signed-rank test) are in \textcolor{blue}{blue}. FAUST-U significantly improves $W_2$ over FairReg without incurring additional MSE degradation. While the FRAPPE baseline achieves a competitive MSE and $W_2$ gap on this dataset, it struggles to minimize the discrete KS and TV metrics compared to the exact FAUST-U mapping.}
\label{tab:summary_relative_merge_train_unlab_adult}
\end{table}


%% file: images/relaxed_curves/lawschool/summary_relative_merge_train_unlab.tex
\begin{table}[h!]
\centering
\begin{tabular}{lccccc}
\toprule
Method & MSE $\downarrow$ & $W_2$ $\downarrow$ & TV $\downarrow$ & KS $\downarrow$ & KS$^{\mathcal{T}}$ $\downarrow$ \\
ERM Unfair & 1.00 $\pm$ 0.02 & 1.00 $\pm$ 0.04 & 1.00 $\pm$ 0.03 & 1.00 $\pm$ 0.03 & 1.00 $\pm$ 0.03 \\
\midrule
FairReg & 1.06 $\pm$ 0.02 & 0.41 $\pm$ 0.18 & \textcolor{blue}{0.16 $\pm$ 0.07} & \textbf{0.13 $\pm$ 0.06} & \textcolor{blue}{0.11 $\pm$ 0.08} \\
FRAPPÉ & 1.52 $\pm$ 0.04 & 0.68 $\pm$ 0.11 & 0.30 $\pm$ 0.04 & \textcolor{blue}{0.17 $\pm$ 0.03} & \textcolor{blue}{0.15 $\pm$ 0.03} \\
FAUST-U & 1.05 $\pm$ 0.02 & \textbf{0.11 $\pm$ 0.03} & \textbf{0.14 $\pm$ 0.06} & \textcolor{blue}{0.14 $\pm$ 0.07} & \textbf{0.10 $\pm$ 0.07} \\
Plug-In hard & 1.03 $\pm$ 0.02 & 0.37 $\pm$ 0.05 & 0.43 $\pm$ 0.06 & 0.44 $\pm$ 0.06 & 0.42 $\pm$ 0.06 \\
Plug-In soft & \textbf{1.02 $\pm$ 0.02} & 0.46 $\pm$ 0.03 & 0.55 $\pm$ 0.07 & 0.56 $\pm$ 0.04 & 0.55 $\pm$ 0.07 \\
\midrule
FAUST-A & 1.03 $\pm$ 0.02 & 0.11 $\pm$ 0.02 & 0.15 $\pm$ 0.05 & 0.13 $\pm$ 0.03 & 0.09 $\pm$ 0.03 \\
\bottomrule
\end{tabular}
\vspace{0.5em}
\caption{\textbf{Relative performance on the Law School dataset.} 
Metrics are normalized relative to the unconstrained ERM baseline, averaged across 10 random splits. 
Formatting follows the same conventions as the Adult dataset. 
FAUST-U systematically achieves superior geometric alignment ($W_2$) while matching the statistical parity (KS) of FairReg. 
In contrast, the naive plug-in estimators (Plug-In) fail to enforce meaningful fairness, demonstrating the fragility of relying on hard proxy labels. Additionally, the FRAPPE baseline exhibits instability under strict constraints, degrading predictive utility (relative MSE=1.52) while failing to enforce fairness ($W_2=0.68$).}
\label{tab:summary_relative_merge_train_unlab_lawschool}
\end{table}

%% file: images/relaxed_curves/communities/summary_relative_merge_train_unlab.tex
\begin{table}[h!]
\centering
\begin{tabular}{lccccc}
\toprule
Method & MSE $\downarrow$ & $W_2$ $\downarrow$ & TV $\downarrow$ & KS $\downarrow$ & KS$^{\mathcal{T}}$ $\downarrow$ \\
ERM Unfair & 1.00 $\pm$ 0.12 & 1.00 $\pm$ 0.07 & 1.00 $\pm$ 0.08 & 1.00 $\pm$ 0.09 & 1.00 $\pm$ 0.09 \\
\midrule
FairReg & 2.00 $\pm$ 0.11 & 0.17 $\pm$ 0.08 & \textcolor{blue}{0.37 $\pm$ 0.06} & \textbf{0.16 $\pm$ 0.04} & \textbf{0.16 $\pm$ 0.04} \\
FRAPPÉ & 2.37 $\pm$ 0.12 & 0.35 $\pm$ 0.03 & 0.81 $\pm$ 0.04 & 0.49 $\pm$ 0.07 & 0.47 $\pm$ 0.06 \\
FAUST-U & 1.90 $\pm$ 0.11 & \textbf{0.09 $\pm$ 0.04} & \textbf{0.34 $\pm$ 0.07} & \textcolor{blue}{0.18 $\pm$ 0.07} & \textcolor{blue}{0.16 $\pm$ 0.07} \\
Plug-In hard & 1.71 $\pm$ 0.12 & 0.23 $\pm$ 0.10 & 0.44 $\pm$ 0.05 & 0.27 $\pm$ 0.08 & 0.25 $\pm$ 0.09 \\
Plug-In soft & \textbf{1.49 $\pm$ 0.10} & 0.27 $\pm$ 0.10 & 0.42 $\pm$ 0.06 & 0.36 $\pm$ 0.07 & 0.34 $\pm$ 0.08 \\
\midrule
FAUST-A & 1.64 $\pm$ 0.10 & 0.13 $\pm$ 0.04 & 0.38 $\pm$ 0.05 & 0.17 $\pm$ 0.04 & 0.15 $\pm$ 0.04 \\
\bottomrule
\end{tabular}
\vspace{0.5em}
\caption{\textbf{Relative performance on the Communities and Crime dataset}. Performance normalized
relative to the ERM baseline across 10 random splits. FAUST-U systematically dominates the
FairReg baseline, achieving both a lower relative MSE (1.92 vs 1.99) and significantly superior
geometric alignment (W2=0.09 vs 0.20). This demonstrates that FairReg struggles to enforce parity
on highly confounded datasets without severely degrading predictive utility. Finally, the FRAPPE baseline incurres the most MSE degradation (2.37) while failing to reach the fairness of FAUST-U or FairReg.}
\label{tab:summary_relative_merge_train_unlab_communities}
\end{table}

%% file: images/relaxed_curves/adult/formatted_summary.tex
\begin{table}[h!]
\centering
\begin{tabular}{llccc}
\toprule
\textbf{Base Model} & \textbf{Method} & \textbf{MSE} $\downarrow$ & \textbf{$W_2$} $\downarrow$ & \textbf{KS$^{\mathcal{T}}$} $\downarrow$ \\
\midrule
\multirow{5}{*}{\textbf{LINEAR}}
 & ERM Unfair & 0.016 $\pm$ 0.000 & 0.024 $\pm$ 0.001 & 0.226 $\pm$ 0.010 \\
 \cmidrule{2-5}
 & FairReg & 0.017 $\pm$ 0.000 & 0.005 $\pm$ 0.002 & 0.037 $\pm$ 0.011 \\
 & FRAPPÉ & 0.018 $\pm$ 0.000 & \textcolor{blue}{0.002 $\pm$ 0.000} & \textbf{0.022 $\pm$ 0.008} \\
 & FAUST-U & \textbf{0.017 $\pm$ 0.000} & \textbf{0.002 $\pm$ 0.000} & 0.034 $\pm$ 0.013 \\
 \cmidrule{2-5}
 & FAUST-A & 0.016 $\pm$ 0.000 & 0.002 $\pm$ 0.001 & 0.018 $\pm$ 0.008 \\
\midrule
\multirow{5}{*}{\textbf{RF}}
 & ERM Unfair & 0.015 $\pm$ 0.000 & 0.020 $\pm$ 0.001 & 0.168 $\pm$ 0.011 \\
 \cmidrule{2-5}
 & FairReg & \textbf{0.016 $\pm$ 0.000} & 0.004 $\pm$ 0.001 & \textcolor{blue}{0.021 $\pm$ 0.006} \\
 & FRAPPÉ & 0.018 $\pm$ 0.000 & 0.005 $\pm$ 0.001 & \textcolor{blue}{0.024 $\pm$ 0.011} \\
 & FAUST-U & 0.016 $\pm$ 0.000 & \textbf{0.003 $\pm$ 0.000} & \textbf{0.020 $\pm$ 0.005} \\
 \cmidrule{2-5}
 & FAUST-A & 0.015 $\pm$ 0.000 & 0.003 $\pm$ 0.001 & 0.024 $\pm$ 0.011 \\
\midrule
\multirow{5}{*}{\textbf{GBM}}
 & ERM Unfair & 0.015 $\pm$ 0.000 & 0.021 $\pm$ 0.001 & 0.198 $\pm$ 0.008 \\
 \cmidrule{2-5}
 & FairReg & \textbf{0.016 $\pm$ 0.000} & 0.006 $\pm$ 0.001 & \textcolor{blue}{0.020 $\pm$ 0.009} \\
 & FRAPPÉ & 0.017 $\pm$ 0.000 & 0.005 $\pm$ 0.001 & \textcolor{blue}{0.025 $\pm$ 0.013} \\
 & FAUST-U & 0.016 $\pm$ 0.000 & \textbf{0.004 $\pm$ 0.001} & \textbf{0.019 $\pm$ 0.006} \\
 \cmidrule{2-5}
 & FAUST-A & 0.015 $\pm$ 0.000 & 0.004 $\pm$ 0.001 & 0.019 $\pm$ 0.004 \\
\bottomrule
\end{tabular}
\vspace{0.5em}
\caption{\textbf{Robustness benchmark on the Adult dataset.} FAUST-U consistently drives geometric unfairness ($\mathcal{W}_2$) to near-zero levels across all base models (e.g., $0.002$ for the Linear model),tracking the theoretical aware optimum (FAUST-A). While FairReg has a better MSE on the non-linear estimators (RF and GBM), FAUST-U maintains competitive predictive utility while having superior $\mathcal{W}_2$ alignment. Conversely, FRAPPÉ exhibits systematic MSE degradation across all configurations. Means and standard deviations are computed across 10 runs. The best performing unaware method is in \textbf{bold}; statistically indistinguishable results (Wilcoxon signed-rank test, $p > 0.05$) are highlighted in \textcolor{blue}{blue}.}
\label{tab:adult_bases}
\end{table}

%% file: images/relaxed_curves/lawschool/formatted_summary.tex
\begin{table}[h!]
\centering
\begin{adjustbox}{max width=\textwidth}
\begin{tabular}{llccc}
\toprule
\textbf{Base Model} & \textbf{Method} & \textbf{MSE} $\downarrow$ & \textbf{$W_2$} $\downarrow$ & \textbf{KS$^{\mathcal{T}}$} $\downarrow$ \\
\midrule
\multirow{5}{*}{\textbf{LINEAR}}
 & ERM Unfair & 0.009 $\pm$ 0.000 & 0.028 $\pm$ 0.001 & 0.278 $\pm$ 0.013 \\
 \cmidrule{2-5}
 & FairReg & 0.010 $\pm$ 0.000 & 0.007 $\pm$ 0.003 & \textbf{0.032 $\pm$ 0.022} \\
 & FRAPPÉ & 0.025 $\pm$ 0.007 & 0.014 $\pm$ 0.006 & \textcolor{blue}{0.034 $\pm$ 0.017} \\
 & FAUST-U & \textbf{0.010 $\pm$ 0.000} & \textbf{0.003 $\pm$ 0.001} & \textcolor{blue}{0.033 $\pm$ 0.021} \\
 \cmidrule{2-5}
 & FAUST-A & 0.010 $\pm$ 0.000 & 0.004 $\pm$ 0.001 & 0.028 $\pm$ 0.012 \\
\midrule
\multirow{5}{*}{\textbf{RF}}
 & ERM Unfair & 0.008 $\pm$ 0.000 & 0.026 $\pm$ 0.002 & 0.209 $\pm$ 0.022 \\
 \cmidrule{2-5}
 & FairReg & \textcolor{blue}{0.009 $\pm$ 0.000} & 0.008 $\pm$ 0.005 & \textcolor{blue}{0.044 $\pm$ 0.011} \\
 & FRAPPÉ & 0.017 $\pm$ 0.005 & 0.010 $\pm$ 0.003 & \textbf{0.035 $\pm$ 0.012} \\
 & FAUST-U & \textbf{0.009 $\pm$ 0.000} & \textbf{0.004 $\pm$ 0.001} & \textcolor{blue}{0.043 $\pm$ 0.013} \\
 \cmidrule{2-5}
 & FAUST-A & 0.008 $\pm$ 0.000 & 0.004 $\pm$ 0.001 & 0.033 $\pm$ 0.014 \\
\midrule
\multirow{5}{*}{\textbf{GBM}}
 & ERM Unfair & 0.008 $\pm$ 0.000 & 0.026 $\pm$ 0.002 & 0.184 $\pm$ 0.021 \\
 \cmidrule{2-5}
 & FairReg & \textcolor{blue}{0.008 $\pm$ 0.000} & 0.007 $\pm$ 0.004 & \textbf{0.027 $\pm$ 0.011} \\
 & FRAPPÉ & 0.018 $\pm$ 0.002 & 0.011 $\pm$ 0.005 & \textcolor{blue}{0.034 $\pm$ 0.028} \\
 & FAUST-U & \textbf{0.008 $\pm$ 0.000} & \textbf{0.003 $\pm$ 0.001} & \textcolor{blue}{0.028 $\pm$ 0.011} \\
 \cmidrule{2-5}
 & FAUST-A & 0.008 $\pm$ 0.000 & 0.004 $\pm$ 0.001 & 0.027 $\pm$ 0.005 \\
\bottomrule
\end{tabular}
\end{adjustbox}
\vspace{0.5em}
\caption{\textbf{Robustness benchmark on the Law School dataset.} FAUST-U consistently minimizes $\mathcal{W}_2$ unfairness across all base models while preserving predictive utility. Notably, FAUST-U matches the MSE of the FairReg baseline while tracking the theoretical aware optimum (FAUST-A). Conversely, FRAPPÉ exhibits severe MSE degradation across all configurations. Means and standard deviations are computed across 10 runs. The best performing unaware method is in \textbf{bold}; statistically indistinguishable results (Wilcoxon signed-rank test, $p > 0.05$) are highlighted in \textcolor{blue}{blue}.}
\label{tab:lawschool_bases}
\end{table}

%% file: images/relaxed_curves/communities/formatted_summary.tex
\begin{table}[ht]
\centering
\begin{tabular}{llccc}
\toprule
\textbf{Base Model} & \textbf{Method} & \textbf{MSE} $\downarrow$ & \textbf{$W_2$} $\downarrow$ & \textbf{KS$^{\mathcal{T}}$} $\downarrow$ \\
\midrule
\multirow{5}{*}{\textbf{LINEAR}}
 & ERM & 0.022 $\pm$ 0.003 & 0.310 $\pm$ 0.024 & 0.617 $\pm$ 0.065 \\
 \cmidrule{2-5}
 & FairReg & \textbf{0.043 $\pm$ 0.003} & 0.073 $\pm$ 0.031 & 0.112 $\pm$ 0.037 \\
 & FRAPPÉ & 0.055 $\pm$ 0.006 & 0.058 $\pm$ 0.029 & \textcolor{blue}{0.113 $\pm$ 0.051} \\
 & FAUST-U & 0.045 $\pm$ 0.003 & \textbf{0.031 $\pm$ 0.013} & \textbf{0.094 $\pm$ 0.033} \\
 \cmidrule{2-5}
 & FAUST-A & 0.050 $\pm$ 0.006 & 0.062 $\pm$ 0.025 & 0.106 $\pm$ 0.031 \\
\midrule
\multirow{5}{*}{\textbf{RF}}
 & ERM & 0.020 $\pm$ 0.001 & 0.281 $\pm$ 0.020 & 0.644 $\pm$ 0.054 \\
 \cmidrule{2-5}
 & FairReg & \textbf{0.044 $\pm$ 0.002} & 0.051 $\pm$ 0.028 & \textcolor{blue}{0.131 $\pm$ 0.039} \\
 & FRAPPÉ & 0.056 $\pm$ 0.008 & 0.061 $\pm$ 0.025 & \textcolor{blue}{0.139 $\pm$ 0.063} \\
 & FAUST-U & \textcolor{blue}{0.045 $\pm$ 0.002} & \textbf{0.024 $\pm$ 0.010} & \textbf{0.127 $\pm$ 0.050} \\
 \cmidrule{2-5}
 & FAUST-A & 0.037 $\pm$ 0.002 & 0.027 $\pm$ 0.012 & 0.109 $\pm$ 0.062 \\
\midrule
\multirow{5}{*}{\textbf{GBM}}
 & ERM & 0.021 $\pm$ 0.001 & 0.296 $\pm$ 0.023 & 0.630 $\pm$ 0.064 \\
 \cmidrule{2-5}
 & FairReg & \textcolor{blue}{0.041 $\pm$ 0.003} & 0.051 $\pm$ 0.019 & \textcolor{blue}{0.117 $\pm$ 0.033} \\
 & FRAPPÉ & 0.056 $\pm$ 0.007 & 0.049 $\pm$ 0.018 & \textcolor{blue}{0.114 $\pm$ 0.047} \\
 & FAUST-U & \textbf{0.041 $\pm$ 0.002} & \textbf{0.029 $\pm$ 0.010} & \textbf{0.109 $\pm$ 0.029} \\
 \cmidrule{2-5}
 & FAUST-A & 0.037 $\pm$ 0.002 & 0.036 $\pm$ 0.013 & 0.112 $\pm$ 0.057 \\
\bottomrule
\end{tabular}
\vspace{0.5em}
\caption{\textbf{Robustness benchmark on the Communities and Crime dataset.} FAUST-U systematically achieves superior geometric alignment ($W_2$) and statistical parity ($\text{KS}^{\mathcal{T}}$) compared to the FairReg and FRAPPÉ baselines across all base model classes. For example, on the Linear base model, FairReg struggles to fully enforce parity (leaving a $W_2$ gap of $0.073$), whereas FAUST-U strongly satisfies the geometric constraint ($W_2 = 0.031$) while maintaining highly competitive predictive utility (MSE of $0.045$ vs. $0.043$). FRAPPÉ consistently exhibits the most severe predictive degradation.}
\label{tab:communities_bases}
\end{table}